\documentclass{article}

\pdfoutput=1

\input{glyphtounicode}
\PassOptionsToPackage{numbers,compress}{natbib}

\usepackage[preprint]{neurips_2026}

\usepackage[utf8]{inputenc} 
\usepackage[T1]{fontenc}    
\usepackage{hyperref}       
\usepackage{url}            
\usepackage{booktabs}       
\usepackage{amsfonts}       
\usepackage{nicefrac}       
\usepackage{microtype}      
\usepackage[table]{xcolor}  

\usepackage{graphicx}
\usepackage{amsmath}
\usepackage{amssymb}
\usepackage{mathtools}
\usepackage{amsthm}

\usepackage[capitalize,noabbrev]{cleveref}

\theoremstyle{plain}
\newtheorem{theorem}{Theorem}[section]
\newtheorem{proposition}[theorem]{Proposition}
\newtheorem{lemma}[theorem]{Lemma}
\newtheorem{corollary}[theorem]{Corollary}
\theoremstyle{definition}

\theoremstyle{remark}

\usepackage{latexsym}

\usepackage{algorithm}


\usepackage{algpseudocode} 
\algrenewcommand\alglinenumber[1]{} 

\usepackage{multirow}
\usepackage{balance}
\usepackage{enumitem}
\usepackage{xspace}

\newcommand{\model}{TR-RAG}
\providecommand{\KL}{\mathrm{KL}}

\newcommand{\stitle}[1]{\vspace{1mm} \noindent {\bf #1}}

\newcommand{\wrt}{w.r.t. }

\newcommand{\RightComment}[1]{\hfill$\triangleright$~#1}

\usepackage{caption}
\usepackage{type1cm} 
\usepackage{lettrine}

\usepackage{paracol} 
\usepackage{wrapfig}

\usepackage[disable,textsize=tiny]{todonotes}

\usepackage{pifont} 

\title{Distill Where the Student Goes: Teacher-Regularized RL for English-Evidence Cross-Lingual RAG}

\author{%
  Haotian Zhou$^{1,2,\ast\ddagger}$ \quad
  Weiran Huang$^{1,3,\ast}$ \quad
  Siqi Liu$^{2,4,\ddagger}$ \quad
  Xiting Wang$^{4}$ \\
  \bfseries Xin Zhang$^{2}$ \quad
  Zhihao Wen$^{2,\dagger}$ \\
  \vspace{1mm}\\
  \textsuperscript{1}School of Computer Science, Shanghai Jiao Tong University \quad
  \textsuperscript{2}Ant International, Ant Group \\
  \textsuperscript{3}Shanghai Innovation Institute \quad
  \textsuperscript{4}Gaoling School of Artificial Intelligence, Renmin University of China
}

\begin{document}

\maketitle

{\renewcommand{\thefootnote}{\fnsymbol{footnote}}%
  \footnotetext[1]{Equal contribution. Contact: \texttt{haotianzhou@sjtu.edu.cn}, \texttt{weiran.huang@outlook.com}.}%
  \footnotetext[2]{Corresponding author. Contact: \texttt{z.wen@ant-intl.com}.}%
  \footnotetext[3]{This work was done during the authors' internship at Ant International, Ant Group.}%
}
\setcounter{footnote}{0}

\begin{abstract}
Cross-lingual retrieval-augmented generation (RAG) is often deployed in an \emph{English-evidence} regime, where users query in diverse languages but retrieved passages remain English.
In this setting, generation can fail despite strong base models: English evidence induces \emph{language drift} (English or code-switching outputs) and models use evidence unreliably when producing non-English answers.
We attribute these failures to two post-training challenges: (i) errors are \emph{prefix-dependent}, so fixed-trajectory supervision suffers from \emph{prefix mismatch}; and (ii) sequence-level (partly discrete / judge-based) rewards yield noisy credit assignment and high-variance updates.
We propose \textbf{\model}, a teacher-regularized RL recipe that couples reward optimization with \emph{on-policy distillation} on student-visited prefixes.
A compact student samples on-policy answers, while a stronger \emph{frozen} teacher is queried only on those prefixes and provides a prefix-wise student-to-teacher reverse-KL anchor.
We further introduce a reward decomposition for English-evidence multilingual generation, combining language consistency, character 3-gram recall, and an LLM-judge score for evidence-grounded correctness.
Across three benchmarks (BioASQ-ENKB5, Hotpot-ENKB5, and naturally multilingual MKQA) and two backbones, \model\ improves the composite of language adherence and evidence-grounded correctness over strong baselines. Crucially, the teacher anchor acts as a \emph{safety net}: on in-domain languages it prevents the large language-consistency collapses (up to $\sim$27 percentage points) that reward-only RL can suffer by drifting below even the base model, while on distant out-of-distribution languages, where reward-only RL stalls at the base model's ceiling, it still improves evidence grounding; and on character 3-gram recall the compact student sometimes surpasses its 70B teacher.
\end{abstract}

\section{Introduction}
\label{sec:intro}

Retrieval-augmented generation (RAG) grounds LLMs in external evidence and enables rapid knowledge updates
\citep{DBLP:conf/nips/LewisPPPKGKLYR020,Gao2023RetrievalAugmentedGF}.
Many deployments are cross-lingual: users query and expect answers in diverse languages, while knowledge bases remain predominantly English
\citep{DBLP:journals/tacl/ClarkPNCGCK20,asai2021xor,zhang-etal-2023-miracl}.
We study this practical regime (\emph{English-evidence cross-lingual RAG}), where queries and answers are non-English but retrieved passages are English.
When retrieval is reasonably strong, the main bottleneck shifts to the \textbf{generation side} \citep{Qi2025OnTC}: the model must use
\textit{English} evidence while producing a fluent, correct \emph{non-English} response: selecting, aggregating, and realizing evidence in the target language within a single decoding trajectory.
We therefore hold retrieval fixed and focus on generator post-training.

\stitle{Industrial motivation.}
This setting is a recurring production bottleneck. In cross-border merchant assistants and global customer-service systems,
documentation and policy text are often English-only while users query in their native languages \citep{DBLP:journals/corr/abs-2505-10089, DBLP:journals/corr/abs-2407-01463}, so the retriever returns
English passages and the generator must respond fluently in the user's language. A common failure is \emph{language drift}: the model switches output language even when the underlying answer is correct \citep{li2025language, Park2025InvestigatingLP}. This motivates our problem formulation, reward design, and prefix-wise teacher anchor.

\stitle{Two generation-stage challenges and their shared optimization bottleneck.}
With English passages injected into context, LLMs face two salient issues:
\textbf{\ding{182} Language drift under English evidence.}
Models often answer in English or code-switch under query--evidence language mismatch \citep{DBLP:journals/corr/abs-2505-10089};
English evidence shifts next-token probabilities toward English continuations, a pull that reasoning-style decoding can further amplify \citep{wei2022chain,li2025language}.
\textbf{\ding{183} Cross-lingual evidence use with non-English answers.}
Models must identify relevant spans, aggregate across documents, and express conclusions faithfully in the target language; in practice they may attend to irrelevant evidence, aggregate it unreliably, or introduce unsupported details \citep{DBLP:journals/corr/abs-2407-01463}.
Both failures are strongly \emph{prefix-dependent} (an early wrong-language token or partial mistake compounds later errors), exposing two training bottlenecks:
\textbf{(i) Prefix mismatch}: SFT/offline KD provides gradients only on reference trajectories, so once the student deviates it visits failure prefixes with no direct supervision (\emph{exposure bias}) \citep{bengio2015scheduled,ranzato2015sequence,DBLP:conf/iclr/AgarwalVZSGGB24}.
\textbf{(ii) Noisy credit assignment}: language adherence and evidence correctness are judged at the sequence level (partly discrete / judge-based), yielding high-variance policy-gradient updates \citep{williams1992simple,sutton1999policy,ouyang2022training,kazemnejad2024vineppo}.

\stitle{Our approach.}
We propose \textbf{\model}, a teacher-regularized on-policy post-training recipe that directly targets (i)--(ii).
A compact student generates answers on-policy, while a stronger \emph{frozen} teacher is queried on
\emph{student-visited prefixes} to provide dense token-level guidance (no teacher rollouts) via a prefix-wise reverse-KL
anchor $\mathrm{KL}(\pi_\theta(\cdot|s_t)\|\pi_T(\cdot|s_t))$ \citep{DBLP:conf/iclr/AgarwalVZSGGB24}.
This discourages unlikely continuations, mitigating prefix mismatch and stabilizing reward optimization.
On top of this anchor, we optimize a cross-lingual task reward combining language consistency
\citep{DBLP:journals/corr/abs-2505-10089,li2025language}, character 3-gram recall
\citep{popovic-2015-chrf,DBLP:conf/bea/KulmizevBBNNPW17}, and an LLM-judge score
\citep{DBLP:conf/nips/ZhengC00WZL0LXZ23}.
Our main contributions:
\begin{itemize}[topsep=-1pt,itemsep=1.5pt,parsep=0pt,partopsep=0pt]
    \item We study \emph{English-evidence cross-lingual RAG} from a generation-centric perspective and identify two prefix-dependent failure modes: language drift and brittle evidence use.
    \item We introduce \model, which combines sequence-level rewards with \emph{on-policy distillation} via a dense prefix-wise reverse-KL anchor on student-visited prefixes.
    \item Across three benchmarks (BioASQ-ENKB5, Hotpot-ENKB5, and naturally multilingual MKQA) and two backbones, \model\ attains the best composite score and, unlike reward-only RL, acts as a \emph{safety net}: it prevents the large (up to $\sim$27pp) in-domain language-consistency collapses that reward-only RL can suffer by drifting below the base model, and still improves grounding on distant out-of-distribution languages (Norwegian, Khmer) where reward-only RL stalls at the base ceiling, while sometimes surpassing its 70B teacher on character 3-gram recall.
\end{itemize}

\begin{figure*}[tbp]
   \centering
   \includegraphics[width=0.95\linewidth]{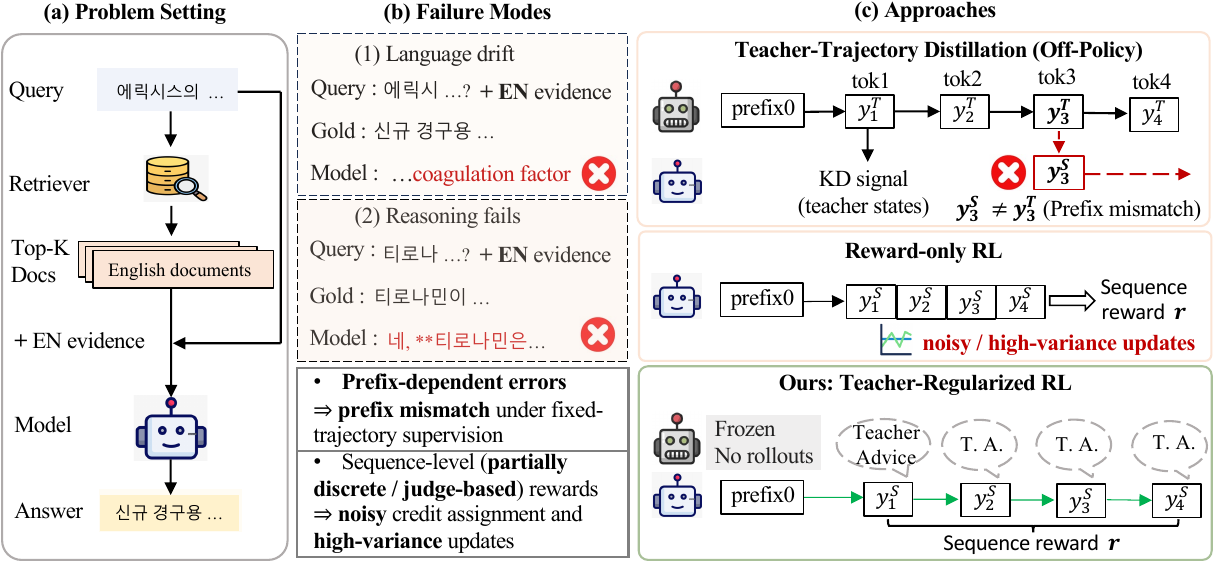}
   \setlength{\belowcaptionskip}{-3pt}
\caption{\textbf{Overview.}
\textbf{(a)} \emph{English-evidence cross-lingual RAG}: queries/answers are in a target language while retrieved passages are English. (The panel labels the retrieved documents ``Top-$K$'' in the generic retrieval sense; throughout the text their count is denoted $M$, as $K$ is reserved for the number of on-policy rollouts per input.)
\textbf{(b)} Two prefix-dependent failure modes (\emph{language drift} and \emph{evidence-use failures}) induce two training obstacles: \emph{prefix mismatch} under fixed-trajectory supervision and high-variance optimization under \emph{sequence-level} rewards.
\textbf{(c)} \model\ couples on-policy RL with on-policy distillation: a frozen teacher is queried on student-visited prefixes to produce a dense reverse-KL anchor $\mathrm{KL}(\pi_\theta\|\pi_T)$ (no teacher rollouts). ``T.A.'' denotes this \emph{teacher anchor}.
}
   \label{fig:motivation}
\end{figure*}
\section{Preliminaries}
\label{sec:prelim}

\stitle{Task: English-evidence cross-lingual RAG.}
We study a practical cross-lingual RAG regime where the knowledge base is English-only.
Given a user query $q$ written in a target language $\ell$,
a retriever returns a set of top-$M$ English passages $D=\{d_k\}_{k=1}^{M}$.
The generator must produce an answer $y=(y_1,\ldots,y_L)$ of length $L$ in the same target language $\ell$,
grounded in $D$.

\stitle{Policies and decoding states.}
Let $x=(q,D)$ denote the conditioning input.
We denote the student/policy model by $\pi_\theta(y\mid x)$ and a stronger teacher model by $\pi_T(y\mid x)$.
Generation is autoregressive:
\begin{align}
\pi_\theta(y\mid x)=\prod_{t=1}^{L}\pi_\theta(y_t\mid x,y_{<t}).
\end{align}
We use $s_t=(x,y_{<t})=(q,D,y_{<t})$ to denote the decoding state at step $t$.

\stitle{On-policy RL fine-tuning.}
We consider on-policy fine-tuning of $\pi_\theta$ with a sequence-level task reward $r(x,y)$,
where $y\sim \pi_\theta(\cdot\mid x)$ is a sampled rollout.
When using policy-gradient updates, a standard objective is
\begin{align}
\max_{\theta}\;\; \mathbb E_{x}\,\mathbb E_{y\sim\pi_\theta(\cdot\mid x)}\big[r(x,y)\big],
\label{eq:pg_prelim}
\end{align}
typically implemented with a variance-reduction baseline.
Our implementation adopts a GRPO-style \citep{shao2024deepseekmath, DBLP:journals/corr/abs-2501-12948} on-policy update, which estimates advantages using $K$ rollouts per input
and a group mean baseline.

\stitle{Positioning \wrt related work.}
\model\ couples this reward optimization with \emph{on-policy distillation} \citep{DBLP:conf/iclr/AgarwalVZSGGB24}. Two contrasts locate our contribution. (i)~Unlike pure on-policy distillation (GKD/OPD), which uses the teacher anchor as the \emph{sole} training signal, we combine the anchor with sequence-level task rewards in a single objective, and show the combination strictly dominates either alone on the composite metric (Section~\ref{sec:exp:In-depth}). (ii)~Unlike concurrent RL recipes for cross-lingual RAG (LcRL~\citep{qi2026lcrl}, CroSearch-R1~\citep{qi2026crosearch}) that optimize the \emph{retrieval} side over multilingual corpora, we hold retrieval fixed and target \emph{generation} under English-only evidence, with a cross-lingual reward decomposition (Section~\ref{sec:method:reward}). A comprehensive discussion is deferred to Appendix~\ref{sec:app:related}.

\section{Methodology}
\label{sec:method}

\subsection{Overview}
\label{sec:method:overview}
We study English-evidence cross-lingual RAG generation, where retrieved evidence $D$ is English but the required answer
must follow the user language $\ell$.
As shown in Fig.~\ref{fig:framework}, given inputs $(q,D)$ we (i) sample on-policy answers and compute task rewards,
(ii) query a stronger teacher on the student-visited prefix states, and (iii) update the student by maximizing a reward--regularizer objective.
Additional theoretical notes appear in Appendix~\ref{app:proofs}.

\subsection{Reward Decomposition}
\label{sec:method:reward}

In the English-evidence cross-lingual RAG setting, we need training signals that jointly encourage (i) strict adherence to the target language, (ii) high answer accuracy relative to reference answers, and (iii) faithful grounding in retrieved evidence. To this end, we decompose the task reward into three complementary components, each designed to capture a distinct aspect of generation quality in multilingual contexts:
\begin{equation}
r(q,D,y)=\lambda_{\text{lang}}\, r_{\text{lang}}(y,a^\star)+\lambda_{\text{3g}}\, r_{\text{3g}}(y,a^\star)+\lambda_{\text{judge}}\, r_{\text{judge}}(q,D,a^\star,y),
\label{eq:reward_total}
\end{equation}
where $a^\star$ is the reference answer (only available during training) and $\lambda_{\text{lang}},\lambda_{\text{3g}},\lambda_{\text{judge}}\ge0$ control the relative importance of each signal.\footnote{We fix $\lambda_{\text{lang}}{=}0.5,\ \lambda_{\text{3g}}{=}\lambda_{\text{judge}}{=}1$ in all runs. The deployed reward also carries two implementation details described in Appendix~\ref{app:reward_weights}: a bounded length penalty $r_{\text{len}}$ folded into the language term to discourage runaway-length outputs, and a rule that skips the LID check for strings shorter than 10 characters, on which LID is unreliable. Because these terms leave $r$ bounded (albeit no longer in $[0,1]$), our stability bounds hold up to the reward range (Appendix~\ref{app:stability}).}

\stitle{Language-consistency reward.}
A key failure mode in multilingual RAG is \emph{language drift}: once English evidence is injected, models may output English or mixed-language text even when the query is in another language \citep{DBLP:journals/corr/abs-2505-10089, li2025language}. To directly counteract this, we include an explicit language-consistency signal:
\begin{equation}
r_{\text{lang}}(y,a^\star)=\mathbb{I}\!\left[\mathrm{LID}(y)=\mathrm{LID}(a^\star)\right],
\label{eq:reward_lang}
\end{equation}
where $\mathrm{LID}(\cdot)$ is a lightweight language identification (LID) classifier (the fastText \texttt{lid.176} model \citep{joulin2017bag, grave2018learning}) and $a^\star$ is the reference answer. Since we ensure that the reference answer language matches the query language during dataset construction (Appendix~\ref{sec:app:exp:datasets}), this is equivalent to checking $\mathrm{LID}(y)=\ell(q)$ in practice, up to LID errors on very short reference strings (see Appendix~\ref{sec:app:exp:experiment:main}).

\stitle{Character 3-gram recall reward.}
Quantifying answer-level accuracy in multilingual contexts is challenging due to morphological variation and tokenization issues. Prior work in machine translation and multilingual evaluation demonstrates that character-level n-gram metrics such as chrF/chrF++ correlate better with human judgments than token-level scores \citep{popovic-2015-chrf, popovic-2017-chrf}, especially across diverse languages. This is especially relevant for our target languages: Thai lacks whitespace word boundaries, Korean is agglutinative with rich morphological inflection, and Vietnamese carries tone-bearing diacritics; in all three, token-level F1/EM is fragile to tokenization choices, while character n-grams are not. Inspired by this, we include a character 3-gram recall term as a lightweight lexical proxy:
\begin{equation}
r_{\text{3g}}(y,a^\star)=\frac{\left|\mathcal{G}_3(y)\cap \mathcal{G}_3(a^\star)\right|}{\left|\mathcal{G}_3(a^\star)\right|},
\label{eq:reward_3gram}
\end{equation}
where $\mathcal{G}_3(\cdot)$ denotes the \emph{multiset} of character 3-grams. For multisets,
$\left|\mathcal{G}_3(y)\cap \mathcal{G}_3(a^\star)\right|
=\sum_{g}\min\{c_y(g),c_{a^\star}(g)\}$, where $c_y(g)$ is the count of 3-gram $g$ in $y$.
This reward captures partial surface overlap and morphological variants that are common in multilingual generation, providing a continuous signal that complements strict match metrics without requiring language-specific tokenization.\footnote{When the reference answer has fewer than 3 characters, $|\mathcal{G}_3(a^\star)|=0$; we define $r_{\text{3g}}=0$ in this edge case to avoid division by zero.}

\stitle{LLM-judge reward.}
Lexical proxies capture surface overlap but cannot fully reflect semantic correctness or evidence grounding.
Recent work shows that LLMs can act as effective context-conditioned evaluators \citep{DBLP:conf/nips/ZhengC00WZL0LXZ23}.
We therefore introduce an LLM-as-a-judge signal:
\begin{align}
r_{\text{judge}}(q,D,a^\star,y)\in[0,1],
\label{eq:reward_judge}
\end{align}
where the evaluator scores how well $y$ (i) answers the query, (ii) is supported by evidence $D$, and (iii) matches the reference $a^\star$ in meaning and completeness.
This continuous signal provides nuanced feedback when exact matches are rare in multilingual QA.
We use $r_{\text{judge}}$ only as one term in our reward decomposition and verify its reliability via an alternative judge on a sampled subset (Appendix~\ref{sec:app:exp:metrics:judge_reliability}).
Together, these complementary rewards encourage \model\ to generate responses that are linguistically correct, evidence-grounded, and semantically accurate across languages.

\begin{figure*}[t]
   \centering
   \includegraphics[width=0.95\linewidth]{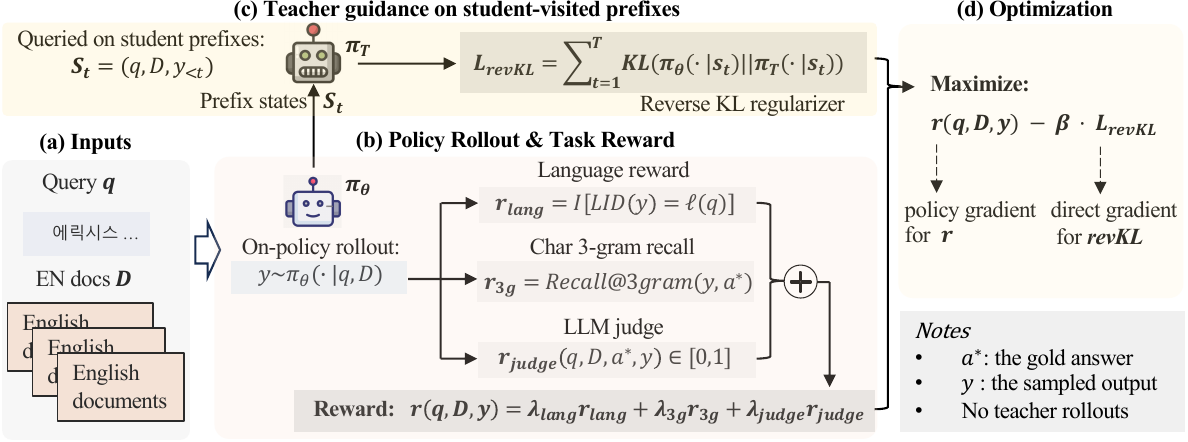}
\caption{
\textbf{Framework.}
(a) Inputs: a multilingual query $q$ and English-only evidence $D$.
(b) The student samples on-policy rollouts and receives the task reward.
(c) The teacher is queried only on student prefixes to compute a KL anchor (no teacher rollouts).
(d) The student is updated with policy gradients and a frozen-teacher reverse-KL anchor.
}
   \label{fig:framework}
\end{figure*}

\subsection{RL Fine-tuning with On-policy Distillation}
\label{sec:method:opd}

\noindent\textbf{(Fig.~\ref{fig:framework}c--d)}
We combine reward-driven RL (Fig.~\ref{fig:framework}b) with prefix-wise teacher regularization (Fig.~\ref{fig:framework}c),
and optimize the objective in Fig.~\ref{fig:framework}(d).
In English-evidence cross-lingual RAG, errors are highly \emph{prefix-dependent}: an early wrong-language token or
incorrect partial claim can push later states $(x,y_{<t})$ away from the support of standard supervision.
We therefore regularize the student on the prefixes it actually visits, directly addressing the prefix mismatch in
Fig.~\ref{fig:framework}(c).
Because our reward is sequence-level and partly discrete (e.g., LID), teacher guidance on student prefixes also provides
a dense, low-variance token-level stabilizer.

\stitle{Student rollouts (on-policy trajectories).}
\noindent\textbf{(Fig.~\ref{fig:framework}b)}
For each training instance $x=(q,D)$, we sample $K$ on-policy student rollouts:
\begin{align}
y^{(i)} \sim \pi_\theta(\cdot\mid x),\qquad i=1,\ldots,K.
\label{eq:rollout_group}
\end{align}
Both task reward and teacher regularization are computed on these trajectories, so training matches the prefix
distribution encountered at test time.

\stitle{Teacher guidance on student-visited prefixes.}
\noindent\textbf{(Fig.~\ref{fig:framework}c)}
Let $y^{(i)}=(y^{(i)}_1,\ldots,y^{(i)}_{L_i})$ and define the student prefix state as
$s^{(i)}_t=(x,y^{(i)}_{<t})=(q,D,y^{(i)}_{<t})$.
Instead of asking the teacher to generate its own rollouts, we query it only on the states actually visited by the student:
\begin{align}
\pi_T(\cdot \mid s^{(i)}_t)\quad \text{for } i=1,\ldots,K,\ \ t=1,\ldots,L_i.
\end{align}
This directly addresses \emph{prefix mismatch}: even after an early deviation (e.g., language drift), the teacher still
provides a corrective distribution on that state, where offline KD would offer no supervision.\footnote{
The teacher is never used for rollouts; it only provides token distributions on student-generated prefixes.
}
These teacher queries also provide a low-variance token-level learning signal that complements the higher-variance reward.

\stitle{``Reverse KL'' as a teacher anchor.}
\noindent\textbf{(Fig.~\ref{fig:framework}c)}
We anchor the student to the teacher on student-visited prefixes with a length-normalized student-to-teacher KL:
\begin{align}
\mathcal{L}_{\text{revKL}}(x,y^{(i)})
=
\frac{1}{L_i}
\sum_{t=1}^{L_i}
\mathrm{KL}\!\Big(\pi_\theta(\cdot\mid s^{(i)}_t)\ \big\|\ \pi_T(\cdot\mid s^{(i)}_t)\Big),
\label{eq:token_revkl}
\end{align}
where $L_i$ is the number of generated tokens in $y^{(i)}$, including the terminal EOS token.
Relative to the standard distillation form $\mathrm{KL}(\pi_T(\cdot|s)\|\pi_\theta(\cdot|s))$, we refer to this as the
``reverse'' direction.
It penalizes probability mass on tokens the teacher deems unlikely under the same prefix, acting as a \emph{local trust
region} that stabilizes reward-driven updates while preserving exploration.
This loss requires a shared tokenizer and vocabulary, which holds in our experiments because student and teacher are paired
within the same family (e.g., Qwen-4B and Qwen-30B).

\stitle{Final objective (instantiation of \Cref{sec:method:reward}).}
Let $\mathcal{B}$ denote the training distribution over $(x,a^\star)$.
We optimize a teacher-regularized objective that trades off task reward and teacher anchoring:
\begin{equation}
\max_{\theta}\; J(\theta)
=
\mathbb{E}_{(x,a^\star)\sim \mathcal{B}}\,\mathbb{E}_{y\sim \pi_\theta(\cdot\mid x)}
\big[(1-\alpha)\, r(x,y) - \alpha\, \mathcal{L}_{\text{revKL}}(x,y)\big],
\label{eq:final_objective_alpha}
\end{equation}
where $\alpha\in[0,1]$ trades off task reward (Section~\ref{sec:method:reward}) and teacher anchoring. For $\alpha\in[0,1)$,
this is equivalent up to a positive constant to:
\begin{equation}
\max_{\theta}\; \tilde{J}(\theta)
=
\mathbb{E}_{(x,a^\star)\sim \mathcal{B}}\,\mathbb{E}_{y\sim \pi_\theta(\cdot\mid x)}
\big[r(x,y) - \beta\, \mathcal{L}_{\text{revKL}}(x,y)\big],
\label{eq:final_objective_beta}
\end{equation}
where $\beta=\alpha/(1-\alpha)\ge0$.

\stitle{Optimization.}
\noindent\textbf{(Fig.~\ref{fig:framework}d)}
We optimize Eq.~\eqref{eq:final_objective_beta} with GRPO-style on-policy updates: for each input $x$, we sample $K$
rollouts $\{y^{(i)}\}_{i=1}^K$ (we use $K{=}8$), compute rewards $r^{(i)} = r(x,y^{(i)})$, and form group-wise advantages
$A^{(i)} = ({r^{(i)}-\bar r})/({\sigma+\epsilon})$ (with $\epsilon{=}10^{-6}$).
The teacher anchor $\mathcal L_{\text{revKL}}$ enters as a differentiable loss with a frozen teacher and direct gradients
with respect to $\theta$; in implementation, the per-token KL is estimated with a low-variance sampled-token estimator
rather than a full-vocabulary summation (Appendix~\ref{app:gradients}).
The full procedure is in Algorithm~\ref{alg:trrag} (Appendix~\ref{sec:app:pseudocode}).

\stitle{Efficiency analysis.}
Relative to standard RL, \model's only extra computation is a single teacher forward pass over student-visited prefixes; crucially, we never sample teacher rollouts.
The overhead therefore scales with the teacher's forward cost: it is small for compact or MoE teachers (e.g., \texttt{Qwen3-30B-A3B}, matching Naive-RL/OPD in wall-clock) and more noticeable, though still a single forward pass, for a large dense teacher (e.g., \texttt{Llama-3.3-70B}).
We report measured per-setting training times in \Cref{sec:exp:In-depth} (\Cref{fig:qwen_2_time,fig:qwen_llama_time}).

\section{Experiments}


We evaluate \model\footnote{Anonymous code and data are available at \textcolor{blue}{\url{https://anonymous.4open.science/r/TR-RAG}} for review.} on English-evidence cross-lingual RAG, comparing against strong baselines and then making in-depth model analysis.
More details are in Appendix~\ref{sec:app:exp}.

\subsection{Experimental Setup}
\label{sec:exp:setup}

\stitle{Datasets.}
Our setting requires an \emph{English-only} knowledge base, which existing multilingual QA benchmarks do not directly satisfy: MKQA~\citep{DBLP:journals/tacl/LongpreLD21} pairs with multilingual Wikipedia, XOR-QA~\citep{asai2021xor} includes non-English evidence, and TyDi~QA~\citep{DBLP:journals/tacl/ClarkPNCGCK20} is monolingual per language.
We therefore adopt two English QA benchmarks, rag-mini-bioasq~\citep{tsatsaronis2015overview} (40k biomedical passages, 4.7k single-hop questions, 80/20 split) and HotpotQA~\citep{yang2018hotpotqa} (multi-hop Wikipedia; 4.4k/1.1k bridge\footnote{We keep only \emph{bridge}-type questions (excluding \emph{comparison} questions): bridge questions require genuine multi-hop entity bridging and give a cleaner, more homogeneous multi-hop generation signal.} train/test), and machine-translate only the questions and short reference answers into five non-English languages (\textit{id}/\textit{ko}/\textit{th}/\textit{vi}/\textit{ja}) using \texttt{Qwen3-235B-A22B-Instruct-2507}~\citep{yang2025qwen3}, while leaving evidence passages untouched in English; this is the most controlled construction at scale for our English-evidence regime. We refer to these multilingual variants as \textit{BioASQ-ENKB5} and \textit{Hotpot-ENKB5} (train on \textit{id}/\textit{ko}/\textit{th}/\textit{vi}; held-out \textit{ja} and the original \textit{en}).
To verify that gains are not artifacts of translationese, we additionally evaluate on \textbf{MKQA} (\Cref{sec:exp:mkqa}), whose questions are authored by native speakers rather than machine-translated~\citep{DBLP:journals/tacl/LongpreLD21}; we pair them with English Wikipedia as the sole evidence source, preserving the same English-evidence regime (train on \textit{ko}/\textit{th}/\textit{vi};\footnote{We drop Indonesian for MKQA because MKQA has no Indonesian split: it covers Malay (\textit{ms}) instead, so \mbox{ENKB-RAG-5}'s \textit{id} has no MKQA counterpart.} OOD \textit{en}/\textit{ja}/\textit{no}, with Norwegian as a deliberately distant stress test).
Additional details are in Appendix~\ref{sec:app:exp:datasets}.

\stitle{Metrics.}
Because English-evidence cross-lingual RAG can fail by switching language, missing key content across scripts, or producing
fluent but weakly grounded answers, we evaluate with three complementary normalized metrics in $[0,1]$ rather than a single
score:
\emph{language consistency} $\mathbb{I}[\mathrm{LID}(y)=\mathrm{LID}(a^\star)]$ from the fastText \texttt{lid.176} model \citep{joulin2017bag, grave2018learning} checks whether the
response stays in the target language;
\emph{char 3-gram recall} (Eq.~\eqref{eq:reward_3gram}) provides a tokenization-robust lexical overlap measure for
multilingual free-form answers; unlike token-level exact match or F1, it remains stable across scripts, word-segmentation
conventions, inflection, and minor orthographic variation, making it especially necessary for English-evidence cross-lingual
RAG where correctness often hinges on preserving translated or transliterated entities from English passages;
and a graded \emph{LLM-judge} score following the LLM-as-a-judge protocol \citep{DBLP:conf/nips/ZhengC00WZL0LXZ23}
measures evidence-grounded semantic correctness beyond surface overlap.
Because our \emph{primary} judge (\texttt{Qwen3-Next-80B}) is the same model that supplies the training reward $r_{\text{judge}}$, we guard against within-family self-preference by additionally reporting a cross-size (\texttt{Qwen3-235B}) and a cross-family (\texttt{Llama-3.3-70B}) judge (Appendices~\ref{sec:app:exp:metrics:judge_reliability} and~\ref{sec:app:exp:metrics:cross_family_judge}); the other two metrics (LC, char 3-gram) are judge-independent by construction.
We also report \emph{Composite} to summarize the overall trade-off: for \mbox{ENKB-RAG-5}, it is the mean of the three metrics; for MKQA, where we report two judges, it is the mean of all four signals.
Appendix~\ref{sec:app:exp:metrics} gives full details.

\stitle{Models and baselines.}
We study two compact backbones, \texttt{Qwen3-4B-Instruct-2507} \citep{yang2025qwen3} and \texttt{Llama-3.2-3B-Instruct} \citep{dubey2024llama}, and pair them with larger teachers, \texttt{Qwen3-30B-A3B-Instruct-2507} and \texttt{Llama-3.3-70B-Instruct} \citep{yang2025qwen3,dubey2024llama}.
Baselines include \textbf{Translate} \citep{asai2021xor}, \textbf{Prompt-Control} \cite{DBLP:journals/corr/abs-2504-18428}, discourse-cue prompting (\textbf{DIT} \citep{luo2025mmath}), supervised fine-tuning (\textbf{SFT}), \textbf{Knowledge Distillation} (KD) \citep{DBLP:journals/corr/HintonVD15}, self-distillation fine-tuning (\textbf{SDFT} \citep{yang-etal-2024-self}), \textbf{Naive-RL} \citep{DBLP:journals/corr/abs-2501-12948}, and pure \textbf{On-Policy Distillation} (OPD) \citep{DBLP:conf/iclr/AgarwalVZSGGB24}.
More details are in Appendix~\ref{sec:app:exp:baselines}.

\begin{table*}[t]
\centering
\scriptsize
\renewcommand{\arraystretch}{0.95} 
\setlength{\tabcolsep}{6pt} 
\caption{Performance comparison in BioASQ-ENKB5, with \texttt{Llama-3.2-3B-Instruct} as the backbone, and \texttt{Llama-3.3-70B-Instruct} as the teacher. \textbf{Bold} is the best, and \underline{underline} is the runner-up. The teacher (\texttt{Llama-3.3-70B-Instruct}) and base (\texttt{Llama-3.2-3B-Instruct}) rows are reference-only and are \emph{not} included in the best/runner-up ranking.}
\begin{tabular}{l|rrrrr|rrr|r}
\hline
\multirow{2}{*}{Methods} &
\multicolumn{5}{c|}{In-Domain Languages} &
\multicolumn{3}{c|}{Out-of-Domain Languages} &
\multirow{2}{*}{ALL-AVG} \\
\cline{2-9}
& id & ko & th & vi & ID-AVG
& en & ja & OOD-AVG & \\
\hline
\multicolumn{10}{c}{\textit{Metric: language consistency (LC, \%)}} \\
\hline
\texttt{Llama-3.3-70B-Instruct}   &90.78&98.51&99.25&98.41&96.74 &99.78&99.36&99.57 &97.68 \\
\hline\rowcolor{gray!8}
\texttt{Llama-3.2-3B-Instruct}   &87.71&95.76&98.72&97.66&94.96 &\textbf{99.78}&98.41&99.10 &96.34 \\
Translate &87.92&91.10&97.67&97.88&93.64 &\textbf{99.78}&93.54&96.66 &94.65 \\
Prompt-Control    &88.77&96.39&99.04&\underline{98.30}&95.63 &\textbf{99.78}&98.83&99.31 &96.85 \\
DIT          &\underline{90.25}&89.83&\textbf{99.25}&97.66&94.25 &\textbf{99.78}&98.94&99.36 &95.95 \\
SFT                  &84.22&95.66&97.03&95.44&93.09 &\underline{99.68}&97.35&98.52 &94.90 \\
Knowledge distillation                        &89.08&97.35&98.62&97.66&95.68 &99.47&99.04&99.26 &96.87 \\
SDFT         &89.61&\textbf{97.77}&98.41&98.19&96.00 &99.57&99.15&99.36 &97.12 \\
Naive-RL      &89.83&96.61&99.15&98.20&95.95 &99.58&98.73&99.16 &97.02 \\
On-Policy Distillation   &\underline{90.25}&\underline{97.56}&98.20&98.09&\underline{96.03} &99.58&\textbf{99.26}&\underline{99.42} &\underline{97.16} \\
\rowcolor{gray!15}
\model &\textbf{90.99}&97.03&\underline{99.15}&\textbf{98.41}&\textbf{96.40} &\underline{99.68}&\underline{99.25}&\textbf{99.47} &\textbf{97.42} \\

\hline
\multicolumn{10}{c}{\textit{Metric: char 3-gram recall (\%)}} \\
\hline
\texttt{Llama-3.3-70B-Instruct}   &69.66&38.13&42.61&65.62&54.01 &76.58&35.28&55.93 &54.65 \\
\hline\rowcolor{gray!8}
\texttt{Llama-3.2-3B-Instruct}   &56.16&26.42&29.62&52.83&41.26 &70.56&25.68&48.12 &43.55 \\
Translate &61.40&21.12&29.38&55.68&41.90 &70.56&20.43&45.50 &43.10 \\
Prompt-Control    &60.78&28.61&33.85&58.41&45.41 &74.31&\underline{28.23}&51.27 &47.37 \\
DIT          &62.95&26.24&31.55&59.41&45.04 &75.78&27.05&51.42 &47.16 \\
SFT                 &37.16&21.62&26.44&36.35&30.39 &47.09&22.29&34.69 &31.83 \\
Knowledge distillation                        &62.22&28.73&35.30&57.41&45.92 &71.40&26.86&49.13 &46.99 \\
SDFT         &62.84&29.10&35.73&57.60&46.32 &72.13&27.10&49.62 &47.42 \\
Naive-RL       &\underline{75.62}&22.66&\underline{39.82}&\underline{69.47}&\underline{51.89} &\underline{79.44}&23.63&\underline{51.54} &\underline{51.77} \\
On-Policy Distillation   &63.84&\underline{29.51}&35.72&59.84&47.23 &72.45&26.70&49.58 &48.01 \\
\rowcolor{gray!15}
\model &\textbf{80.16}&\textbf{29.55}&\textbf{48.69}&\textbf{73.54}&\textbf{57.99} &\textbf{85.14}&\textbf{31.46}&\textbf{58.30} &\textbf{58.09} \\
\hline

\multicolumn{10}{c}{\textit{Metric: LLM-judge (\%)}} \\
\hline
\texttt{Llama-3.3-70B-Instruct}   &76.83&73.98&75.27&76.57&75.66 &78.78&76.42&77.60 &76.31 \\
\hline\rowcolor{gray!8}
\texttt{Llama-3.2-3B-Instruct}   &63.79&47.34&54.52&63.11&57.19 &72.61&48.94&60.78 &58.39 \\
Translate &63.11&34.06&40.06&56.52&48.44 &72.61&41.18&56.90 &51.26 \\
Prompt-Control    &63.60&41.95&52.23&61.77&54.89 &71.79&49.41&60.60 &56.79 \\
DIT          &60.83&42.44&45.83&57.46&51.64 &67.91&47.79&57.85 &53.71 \\
SFT                  &66.10&54.45&57.84&60.70&59.77 &\underline{75.32}&61.86&68.59 &62.71 \\
Knowledge distillation                        &66.70&50.97&57.61&64.77&60.01 &74.76&58.66&66.71 &62.25 \\
SDFT         &67.27&53.77&58.06&65.98&61.27 &74.50&58.13&66.32 &62.95 \\
Naive-RL       &\underline{71.50}&\textbf{70.11}&60.38&\underline{68.04}&\textbf{69.57} &74.94&\textbf{69.03}&\textbf{71.99} &\textbf{70.38} \\
On-Policy Distillation   &68.66&55.75&\underline{61.00}&67.17&63.15 &\textbf{75.74}&60.80&68.27 &64.85 \\
\rowcolor{gray!15}
\model &\textbf{71.94}&\underline{66.44}&\textbf{66.36}&\textbf{70.35}&\underline{68.77} &75.08&\underline{63.97}&\underline{69.53} &\underline{69.02} \\
\hline

\multicolumn{10}{c}{\textit{Metric: Composite (\%)}} \\
\hline
\texttt{Llama-3.3-70B-Instruct}   &79.09&70.21&72.38&80.20&75.47 &85.05&70.35&77.70 &76.21 \\
\hline\rowcolor{gray!8}
\texttt{Llama-3.2-3B-Instruct}   &69.22&56.51&60.95&71.20&64.47 &80.98&57.68&69.33 &66.09 \\
Translate &70.81&48.76&55.70&70.03&61.33 &80.98&51.72&66.35 &63.00 \\
Prompt-Control    &71.05&55.65&61.71&72.83&65.31 &81.96&58.82&70.39 &67.00 \\
DIT          &71.34&52.84&58.88&71.51&63.64 &81.16&57.93&69.55 &65.61 \\
SFT                  &62.49&57.24&60.44&64.16&61.08 &74.03&60.50&67.27 &63.14 \\
Knowledge distillation                        &72.67&59.02&63.84&73.28&67.20 &81.88&61.52&71.70 &68.70 \\
SDFT         &73.24&60.21&64.07&73.92&67.86 &82.07&61.46&71.77 &69.16 \\
Naive-RL       &\underline{78.98}&\underline{63.13}&\underline{69.00}&\underline{78.77}&\underline{72.47} &\underline{84.65}&\underline{63.80}&\underline{74.23} &\underline{73.06} \\
On-Policy Distillation   &74.25&60.94&64.97&75.03&68.80 &82.59&62.25&72.42 &70.01 \\
\rowcolor{gray!15}
\model &\textbf{81.03}&\textbf{64.34}&\textbf{71.40}&\textbf{80.77}&\textbf{74.38} &\textbf{86.63}&\textbf{64.89}&\textbf{75.76} &\textbf{74.84} \\
\hline

\end{tabular}
\label{tab:llama-bioasq}
\end{table*}

\subsection{Main Results}
\label{sec:exp:main_results}

We use \Cref{tab:llama-bioasq} (BioASQ-ENKB5 with \texttt{Llama-3.2-3B-Instruct} as the student) as the representative example in the main text; results for the Qwen-4B backbone on BioASQ-ENKB5 and for both backbones on Hotpot-ENKB5 (\Cref{tab:qwen-bioasq,tab:llama-hotpotqa,tab:qwen-hotpotqa}, Appendix~\ref{sec:app:exp:experiment:main}) show the same trend that \model\ consistently outperforms strong baselines.

\textbf{Language Consistency.}
Across settings, \model\ ranks among the top methods on language consistency, substantially mitigating response-language drift under English evidence. The gain over SFT-style fine-tuning is clearest on in-domain languages (e.g., $+3.3$pp ID-AVG on BioASQ-ENKB5, \Cref{tab:llama-bioasq}, where offline supervision can even push LC \emph{below} the base model), indicating that \model's on-policy teacher regularization preserves language adherence more robustly than offline supervision.\\
\textbf{Character 3-gram Recall.}
\model\ is the top performer on char 3-gram recall, sometimes even surpassing the 70B teacher, giving direct, \emph{judge-independent} evidence that it transfers the right lexical content across languages with fewer dropped, mistranslated, or corrupted strings (we motivate this tokenization-robust metric in the Metrics paragraph above). This matches our design: the char 3-gram reward applies dense pressure against missing entities and keywords, while the teacher anchor stabilizes decoding under early prefix deviations, so source-supported content is preserved even when generation begins to drift.\\
\textbf{LLM-judge.}
On the evidence-grounded LLM-judge metric, \model\ achieves the second-best performance overall, closely trailing Naive-RL. This suggests strong evidence-grounded reasoning beyond surface-form overlap, while highlighting that our main gains arise from balanced improvements across complementary objectives rather than maximizing a single metric. Part of this small gap is a length--conciseness artifact of the judge rubric (the char 3-gram reward promotes more complete, longer answers that the conciseness criterion penalizes) rather than a correctness regression, as we isolate in Appendix~\ref{sec:app:exp:metrics:conciseness}.\\
\textbf{Composite Metric.}
The composite score offers the clearest summary of end-to-end utility: \model\ delivers the \textbf{best} overall trade-off among language adherence, lexical faithfulness, and evidence-grounded correctness, and is the top method on \textbf{both in-domain and out-of-domain} languages. This indicates that the gains are not restricted to seen languages, but transfer robustly under language shift, a key requirement in English-evidence deployments. In contrast, prompt-based control yields only marginal improvements, suggesting that prompt engineering alone is insufficient to resolve generation-side failures once English passages are injected.\\
\textbf{Additional observations.}
The \textit{Translate} pipeline attains high language consistency but degrades other metrics and the composite score, reflecting compounding translation-induced mismatches that distort rare entities and weaken evidence anchoring. \textit{SFT} is a revealing outlier: its char 3-gram recall drops well below the untuned base model ($43.55\rightarrow31.83$ on BioASQ), because teacher-forcing on short gold answers instills an ultra-terse extractive style that under-covers the reference content; full-parameter SFT is worse still (Appendix~\ref{sec:app:expimplementation_details:baseline_hparams}, \Cref{tab:sft}). On-policy distillation also clearly beats its offline counterpart: querying the teacher on \emph{student-visited} prefixes provides corrective guidance exactly where the student drifts, yielding stronger composite performance under both \textbf{ID} and \textbf{OOD} settings. Finally, the compact \model\ \textbf{matches or even exceeds} its teacher in composite performance (e.g., 79.68 vs.\ 78.10 for the Qwen-4B/30B pair on Hotpot-ENKB5, \Cref{tab:qwen-hotpotqa}), indicating effective compression of cross-lingual RAG generation quality into small models.

\begin{table*}[]
    \centering
    \footnotesize
    \caption{Ablation studies (\textit{composite ALL-AVG} metric, \%). \textbf{RL-only} is \model\ without the teacher anchor ($\beta{=}0$), i.e., identical to the \textbf{Naive-RL} baseline; \textbf{OPD-only} uses the teacher anchor alone (no task reward). \textbf{Bold}/\underline{underline} denote best/runner-up.
    }
    \vspace{-2mm} 
    \label{table:ablation}%
    \renewcommand{\arraystretch}{0.9} 
    \setlength{\tabcolsep}{4pt} 
    \resizebox{\linewidth}{!}{%
    \begin{tabular}{@{}l|cccc|cc|cc|c@{}}
    \toprule
    \multirow{2}*{Variants}
    &Language&Char 3-gram& LLM-judge & On-Policy  &\multicolumn{2}{c|}{BioASQ-ENKB5} &\multicolumn{2}{c|}{Hotpot-ENKB5} & \multirow{2}*{Average}  \\
    &Reward&Reward& Reward & Distillation & Llama$_{3B}$ & Qwen$_{4B}$ & Llama$_{3B}$  & Qwen$_{4B}$ &    \\
    \midrule
    RL-only
    &$\checkmark$&$\checkmark$ & $\checkmark$ & $\times$  &73.06&81.98&72.66&74.79&75.62 \\ 
    OPD-only
    &$\times$&$\times$ & $\times$ & $\checkmark$ &70.01  &79.52&71.15&76.49&74.29\\
    No language reward
    &$\times$&$\checkmark$&$\checkmark$&$\checkmark$ &\underline{74.14}&\textbf{83.73}&\underline{73.07}&\textbf{79.70}&\underline{77.66} \\
    No 3-gram reward
    &$\checkmark$&$\times$&$\checkmark$&$\checkmark$ &69.79&75.28&70.10&76.01&72.80 \\
    No LLM-judge reward
    &$\checkmark$&$\checkmark$&$\times$&$\checkmark$ &72.55&81.42&71.19&78.46&75.91 \\
    \midrule
    \model
    & $\checkmark$ & $\checkmark$ &$\checkmark$&$\checkmark$   &\textbf{74.84}&\underline{82.85}&\textbf{73.88}&\underline{79.68}&\textbf{77.81} \\  
    \bottomrule
    \end{tabular}}
\end{table*}

\subsection{In-depth Model Analysis}
\label{sec:exp:In-depth}

\stitle{Ablation Studies.}
After establishing strong end-to-end gains in the main results, we ask \emph{which ingredients drive the improvements}. 
\Cref{table:ablation} reports ablations on our composite ALL-AVG metric. The full \model\ achieves the best score (\textbf{77.81}), and its gains come from \emph{combining} reward optimization with on-policy distillation rather than either alone: RL-only reaches 75.62 and OPD-only reaches 74.29, while \model\ improves by \textbf{+2.19} and \textbf{+3.52}, respectively, consistent with teacher anchoring stabilizing on-policy updates on student-visited prefixes.
The reward components contribute unequally. Removing character 3-gram recall causes the largest drop (77.81$\rightarrow$72.80, \textbf{$-5.01$}), indicating this dense, language-robust signal is critical in the English-evidence multilingual regime. Removing the LLM-judge also hurts ($-1.90$), suggesting complementary semantic/evidence-grounding feedback. In contrast, removing the language reward barely moves the aggregate composite ($-0.15$) \emph{or} aggregate LC (91.66 vs.\ 91.92): on these largely high-resource \mbox{ENKB-RAG-5} splits the teacher anchor and char 3-gram reward already keep outputs in-language, so the explicit language reward contributes little \emph{here}. We stress that this ablation covers only the \mbox{ENKB-RAG-5} splits; whether the language reward is more valuable on extreme low-resource OOD languages is not separately isolated, and we make no such claim. The two overlap-oriented signals also interact: removing the char 3-gram reward yields the \emph{highest} LC (94.31 vs.\ 91.92 for full \model; \Cref{table:ablation-Language Consistency}), consistent with the char 3-gram term rewarding overlap with English-derived entities and thereby mildly pulling outputs toward English. More results are in Appendix~\ref{sec:app:exp:ablation studies}.


\begin{figure}[t!]
  \centering
  \begin{minipage}[t]{0.190\linewidth}
    \centering
    \includegraphics[width=1\linewidth]{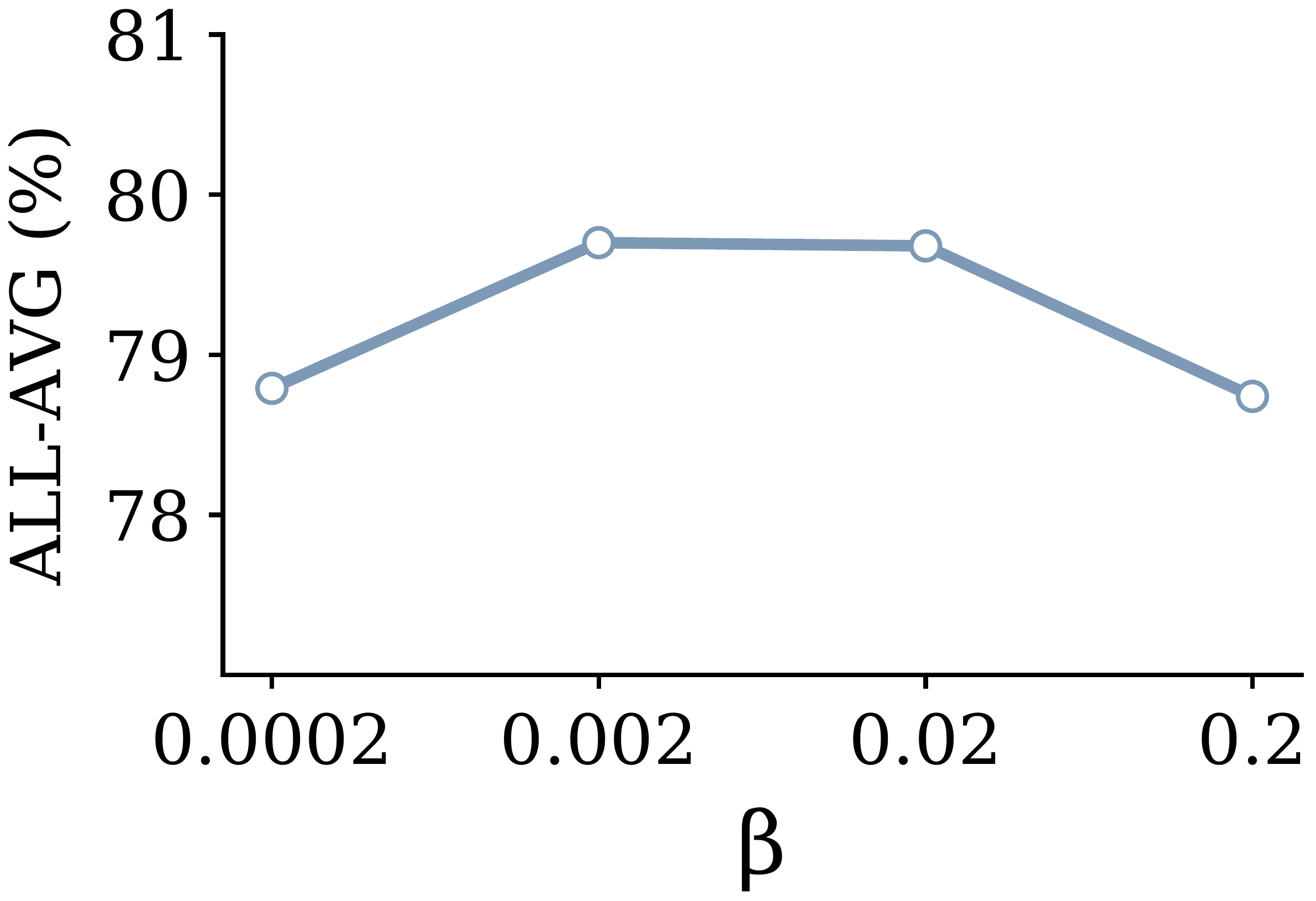}
    \vspace{-5mm}
    \caption{\tiny Teacher weight $\beta$ (Hotpot-ENKB5, Qwen-4B). Moderate $\beta$ ($0.002$--$0.02$) is best; $\beta{=}0.2$ over-regularizes. We use $\beta{=}0.02$.}
    \label{fig:hyper sensitivity}
  \end{minipage}%
  \hfill
  \begin{minipage}[t]{0.218\linewidth}
    \centering
    \includegraphics[width=1\linewidth]{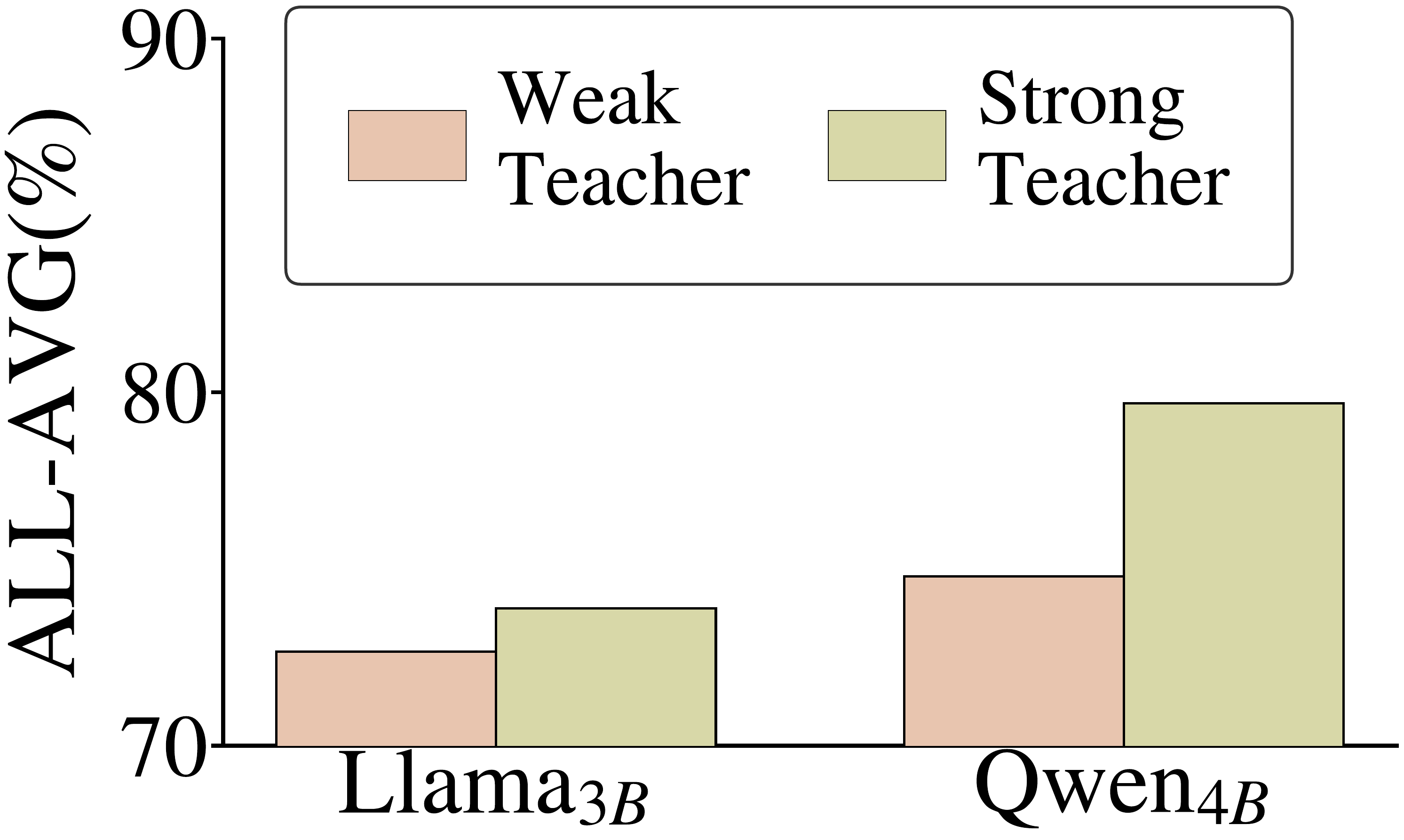}
    \vspace{-5mm}
    \caption{\tiny Teacher quality (Hotpot-ENKB5): strong vs.\ same-size weak teacher. A stronger teacher helps most; a weak teacher still performs competitively.}
    \label{fig:teacher_comparison_hotpotQA}
  \end{minipage}%
  \hfill
  \begin{minipage}[t]{0.210\linewidth}
    \centering
    \includegraphics[width=1\linewidth]{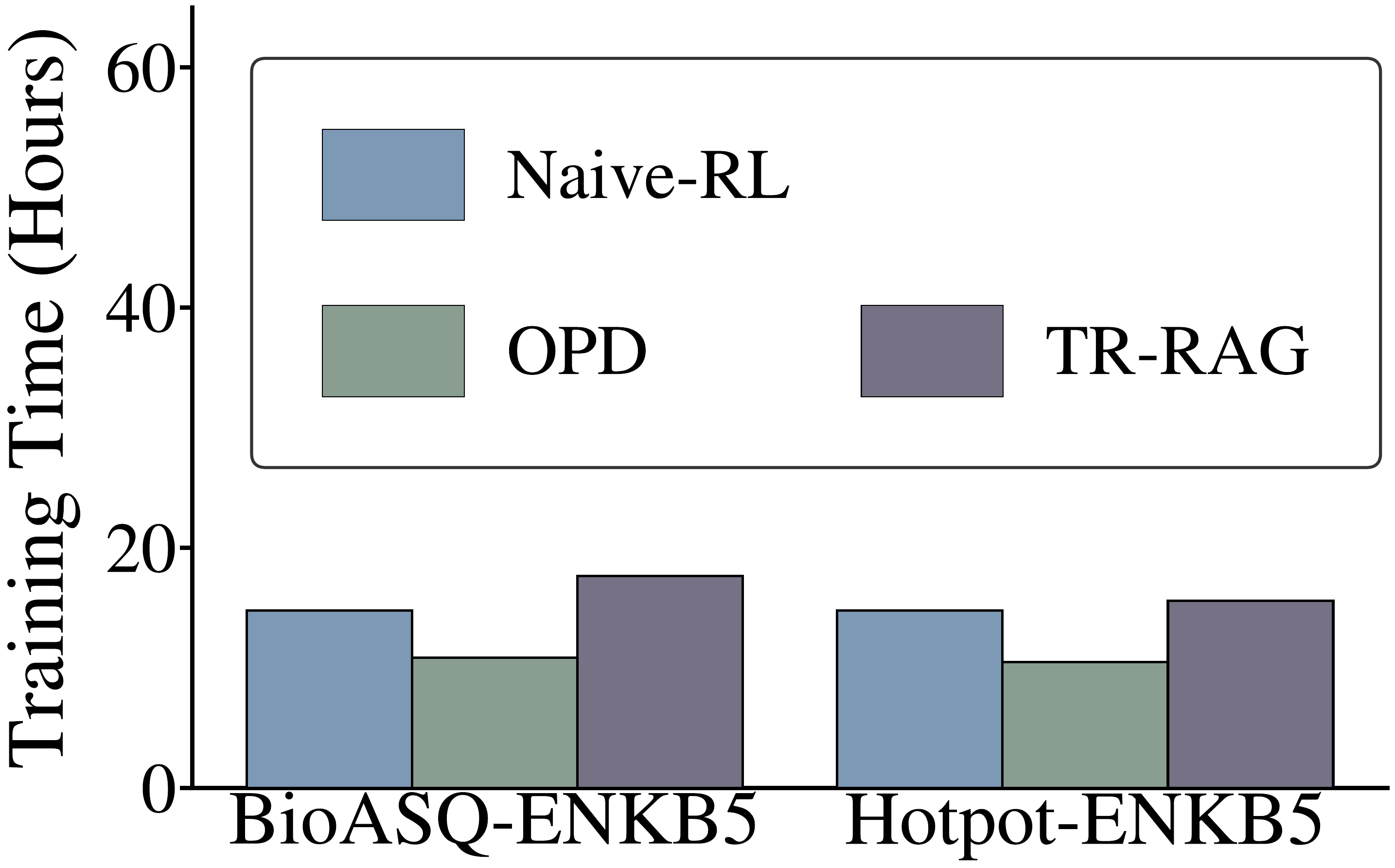}
    \vspace{-5mm}
    \caption{\tiny Training time (Qwen-4B backbone, BioASQ + Hotpot). \model\ $\approx$ Naive-RL/OPD (compact/MoE teacher).}
    \label{fig:qwen_2_time}
  \end{minipage}%
  \hfill
  \begin{minipage}[t]{0.165\linewidth}
    \centering
    \includegraphics[width=1\linewidth]{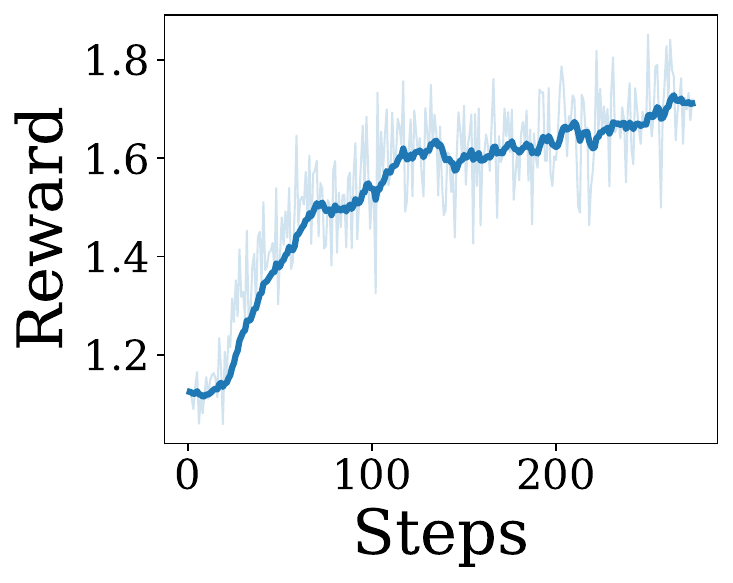}
    \vspace{-5mm}
    \caption{\tiny Task reward $r(x,y)$ (Llama, Hotpot): rises then plateaus ($\sim$250 steps).}
    \label{fig:Task reward}
  \end{minipage}%
  \hfill
  \begin{minipage}[t]{0.175\linewidth}
    \centering
    \includegraphics[width=1\linewidth]{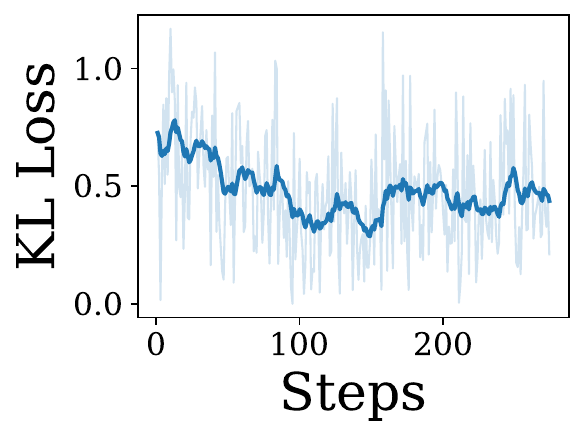}
    \vspace{-5mm}
    \caption{\tiny Reverse-KL $\mathcal{L}_{\text{revKL}}$ (Llama, Hotpot): drops early, stays bounded.}
    \label{fig:Reverse KL loss}
  \end{minipage}
  \vspace{-3mm}
\end{figure}

\stitle{Sensitivity to the Teacher Weight $\beta$.}
Having confirmed that teacher regularization is a key ingredient, we examine how strongly it should be applied.
Fig.~\ref{fig:hyper sensitivity} varies $\beta$ (Eq.~\eqref{eq:final_objective_beta}) on Hotpot-ENKB5 with the Qwen-4B backbone.
Language consistency and char 3-gram recall improve as $\beta$ increases from very small values and remain stable over a moderate range ($0.002$--$0.02$); the LLM-judge score, however, peaks at smaller $\beta$ and decreases monotonically as $\beta$ grows (\Cref{fig:hyper sensitivity-LLM-judge}), reflecting a metric-specific trade-off between teacher alignment and judge optimization. Overall, moderate $\beta$ provides useful token-level stabilization without overly restricting reward-driven exploration.
However, overly large $\beta$ (e.g., $0.2$) degrades performance, consistent with over-regularization that keeps the student too close to the frozen teacher and limits task-specific adaptation. This suggests a practical sweet spot where the teacher anchor is strong enough to prevent drift but light enough to allow reward-driven gains. Based on this sweep, we fix $\beta{=}0.02$ for every backbone--teacher pair in all main experiments (Appendix~\ref{sec:app:exp:exp:hparams}).

\stitle{Impact of Teacher Quality.}
Beyond the weight of the anchor, we ask whether \model\ depends on a particularly strong teacher.
We replace the default teacher with a weaker frozen base model of the same size (Fig.~\ref{fig:teacher_comparison_hotpotQA}; full results in Appendix~\ref{sec:app:exp:teacher_quality}).
\model\ remains consistently beneficial under the weak-teacher setting, while stronger teachers yield the best performance.
This suggests that teacher quality mainly affects the achievable ceiling rather than being a strict prerequisite for gains: even a modest teacher can stabilize on-policy updates and prevent language drift, making the recipe applicable when only same-scale models are available.

\stitle{Efficiency.}
\Cref{fig:qwen_llama_time,fig:qwen_2_time} compare end-to-end training time.
\model's only extra cost over standard RL is a single teacher forward pass on student-visited prefixes (never teacher rollouts).
With a compact or MoE teacher (\texttt{Qwen3-30B-A3B}) this is essentially free: wall-clock matches Naive-RL/OPD.
With a large \emph{dense} teacher (\texttt{Llama-3.3-70B}) the extra forward pass is visible in \Cref{fig:qwen_llama_time} but remains bounded (it is one forward pass, not additional generation), which we consider a worthwhile cost for the robustness gains reported below.

\stitle{Training Dynamics.}
\Cref{fig:Task reward} plots the training-time evolution of the mean task reward $r(x,y)$ over the $\sim$250 update steps of a run: it increases steadily and then plateaus, suggesting stable improvement of our multilingual objective. We evaluate the final checkpoint (Appendix~\ref{app:impl_details}).
\Cref{fig:Reverse KL loss} shows $\mathcal{L}_{\text{revKL}}$ drops early and remains bounded, indicating improved alignment with the frozen teacher on student-visited prefixes while still allowing reward-driven deviations.
Together, these curves match \model's design goal: improving task performance under a stable anchor without uncontrolled policy drift.

\stitle{Qualitative Case Study.}
\Cref{fig:case} makes the earlier metric gains concrete at the example level. Under English evidence, baseline systems exhibit two recurring failures: \emph{language drift}, where the answer begins in Indonesian but slips into copied English spans from the retrieved passages, and \emph{evidence misuse}, where the model latches onto a salient but unsupported entity instead of the one justified by the full evidence chain. By contrast, \model\ stays in fluent Indonesian, preserves the key lexical content needed for the answer, and selects the correct entity, suggesting that the teacher anchor improves both response-language control and evidence-faithful generation. This qualitative pattern mirrors the quantitative gains in language consistency, char 3-gram recall, and composite score. Parallel ja/ko/th/vi cases are in Appendix~\ref{sec:app:exp:case study}.

\begin{figure*}[!t]
   \centering
   \includegraphics[width=1\linewidth]{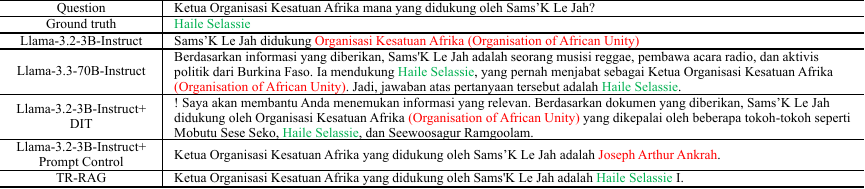}
   \vspace{-2mm}
   \caption{Qualitative case study (\texttt{Llama-3.2-3B-Instruct} backbone, Indonesian, Hotpot-ENKB5). Highlighting is consistent across panels: \emph{red} marks language drift (Indonesian slipping into copied English spans) or evidence-unsupported content, and \emph{green} marks correct, evidence-grounded content in the target language. Baseline systems drift and latch onto a salient but unsupported entity, whereas \model\ stays in fluent Indonesian and returns the entity justified by the full evidence chain. Best viewed in color.}
   \vspace{-2mm}
   \label{fig:case}
\end{figure*}

\subsection{Generalization to Naturally Multilingual MKQA}
\label{sec:exp:mkqa}

We evaluate on the natively multilingual \textbf{MKQA} setup defined in \Cref{sec:exp:setup}. We additionally report a \emph{cross-family} \texttt{Llama-3.3-70B-Instruct} judge to rule out within-family reward hacking.

\begin{table*}[t]
\centering
\footnotesize
\renewcommand{\arraystretch}{0.95}
\setlength{\tabcolsep}{7pt}
\caption{Generalization to \textbf{naturally multilingual MKQA} with English-evidence (\texttt{Llama-3.2-3B-Instruct} backbone). Per-language Composite, Qwen3-Next-80B Judge, and cross-family Llama-3.3-70B Judge are reported. \textbf{Bold}/\underline{underline} denote best/runner-up.
}
\label{tab:mkqa}
\begin{tabular}{l|rrrr|rrrr|r}
\hline
\multirow{2}{*}{Methods} &
\multicolumn{4}{c|}{In-Domain Languages} &
\multicolumn{4}{c|}{Out-of-Domain Languages} &
\multirow{2}{*}{ALL-AVG} \\
\cline{2-9}
& ko & th & vi & ID-AVG
& en & ja & no & OOD-AVG & \\
\hline
\multicolumn{10}{c}{\textit{Metric: Composite (\%)}} \\
\hline
Base     & 49.72 & 50.62 & 53.55 & 51.30 & 67.35 & 53.62 & 47.36 & 56.11 & 53.70 \\
Naive-RL & \textbf{50.80} & \underline{55.83} & \underline{58.00} & \underline{54.87} & \textbf{70.27} & \underline{55.54} & \underline{52.62} & \underline{59.48} & \underline{57.18} \\
\rowcolor{gray!15}
\model   & \underline{50.11} & \textbf{57.62} & \textbf{58.67} & \textbf{55.47} & \underline{69.57} & \textbf{55.95} & \textbf{55.00} & \textbf{60.17} & \textbf{57.82} \\
\hline
\multicolumn{10}{c}{\textit{Metric: Qwen3-Next-80B Judge (\%)}} \\
\hline
Base     & 58.35 & 63.58 & 73.13 & 65.02 & 68.96 & 62.93 & 67.16 & 66.35 & 65.69 \\
Naive-RL & \textbf{62.35} & \underline{71.45} & \textbf{78.50} & \textbf{70.77} & \textbf{73.10} & \underline{67.83} & \underline{74.38} & \underline{71.77} & \underline{71.27} \\
\rowcolor{gray!15}
\model   & \underline{60.05} & \textbf{74.40} & \underline{77.05} & \underline{70.50} & \underline{72.80} & \textbf{69.63} & \textbf{77.93} & \textbf{73.45} & \textbf{71.98} \\
\hline
\multicolumn{10}{c}{\textit{Metric: cross-family Llama-3.3-70B Judge (\%)}} \\
\hline
Base     & 52.50 & 62.85 & 71.77 & 62.37 & 70.86 & 60.70 & 63.80 & 65.12 & 63.75 \\
Naive-RL & \underline{57.90} & \underline{71.25} & \underline{77.95} & \underline{69.03} & \textbf{76.90} & \underline{67.35} & \underline{72.15} & \underline{72.13} & \underline{70.58} \\
\rowcolor{gray!15}
\model   & \textbf{58.10} & \textbf{74.60} & \textbf{79.15} & \textbf{70.62} & \underline{75.20} & \textbf{68.05} & \textbf{75.45} & \textbf{72.90} & \textbf{71.76} \\
\hline
\end{tabular}
\end{table*}

\stitle{Results.}
The trends carry over to native MKQA queries: \model\ wins ALL-AVG on Composite (\textbf{57.82} vs.\ 57.18), the Qwen judge (\textbf{71.98} vs.\ 71.27), and the cross-family Llama judge (\textbf{71.76} vs.\ 70.58), the wider Llama margin ruling out within-family reward hacking, while \emph{matching or slightly improving} on the judge-independent signals (LC 44.50 vs.\ 44.00, char 3-gram 43.05 vs.\ 42.85; Appendix~\ref{sec:app:exp:mkqa}). On the hardest split, Norwegian (base LC only 7\%), \model\ gains \textbf{+2.38}/\textbf{+3.55}/\textbf{+3.30} (composite/Qwen/Llama) over Naive-RL. Conversely, Naive-RL only edges \model\ in ``safe'' settings (e.g., the BioASQ Llama-3B judge, by 1--1.5pp) but collapses elsewhere: on Hotpot-ENKB5/Qwen-4B its Indonesian LC drops to \textbf{28.27\%} (vs.\ \textbf{55.60\%} for \model), and on Khmer (an extreme-script stress test beyond the MKQA languages in \Cref{sec:exp:setup}; App.~\ref{sec:app:exp:mkqa:lowres}), where \emph{both} methods remain capped at the base model's $18\%$ LC ceiling, \model\ is nonetheless slightly better on grounding ($+0.66$/$+0.50$/$+1.85$ on char 3-gram / Qwen / Llama judge over Naive-RL). The teacher anchor thus pays a small ``insurance premium'' in safe settings to prevent the large (up to $\sim$27pp) in-domain LC collapses that reward-only RL can suffer, and to keep improving grounding where reward-only RL stalls at the base ceiling on distant OOD languages (extended analysis in Appendix~\ref{sec:app:exp:robustness}).



\section{Conclusion}
\label{sec:conclusion}
We presented \model, a teacher-regularized RL post-training framework for \emph{English-evidence} cross-lingual RAG that couples reward optimization with prefix-wise reverse-KL anchoring on student-visited prefixes. The frozen-teacher anchor supplies dense stabilization without teacher rollouts, enabling compact students to improve language adherence and evidence-grounded correctness; \model\ achieves the best overall composite performance across three benchmarks and consistently outperforms strong baselines. Future work will extend \model\ to multimodal multilingual settings and explore richer reward designs.

% ...

\bibliographystyle{plainnat}
\bibliography{references,new_refs}

@misc{zhang2026selfevaluation,
  title = {Self-Evaluation Is Already There: Eliciting Latent Judge Calibration in Base {LLM}s with Minimal Data},
  author = {Zhang, XiuYu and Shan, Yi and Fang, Junfeng and Liang, Zhenkai},
  year = {2026},
  eprint = {2606.05122},
  archivePrefix = {arXiv},
  primaryClass = {cs.CL},
  url = {https://arxiv.org/abs/2606.05122}
}

@misc{pala2026grail,
  title = {{GRAIL}: Gradient-Reweighted Advantages for Reinforcement Learning with Verifiable Rewards},
  author = {Pala, Tej Deep and Toh, Vernon and Poria, Soujanya},
  year = {2026},
  eprint = {2606.04889},
  archivePrefix = {arXiv},
  primaryClass = {cs.CL},
  url = {https://arxiv.org/abs/2606.04889}
}

@misc{payoungkhamdee2026dudi,
  title = {{DuDi}: Dual-Signal Distillation with Cross-Lingual Verbalizer},
  author = {Payoungkhamdee, Patomporn and Udsa, Tinnakit and Ngui, Jian Gang and Nutanong, Sarana and Aji, Alham Fikri and Limkonchotiwat, Peerat},
  year = {2026},
  eprint = {2606.04694},
  archivePrefix = {arXiv},
  primaryClass = {cs.CL},
  url = {https://arxiv.org/abs/2606.04694}
}

@misc{nourbakhsh2026retrieval,
  title = {When Retrieval Doesn't Help: A Large-Scale Study of Biomedical {RAG}},
  author = {Nourbakhsh, Erfan and Slavin, Rocky and Yang, Ke and Rios, Anthony},
  year = {2026},
  eprint = {2606.04127},
  archivePrefix = {arXiv},
  primaryClass = {cs.CL},
  url = {https://arxiv.org/abs/2606.04127}
}

@misc{wang2026cherrl,
  title = {Reproducing, Analyzing, and Detecting Reward Hacking in Rubric-Based Reinforcement Learning},
  author = {Wang, Xuekang and Hao, Zhuoyuan and Hou, Shuo and Peng, Hao and Li, Juanzi and Wang, Xiaozhi},
  year = {2026},
  eprint = {2606.04923},
  archivePrefix = {arXiv},
  primaryClass = {cs.CL},
  url = {https://arxiv.org/abs/2606.04923}
}

@article{shao2024deepseekmath,
  title = {{DeepSeekMath}: Pushing the Limits of Mathematical Reasoning in Open Language Models},
  author = {Shao, Zhihong and Wang, Peiyi and Zhu, Qihao and Xu, Runxin and Song, Junxiao and Bi, Xiao and Zhang, Haowei and Zhang, Mingchuan and Li, Y. K. and Wu, Y. and Guo, Daya},
  journal = {arXiv preprint arXiv:2402.03300},
  year = {2024}
}

@inproceedings{joulin2017bag,
  title = {Bag of Tricks for Efficient Text Classification},
  author = {Joulin, Armand and Grave, Edouard and Bojanowski, Piotr and Mikolov, Tomas},
  booktitle = {Proceedings of the 15th Conference of the European Chapter of the Association for Computational Linguistics (EACL): Volume 2, Short Papers},
  pages = {427--431},
  year = {2017}
}

@misc{schulman2020kl,
  title = {Approximating {KL} Divergence},
  author = {Schulman, John},
  year = {2020},
  howpublished = {\url{http://joschu.net/blog/kl-approx.html}},
  note = {Blog post}
}

@inproceedings{rafailov2023direct,
  title = {Direct Preference Optimization: Your Language Model is Secretly a Reward Model},
  author = {Rafailov, Rafael and Sharma, Archit and Mitchell, Eric and Ermon, Stefano and Manning, Christopher D. and Finn, Chelsea},
  booktitle = {Advances in Neural Information Processing Systems (NeurIPS)},
  year = {2023}
}

@misc{zhang2026safetyneurons,
  title = {Who Transfers Safety? Identifying and Targeting Cross-Lingual Shared Safety Neurons},
  author = {Zhang, Xianhui and Xie, Chengyu and Zhu, Linxia and Yang, Yonghui and Zhao, Weixiang and Cheng, Zifeng and Wang, Cong and Shen, Fei and Chua, Tat-Seng},
  year = {2026},
  eprint = {2602.01283},
  archivePrefix = {arXiv},
  primaryClass = {cs.CL},
  url = {https://arxiv.org/abs/2602.01283}
}

@article{DBLP:journals/corr/abs-2407-01463,
  author       = {Nadezhda Chirkova and
                  David Rau and
                  Herv{\'{e}} D{\'{e}}jean and
                  Thibault Formal and
                  St{\'{e}}phane Clinchant and
                  Vassilina Nikoulina},
  title        = {Retrieval-augmented generation in multilingual settings},
  journal      = {CoRR},
  volume       = {abs/2407.01463},
  year         = {2024},
  doi          = {10.48550/ARXIV.2407.01463},
  eprinttype    = {arXiv},
  eprint       = {2407.01463},
  bibsource    = {dblp computer science bibliography, https://dblp.org}
}

@article{Qi2025OnTC,
  title={On the Consistency of Multilingual Context Utilization in Retrieval-Augmented Generation},
  author={Jirui Qi and Raquel Fern{\'a}ndez and Arianna Bisazza},
  journal={ArXiv},
  year={2025},
  volume={abs/2504.00597},
}

@article{DBLP:journals/corr/abs-2505-10089,
  author       = {Wei Liu and
                  Sony Trenous and
                  Leonardo F. R. Ribeiro and
                  Bill Byrne and
                  Felix Hieber},
  title        = {{XRAG:} Cross-lingual Retrieval-Augmented Generation},
  journal      = {CoRR},
  volume       = {abs/2505.10089},
  year         = {2025},
  doi          = {10.48550/ARXIV.2505.10089},
  eprinttype    = {arXiv},
  eprint       = {2505.10089},
  bibsource    = {dblp computer science bibliography, https://dblp.org}
}

@article{Park2025InvestigatingLP,
  title={Investigating Language Preference of Multilingual {RAG} Systems},
  author={Jeonghyun Park and Hwanhee Lee},
  journal={ArXiv},
  year={2025},
  volume={abs/2502.11175},
}

@inproceedings{Kamalloo2023EvaluatingOQ,
  title={Evaluating Open-Domain Question Answering in the Era of Large Language Models},
  author={Ehsan Kamalloo and Nouha Dziri and Charles L. A. Clarke and Davood Rafiei},
  booktitle={Annual Meeting of the Association for Computational Linguistics},
  year={2023},
}

@article{Gao2023RetrievalAugmentedGF,
  title={Retrieval-Augmented Generation for Large Language Models: A Survey},
  author={Yunfan Gao and Yun Xiong and Xinyu Gao and Kangxiang Jia and Jinliu Pan and Yuxi Bi and Yi Dai and Jiawei Sun and Qianyu Guo and Meng Wang and Haofen Wang},
  journal={ArXiv},
  year={2023},
  volume={abs/2312.10997},
}

@article{li2025language,
  title={Language Drift in Multilingual Retrieval-Augmented Generation: Characterization and Decoding-Time Mitigation},
  author={Li, Bo and Xu, Zhenghua and Xie, Rui},
  journal={arXiv preprint arXiv:2511.09984},
  year={2025}
}

@inproceedings{asai2021xor,
  title={{XOR QA}: Cross-lingual Open-Retrieval Question Answering},
  author={Asai, Akari and Kasai, Jungo and Clark, Jonathan H and Lee, Kenton and Choi, Eunsol and Hajishirzi, Hannaneh},
  booktitle={Proceedings of the 2021 conference of the North American chapter of the association for computational linguistics: human language technologies},
  pages={547--564},
  year={2021}
}

@article{wei2022chain,
  title={Chain-of-thought prompting elicits reasoning in large language models},
  author={Wei, Jason and Wang, Xuezhi and Schuurmans, Dale and Bosma, Maarten and Xia, Fei and Chi, Ed and Le, Quoc V and Zhou, Denny and others},
  journal={Advances in neural information processing systems},
  volume={35},
  pages={24824--24837},
  year={2022}
}

@inproceedings{DBLP:conf/iclr/AgarwalVZSGGB24,
  author       = {Rishabh Agarwal and
                  Nino Vieillard and
                  Yongchao Zhou and
                  Piotr Stanczyk and
                  Sabela Ramos Garea and
                  Matthieu Geist and
                  Olivier Bachem},
  title        = {On-Policy Distillation of Language Models: Learning from Self-Generated
                  Mistakes},
  booktitle    = {The Twelfth International Conference on Learning Representations,
                  {ICLR} 2024, Vienna, Austria, May 7-11, 2024},
  publisher    = {OpenReview.net},
  year         = {2024},
  bibsource    = {dblp computer science bibliography, https://dblp.org}
}

@inproceedings{DBLP:conf/nips/LewisPPPKGKLYR020,
  author       = {Patrick Lewis and
                  Ethan Perez and
                  Aleksandra Piktus and
                  Fabio Petroni and
                  Vladimir Karpukhin and
                  Naman Goyal and
                  Heinrich K{\"{u}}ttler and
                  Mike Lewis and
                  Wen{-}tau Yih and
                  Tim Rockt{\"{a}}schel and
                  Sebastian Riedel and
                  Douwe Kiela},
  editor       = {Hugo Larochelle and
                  Marc'Aurelio Ranzato and
                  Raia Hadsell and
                  Maria{-}Florina Balcan and
                  Hsuan{-}Tien Lin},
  title        = {Retrieval-Augmented Generation for Knowledge-Intensive {NLP} Tasks},
  booktitle    = {Advances in Neural Information Processing Systems 33: Annual Conference
                  on Neural Information Processing Systems 2020, NeurIPS 2020, December
                  6-12, 2020, virtual},
  year         = {2020},
  bibsource    = {dblp computer science bibliography, https://dblp.org}
}

@article{DBLP:journals/tacl/ClarkPNCGCK20,
  author       = {Jonathan H. Clark and
                  Jennimaria Palomaki and
                  Vitaly Nikolaev and
                  Eunsol Choi and
                  Dan Garrette and
                  Michael Collins and
                  Tom Kwiatkowski},
  title        = {TyDi {QA:} {A} Benchmark for Information-Seeking Question Answering
                  in Typologically Diverse Languages},
  journal      = {Trans. Assoc. Comput. Linguistics},
  volume       = {8},
  pages        = {454--470},
  year         = {2020},
}

@article{DBLP:journals/tacl/LongpreLD21,
  author       = {Shayne Longpre and
                  Yi Lu and
                  Joachim Daiber},
  title        = {{MKQA:} {A} Linguistically Diverse Benchmark for Multilingual Open
                  Domain Question Answering},
  journal      = {Trans. Assoc. Comput. Linguistics},
  volume       = {9},
  pages        = {1389--1406},
  year         = {2021},
}

@article{zhang-etal-2023-miracl,
    title = "{MIRACL}: A Multilingual Retrieval Dataset Covering 18 Diverse Languages",
    author = "Zhang, Xinyu  and
      Thakur, Nandan  and
      Ogundepo, Odunayo  and
      Kamalloo, Ehsan  and
      Alfonso-Hermelo, David  and
      Li, Xiaoguang  and
      Liu, Qun  and
      Rezagholizadeh, Mehdi  and
      Lin, Jimmy",
    journal = "Transactions of the Association for Computational Linguistics",
    volume = "11",
    year = "2023",
    address = "Cambridge, MA",
    publisher = "MIT Press",
    doi = "10.1162/tacl_a_00595",
    pages = "1114--1131",
}

@inproceedings{DBLP:conf/nips/AsaiYKH21,
  author       = {Akari Asai and
                  Xinyan Yu and
                  Jungo Kasai and
                  Hanna Hajishirzi},
  editor       = {Marc'Aurelio Ranzato and
                  Alina Beygelzimer and
                  Yann N. Dauphin and
                  Percy Liang and
                  Jennifer Wortman Vaughan},
  title        = {One Question Answering Model for Many Languages with Cross-lingual
                  Dense Passage Retrieval},
  booktitle    = {Advances in Neural Information Processing Systems 34: Annual Conference
                  on Neural Information Processing Systems 2021, NeurIPS 2021, December
                  6-14, 2021, virtual},
  pages        = {7547--7560},
  year         = {2021},
  bibsource    = {dblp computer science bibliography, https://dblp.org}
}

@article{DBLP:journals/corr/abs-2504-03616,
  author       = {Leonardo Ranaldi and
                  Barry Haddow and
                  Alexandra Birch},
  title        = {Multilingual Retrieval-Augmented Generation for Knowledge-Intensive
                  Task},
  journal      = {CoRR},
  volume       = {abs/2504.03616},
  year         = {2025},
  doi          = {10.48550/ARXIV.2504.03616},
  eprinttype    = {arXiv},
  eprint       = {2504.03616},
  bibsource    = {dblp computer science bibliography, https://dblp.org}
}

@inproceedings{DBLP:conf/bea/KulmizevBBNNPW17,
  author       = {Artur Kulmizev and
                  Bo Blankers and
                  Johannes Bjerva and
                  Malvina Nissim and
                  Gertjan van Noord and
                  Barbara Plank and
                  Martijn Wieling},
  editor       = {Joel R. Tetreault and
                  Jill Burstein and
                  Claudia Leacock and
                  Helen Yannakoudakis},
  title        = {The Power of Character N-grams in Native Language Identification},
  booktitle    = {Proceedings of the 12th Workshop on Innovative Use of {NLP} for Building
                  Educational Applications, BEA@EMNLP 2017, Copenhagen, Denmark, September
                  8, 2017},
  pages        = {382--389},
  publisher    = {Association for Computational Linguistics},
  year         = {2017},
  doi          = {10.18653/V1/W17-5043},
  bibsource    = {dblp computer science bibliography, https://dblp.org}
}

@inproceedings{popovic-2015-chrf,
    title = "chr{F}: character n-gram {F}-score for automatic {MT} evaluation",
    author = "Popovi{\'c}, Maja",
    editor = "Bojar, Ond{\v{r}}ej  and
      Chatterjee, Rajen  and
      Federmann, Christian  and
      Haddow, Barry  and
      Hokamp, Chris  and
      Huck, Matthias  and
      Logacheva, Varvara  and
      Pecina, Pavel",
    booktitle = "Proceedings of the Tenth Workshop on Statistical Machine Translation",
    month = sep,
    year = "2015",
    address = "Lisbon, Portugal",
    publisher = "Association for Computational Linguistics",
    doi = "10.18653/v1/W15-3049",
    pages = "392--395"
}

@inproceedings{DBLP:conf/nips/ZhengC00WZL0LXZ23,
  author       = {Lianmin Zheng and
                  Wei{-}Lin Chiang and
                  Ying Sheng and
                  Siyuan Zhuang and
                  Zhanghao Wu and
                  Yonghao Zhuang and
                  Zi Lin and
                  Zhuohan Li and
                  Dacheng Li and
                  Eric P. Xing and
                  Hao Zhang and
                  Joseph E. Gonzalez and
                  Ion Stoica},
  editor       = {Alice Oh and
                  Tristan Naumann and
                  Amir Globerson and
                  Kate Saenko and
                  Moritz Hardt and
                  Sergey Levine},
  title        = {Judging {LLM}-as-a-Judge with {MT-Bench} and Chatbot Arena},
  booktitle    = {Advances in Neural Information Processing Systems 36: Annual Conference
                  on Neural Information Processing Systems 2023, NeurIPS 2023, New Orleans,
                  LA, USA, December 10 - 16, 2023},
  year         = {2023},
  bibsource    = {dblp computer science bibliography, https://dblp.org}
}

@article{DBLP:journals/corr/abs-2501-12948,
  author       = {DeepSeek{-}AI},
  title        = {{DeepSeek-R1}: Incentivizing Reasoning Capability in {LLMs} via Reinforcement
                  Learning},
  journal      = {CoRR},
  volume       = {abs/2501.12948},
  year         = {2025},
  doi          = {10.48550/ARXIV.2501.12948},
  eprinttype    = {arXiv},
  eprint       = {2501.12948},
  bibsource    = {dblp computer science bibliography, https://dblp.org}
}

@article{DBLP:journals/corr/HintonVD15,
  author       = {Geoffrey E. Hinton and
                  Oriol Vinyals and
                  Jeffrey Dean},
  title        = {Distilling the Knowledge in a Neural Network},
  journal      = {CoRR},
  volume       = {abs/1503.02531},
  year         = {2015},
  eprinttype    = {arXiv},
  eprint       = {1503.02531},
  bibsource    = {dblp computer science bibliography, https://dblp.org}
}

@article{DBLP:journals/corr/abs-2504-18428,
  author       = {Yiming Wang and
                  Pei Zhang and
                  Jialong Tang and
                  Haoran Wei and
                  Baosong Yang and
                  Rui Wang and
                  Chenshu Sun and
                  Feitong Sun and
                  Jiran Zhang and
                  Junxuan Wu and
                  Qiqian Cang and
                  Yichang Zhang and
                  Fei Huang and
                  Junyang Lin and
                  Fei Huang and
                  Jingren Zhou},
  title        = {PolyMath: Evaluating Mathematical Reasoning in Multilingual Contexts},
  journal      = {CoRR},
  volume       = {abs/2504.18428},
  year         = {2025},
  doi          = {10.48550/ARXIV.2504.18428},
  eprinttype    = {arXiv},
  eprint       = {2504.18428},
  bibsource    = {dblp computer science bibliography, https://dblp.org}
}

@article{tsatsaronis2015overview,
  title={An overview of the BIOASQ large-scale biomedical semantic indexing and question answering competition},
  author={Tsatsaronis, George and Balikas, Georgios and Malakasiotis, Prodromos and Partalas, Ioannis and Zschunke, Matthias and Alvers, Michael R and Weissenborn, Dirk and Krithara, Anastasia and Petridis, Sergios and Polychronopoulos, Dimitris and others},
  journal={BMC bioinformatics},
  volume={16},
  number={1},
  pages={138},
  year={2015},
  publisher={Springer}
}

@inproceedings{yang2018hotpotqa,
  title={HotpotQA: A dataset for diverse, explainable multi-hop question answering},
  author={Yang, Zhilin and Qi, Peng and Zhang, Saizheng and Bengio, Yoshua and Cohen, William and Salakhutdinov, Ruslan and Manning, Christopher D},
  booktitle={Proceedings of the 2018 conference on empirical methods in natural language processing},
  pages={2369--2380},
  year={2018}
}

@article{yang2025qwen3,
  title={Qwen3 technical report},
  author={Yang, An and Li, Anfeng and Yang, Baosong and Zhang, Beichen and Hui, Binyuan and Zheng, Bo and Yu, Bowen and Gao, Chang and Huang, Chengen and Lv, Chenxu and others},
  journal={arXiv preprint arXiv:2505.09388},
  year={2025}
}

@article{dubey2024llama,
  title={The llama 3 herd of models},
  author={Dubey, Abhimanyu and Jauhri, Abhinav and Pandey, Abhinav and Kadian, Abhishek and Al-Dahle, Ahmad and Letman, Aiesha and Mathur, Akhil and Schelten, Alan and Yang, Amy and Fan, Angela and others},
  journal={arXiv preprint arXiv:2407.21783},
  year={2024}
}

@article{luo2025mmath,
  title={MMATH: A Multilingual Benchmark for Mathematical Reasoning},
  author={Luo, Wenyang and Zhao, Wayne Xin and Sha, Jing and Wang, Shijin and Wen, Ji-Rong},
  journal={arXiv preprint arXiv:2505.19126},
  year={2025}
}

@inproceedings{yang-etal-2024-self,
    title = "Self-Distillation Bridges Distribution Gap in Language Model Fine-Tuning",
    author = "Yang, Zhaorui  and
      Pang, Tianyu  and
      Feng, Haozhe  and
      Wang, Han  and
      Chen, Wei  and
      Zhu, Minfeng  and
      Liu, Qian",
    editor = "Ku, Lun-Wei  and
      Martins, Andre  and
      Srikumar, Vivek",
    booktitle = "Proceedings of the 62nd Annual Meeting of the Association for Computational Linguistics (Volume 1: Long Papers)",
    month = aug,
    year = "2024",
    address = "Bangkok, Thailand",
    publisher = "Association for Computational Linguistics",
    doi = "10.18653/v1/2024.acl-long.58",
    pages = "1028--1043"
}

@article{qiu2025gated,
  title={Gated Attention for Large Language Models: Non-linearity, Sparsity, and Attention-Sink-Free},
  author={Qiu, Zihan and Wang, Zekun and Zheng, Bo and Huang, Zeyu and Wen, Kaiyue and Yang, Songlin and Men, Rui and Yu, Le and Huang, Fei and Huang, Suozhi and others},
  journal={arXiv preprint arXiv:2505.06708},
  year={2025}
}

@inproceedings{grave2018learning,
  title={Learning word vectors for 157 languages},
  author={Grave, Edouard and Bojanowski, Piotr and Gupta, Prakhar and Joulin, Armand and Mikolov, Tom{\'a}{\v{s}}},
  booktitle={Proceedings of the eleventh international conference on language resources and evaluation (LREC 2018)},
  year={2018}
}

@article{zhang2025think,
  title={Think Natively: Unlocking Multilingual Reasoning with Consistency-Enhanced Reinforcement Learning},
  author={Zhang, Xue and Liang, Yunlong and Meng, Fandong and Zhang, Songming and Huang, Kaiyu and Chen, Yufeng and Xu, Jinan and Zhou, Jie},
  journal={arXiv preprint arXiv:2510.07300},
  year={2025}
}

@article{schulman2017proximal,
  title={Proximal policy optimization algorithms},
  author={Schulman, John and Wolski, Filip and Dhariwal, Prafulla and Radford, Alec and Klimov, Oleg},
  journal={arXiv preprint arXiv:1707.06347},
  year={2017}
}

@article{ouyang2022training,
  title={Training language models to follow instructions with human feedback},
  author={Ouyang, Long and Wu, Jeffrey and Jiang, Xu and Almeida, Diogo and Wainwright, Carroll and Mishkin, Pamela and Zhang, Chong and Agarwal, Sandhini and Slama, Katarina and Ray, Alex and others},
  journal={Advances in neural information processing systems},
  volume={35},
  pages={27730--27744},
  year={2022}
}

@inproceedings{popovic-2017-chrf,
    title = "chr{F}++: words helping character n-grams",
    author = "Popovi{\'c}, Maja",
    editor = "Bojar, Ond{\v{r}}ej  and
      Buck, Christian  and
      Chatterjee, Rajen  and
      Federmann, Christian  and
      Graham, Yvette  and
      Haddow, Barry  and
      Huck, Matthias  and
      Yepes, Antonio Jimeno  and
      Koehn, Philipp  and
      Kreutzer, Julia",
    booktitle = "Proceedings of the Second Conference on Machine Translation",
    month = sep,
    year = "2017",
    address = "Copenhagen, Denmark",
    publisher = "Association for Computational Linguistics",
    doi = "10.18653/v1/W17-4770",
    pages = "612--618"
}

@inproceedings{hu2020xtreme,
  title={Xtreme: A massively multilingual multi-task benchmark for evaluating cross-lingual generalisation},
  author={Hu, Junjie and Ruder, Sebastian and Siddhant, Aditya and Neubig, Graham and Firat, Orhan and Johnson, Melvin},
  booktitle={International conference on machine learning},
  pages={4411--4421},
  year={2020},
  organization={PMLR}
}

@inproceedings{lewis2020mlqa,
  title={{MLQA}: Evaluating Cross-Lingual Extractive Question Answering},
  author={Lewis, Patrick and Oguz, Barlas and Rinott, Ruty and Riedel, Sebastian and Schwenk, Holger},
  booktitle={Proceedings of the 58th annual meeting of the association for computational linguistics},
  pages={7315--7330},
  year={2020}
}

@article{sanh2019distilbert,
  title={{DistilBERT}, a Distilled Version of {BERT}: Smaller, Faster, Cheaper and Lighter},
  author={Sanh, Victor and Debut, Lysandre and Chaumond, Julien and Wolf, Thomas},
  journal={arXiv preprint arXiv:1910.01108},
  year={2019}
}

@article{sun2019patient,
  title={Patient Knowledge Distillation for {BERT} Model Compression},
  author={Sun, Siqi and Cheng, Yu and Gan, Zhe and Liu, Jingjing},
  journal={arXiv preprint arXiv:1908.09355},
  year={2019}
}

@inproceedings{jiao2020tinybert,
  title={{TinyBERT}: Distilling {BERT} for Natural Language Understanding},
  author={Jiao, Xiaoqi and Yin, Yichun and Shang, Lifeng and Jiang, Xin and Chen, Xiao and Li, Linlin and Wang, Fang and Liu, Qun},
  booktitle={Findings of the association for computational linguistics: EMNLP 2020},
  pages={4163--4174},
  year={2020}
}

@article{wang2020minilm,
  title={{MiniLM}: Deep Self-Attention Distillation for Task-Agnostic Compression of Pre-Trained Transformers},
  author={Wang, Wenhui and Wei, Furu and Dong, Li and Bao, Hangbo and Yang, Nan and Zhou, Ming},
  journal={Advances in neural information processing systems},
  volume={33},
  pages={5776--5788},
  year={2020}
}

@inproceedings{kim2016sequence,
  title={Sequence-level knowledge distillation},
  author={Kim, Yoon and Rush, Alexander M},
  booktitle={Proceedings of the 2016 conference on empirical methods in natural language processing},
  pages={1317--1327},
  year={2016}
}

@article{gu2023minillm,
  title={{MiniLLM}: Knowledge Distillation of Large Language Models},
  author={Gu, Yuxian and Dong, Li and Wei, Furu and Huang, Minlie},
  journal={arXiv preprint arXiv:2306.08543},
  year={2023}
}

@article{bengio2015scheduled,
  title={Scheduled sampling for sequence prediction with recurrent neural networks},
  author={Bengio, Samy and Vinyals, Oriol and Jaitly, Navdeep and Shazeer, Noam},
  journal={Advances in neural information processing systems},
  volume={28},
  year={2015}
}

@article{ranzato2015sequence,
  title={Sequence level training with recurrent neural networks},
  author={Ranzato, Marc'Aurelio and Chopra, Sumit and Auli, Michael and Zaremba, Wojciech},
  journal={arXiv preprint arXiv:1511.06732},
  year={2015}
}

@article{williams1992simple,
  title={Simple statistical gradient-following algorithms for connectionist reinforcement learning},
  author={Williams, Ronald J},
  journal={Machine learning},
  volume={8},
  number={3},
  pages={229--256},
  year={1992},
  publisher={Springer}
}

@article{sutton1999policy,
  title={Policy gradient methods for reinforcement learning with function approximation},
  author={Sutton, Richard S and McAllester, David and Singh, Satinder and Mansour, Yishay},
  journal={Advances in neural information processing systems},
  volume={12},
  year={1999}
}

@article{kazemnejad2024vineppo,
  title={{VinePPO}: Refining Credit Assignment in {RL} Training of {LLMs}},
  author={Kazemnejad, Amirhossein and Aghajohari, Milad and Portelance, Eva and Sordoni, Alessandro and Reddy, Siva and Courville, Aaron and Roux, Nicolas Le},
  journal={arXiv preprint arXiv:2410.01679},
  year={2024}
}

@inproceedings{DBLP:conf/emnlp/HuWLXLB0PL23,
  author       = {Zhiqiang Hu and
                  Lei Wang and
                  Yihuai Lan and
                  Wanyu Xu and
                  Ee{-}Peng Lim and
                  Lidong Bing and
                  Xing Xu and
                  Soujanya Poria and
                  Roy Ka{-}Wei Lee},
  editor       = {Houda Bouamor and
                  Juan Pino and
                  Kalika Bali},
  title        = {LLM-Adapters: An Adapter Family for Parameter-Efficient Fine-Tuning
                  of Large Language Models},
  booktitle    = {Proceedings of the 2023 Conference on Empirical Methods in Natural
                  Language Processing, {EMNLP} 2023, Singapore, December 6-10, 2023},
  pages        = {5254--5276},
  publisher    = {Association for Computational Linguistics},
  year         = {2023},
  doi          = {10.18653/V1/2023.EMNLP-MAIN.319},
}

@article{DBLP:journals/corr/abs-2409-19256,
  author       = {Guangming Sheng and
                  Chi Zhang and
                  Zilingfeng Ye and
                  Xibin Wu and
                  Wang Zhang and
                  Ru Zhang and
                  Yanghua Peng and
                  Haibin Lin and
                  Chuan Wu},
  title        = {HybridFlow: {A} Flexible and Efficient {RLHF} Framework},
  journal      = {CoRR},
  volume       = {abs/2409.19256},
  year         = {2024},
  doi          = {10.48550/ARXIV.2409.19256},
  eprinttype    = {arXiv},
  eprint       = {2409.19256},
}

@inproceedings{DBLP:conf/sosp/KwonLZ0ZY0ZS23,
  author       = {Woosuk Kwon and
                  Zhuohan Li and
                  Siyuan Zhuang and
                  Ying Sheng and
                  Lianmin Zheng and
                  Cody Hao Yu and
                  Joseph Gonzalez and
                  Hao Zhang and
                  Ion Stoica},
  editor       = {Jason Flinn and
                  Margo I. Seltzer and
                  Peter Druschel and
                  Antoine Kaufmann and
                  Jonathan Mace},
  title        = {Efficient Memory Management for Large Language Model Serving with
                  PagedAttention},
  booktitle    = {Proceedings of the 29th Symposium on Operating Systems Principles,
                  {SOSP} 2023, Koblenz, Germany, October 23-26, 2023},
  pages        = {611--626},
  publisher    = {{ACM}},
  year         = {2023},
  doi          = {10.1145/3600006.3613165},
}

@article{qi2026lcrl,
  title={Language-Coupled Reinforcement Learning for Multilingual Retrieval-Augmented Generation},
  author={Qi, Rui and Mo, Fengran and Chen, Yufeng and Zhang, Xue and Wang, Shuo and Li, Hongliang and Xu, Jinan and Jiang, Meng and Nie, Jian-Yun and Huang, Kaiyu},
  journal={arXiv preprint arXiv:2601.14896},
  year={2026}
}

@article{qi2026crosearch,
  title={{CroSearch-R1}: Better Leveraging Cross-lingual Knowledge for Retrieval-Augmented Generation},
  author={Qi, Rui and Mo, Fengran and Lu, Sijin and Chen, Yufeng and Nie, Jian-Yun and Huang, Kaiyu},
  journal={arXiv preprint arXiv:2604.25182},
  year={2026}
}

@article{zhang2026rlad,
  title={Reinforcement-Aware Knowledge Distillation for {LLM} Reasoning},
  author={Zhang, Zhaoyang and Jiang, Shuli and Shen, Yantao and Zhang, Yuting and Ram, Dhananjay and Yang, Shuo and Tu, Zhuowen and Xia, Wei and Soatto, Stefano},
  journal={arXiv preprint arXiv:2602.22495},
  year={2026}
}

@article{huang2025ragrl,
  title={{RAG-RL}: Advancing Retrieval-Augmented Generation via {RL} and Curriculum Learning},
  author={Huang, Jerry and Madala, Siddarth and Sidhu, Risham and Niu, Cheng and Peng, Hao and Hockenmaier, Julia and Zhang, Tong},
  journal={arXiv preprint arXiv:2503.12759},
  year={2025}
}

@article{zhao2026opsd,
  title={Self-Distilled Reasoner: On-Policy Self-Distillation for Large Language Models},
  author={Zhao, Siyan and Xie, Zhihui and Liu, Mengchen and Huang, Jing and Pang, Guan and Chen, Feiyu and Grover, Aditya},
  journal={arXiv preprint arXiv:2601.18734},
  year={2026}
}

@article{hwang2025bridge,
  title={Learn Globally, Speak Locally: Bridging the Gaps in Multilingual Reasoning},
  author={Hwang, Jaedong and Tanmay, Kumar and Lee, Seok-Jin and Agrawal, Ayush and Palangi, Hamid and Ayush, Kumar and Fiete, Ila and Liang, Paul Pu},
  journal={arXiv preprint arXiv:2507.05418},
  year={2025}
}

@article{elhady2026xlsc,
  title={Cross-lingual Self-Consistency for Multilingual Reasoning with Language Models},
  author={Elhady, Ahmed and Agirre, Eneko and Artetxe, Mikel},
  journal={arXiv preprint arXiv:2606.01464},
  year={2026}
}

@article{shen2026rlkd,
  title={Reinforcement Learning-based Knowledge Distillation with {LLM}-as-a-Judge},
  author={Shen, Yiyang and Tu, Lifu and Wang, Weiran},
  journal={arXiv preprint arXiv:2604.02621},
  year={2026}
}


\newpage
\appendix
\section*{Appendices}

\section{Limitations and Future Work}
\label{sec:limitations}

We focus on a specific but practically important setting: \emph{English-evidence} cross-lingual generation, where questions and answers are non-English but retrieved evidence is English. This leaves several extensions for future work, including multilingual or mixed-language evidence, cross-lingual retrieval, and tighter joint optimization of retrieval and generation. Our method also relies on stronger teachers and an LLM-based judge during training, which introduces additional cost and potential evaluator bias even though neither is required at inference time. One specific overlap deserves note: for the Llama backbone, our cross-family evaluation judge (\texttt{Llama-3.3-70B-Instruct}) coincides with the distillation teacher, so its scores are not fully independent of the teacher anchor. We partly control for this by reporting judge-independent signals (language consistency and character 3-gram recall) and by the Qwen backbone, whose teacher and primary judge come from different families; a fully neutral third-family judge (e.g., Gemma or Mistral) on an MKQA subset would further strengthen the evaluation and is left to future work. In addition, part of our training data is derived from translated benchmarks, so broader validation on natively authored cross-lingual RAG datasets would strengthen the conclusions. Finally, our current language-consistency reward is intentionally simple, and richer language-control signals may further improve robustness, especially for low-resource languages.
\section{Extended Proofs for Teacher-Regularized RL}
\label{app:proofs}

\subsection{Notation and Setup}
Let $x$ denote an input instance (query--evidence pair) drawn together with a reference answer $a^\star$ from a training distribution $\mathcal{B}$.
Given $x$, the student policy $\pi_\theta$ generates an output sequence
$y=(y_1,\ldots,y_{L(y)})$ with a (possibly trajectory-dependent) length $L(y)$, using the autoregressive factorization
\begin{equation}
\pi_\theta(y\mid x)=\prod_{t=1}^{L(y)}\pi_\theta(y_t\mid s_t),
\qquad s_t=(x,y_{<t}).
\end{equation}
Here $s_t$ denotes the prefix state induced by the realized student trajectory.
The teacher policy $\pi_T$ is \emph{frozen} and is evaluated on the same student-visited prefix states,
i.e., we query $\pi_T(\cdot\mid s_t)$ for $s_t=(x,y_{<t})$ where $y\sim\pi_\theta$.

We use a length-normalized (token-mean) reverse-KL teacher anchor along the student trajectory:
\begin{equation}
\mathcal{L}_{\mathrm{revKL}}(x,y)
=
\frac{1}{L(y)}\sum_{t=1}^{L(y)}
\KL\!\Big(\pi_\theta(\cdot\mid s_t)\ \big\|\ \pi_T(\cdot\mid s_t)\Big).
\label{eq:app:token_mean_revkl}
\end{equation}
This normalization decouples the regularization strength from output length, while remaining well-defined when $L(y)$ varies (e.g., EOS-terminated generation).

Our teacher-regularized objective is
\begin{equation}
\max_{\theta}\;
\tilde{J}(\theta)
=
\mathbb{E}_{(x,a^\star)\sim\mathcal{B}}
\mathbb{E}_{y\sim\pi_\theta(\cdot\mid x)}
\Big[
r(x,y)-\beta\,\mathcal{L}_{\mathrm{revKL}}(x,y)
\Big],
\label{eq:app:main_obj_beta}
\end{equation}
where $r(x,y)$ is the task reward and $\beta>0$ controls the strength of the teacher anchor.
All statements below are written to remain valid when $L(y)$ is random.

\subsection{Reverse KL Decomposition: Teacher Likelihood + Entropy}
\label{app:revkl_decomp}

\begin{lemma}[Per-token reverse KL identity]
\label{lem:revkl_identity}
For any prefix state $s$ and distributions $\pi(\cdot\mid s)$ and $\pi_T(\cdot\mid s)$,
\begin{equation}
-\KL\!\big(\pi(\cdot\mid s)\ \big\|\ \pi_T(\cdot\mid s)\big)
=
\mathbb{E}_{a\sim\pi(\cdot\mid s)}[\log \pi_T(a\mid s)]
+\mathcal{H}\!\big(\pi(\cdot\mid s)\big),
\label{eq:app:revkl_identity}
\end{equation}
where $\mathcal{H}(\pi)=-\mathbb{E}_{a\sim\pi}[\log \pi(a)]$ is entropy.
\end{lemma}

\begin{proof}
By definition,
$\KL(\pi\|\pi_T)=\mathbb{E}_{a\sim\pi}[\log\pi(a)-\log\pi_T(a)]$.
Negating both sides gives
$-\KL(\pi\|\pi_T)=\mathbb{E}_{a\sim\pi}[\log\pi_T(a)]-\mathbb{E}_{a\sim\pi}[\log\pi(a)]
=\mathbb{E}_{a\sim\pi}[\log\pi_T(a)]+\mathcal{H}(\pi)$.
\end{proof}

\begin{corollary}[Token-mean reverse KL decomposition on student prefixes]
\label{cor:token_mean_decomp}
For a sampled trajectory $y\sim\pi_\theta(\cdot\mid x)$,
\begin{equation}
-\mathcal{L}_{\mathrm{revKL}}(x,y)
=
\frac{1}{L(y)}\sum_{t=1}^{L(y)}
\Big(
\mathbb{E}_{a\sim\pi_\theta(\cdot\mid s_t)}[\log \pi_T(a\mid s_t)]
+\mathcal{H}(\pi_\theta(\cdot\mid s_t))
\Big).
\label{eq:app:token_mean_revkl_decomp}
\end{equation}
\end{corollary}

\paragraph{Implication.}
Eq.~\eqref{eq:app:token_mean_revkl_decomp} shows the teacher anchor has two effects:
(i) it increases \emph{teacher-consistency} by encouraging student actions that the teacher assigns high probability;
(ii) it adds a \emph{local entropy bonus} that stabilizes exploration under reward optimization.

\paragraph{Important note on realizations vs.\ expectations.}
In general, when $L(y)$ is trajectory-dependent, one should \emph{not} claim that
\[
\mathbb{E}_{y\sim\pi_\theta}\!\left[\frac{1}{L(y)}\sum_{t=1}^{L(y)}\mathbb{E}_{a\sim\pi_\theta(\cdot\mid s_t)}[\log \pi_T(a\mid s_t)]\right]
\]
is exactly equivalent to replacing the inner expectation by the realized token $\log\pi_T(y_t\mid s_t)$.
The realized-token form is a \emph{Monte Carlo estimator} of the per-state expectation,
but the presence of the random normalization $1/L(y)$ means additional care is needed when rewriting objectives.
We therefore use the decomposition in Eq.~\eqref{eq:app:token_mean_revkl_decomp} as the mathematically correct identity,
and treat realized-token quantities only as estimators when discussing implementation.

\subsection{Closed-form Optimum in a Simplified Bandit View (Sum-KL Variant)}
\label{app:closed_form}

To build intuition, consider a contextual bandit abstraction where an ``action'' is the entire sequence $y$,
and consider a \emph{sum-KL} regularizer (no length normalization).
Define a sequence-level teacher ``prior'' by evaluating the teacher on the same prefix states induced by $y$, i.e.,
\begin{equation}
\pi_T(y\mid x)\triangleq \prod_{t=1}^{L(y)}\pi_T(y_t\mid s_t).
\label{eq:app:teacher_seq_prior}
\end{equation}

Define the sequence-level reverse KL
\begin{equation}
\KL\!\big(\pi(\cdot\mid x)\ \big\|\ \pi_T(\cdot\mid x)\big)
=
\mathbb{E}_{y\sim\pi(\cdot\mid x)}
\big[\log \pi(y\mid x)-\log \pi_T(y\mid x)\big].
\label{eq:app:seq_kl_def}
\end{equation}

\begin{lemma}[Autoregressive KL chain rule]
\label{lem:kl_chain_rule}
If $\pi$ and $\pi_T$ share the same autoregressive factorization on the same prefix states,
then the sequence-level KL equals the expected \emph{sum} of per-token KL terms:
\begin{equation}
\KL\!\big(\pi(\cdot\mid x)\ \big\|\ \pi_T(\cdot\mid x)\big)
=
\mathbb{E}_{y\sim\pi(\cdot\mid x)}
\sum_{t=1}^{L(y)}\KL\!\big(\pi(\cdot\mid s_t)\ \big\|\ \pi_T(\cdot\mid s_t)\big).
\label{eq:app:kl_chain_rule}
\end{equation}
\end{lemma}

\begin{proof}
By the shared autoregressive factorization on identical prefix states $s_t=(x,y_{<t})$,
\begin{equation*}
\log\frac{\pi(y\mid x)}{\pi_T(y\mid x)}
=\sum_{t=1}^{L(y)}\log\frac{\pi(y_t\mid s_t)}{\pi_T(y_t\mid s_t)}.
\end{equation*}
Taking $\mathbb{E}_{y\sim\pi(\cdot\mid x)}$ of the left-hand side gives $\KL(\pi(\cdot\mid x)\|\pi_T(\cdot\mid x))$ by definition (Eq.~\eqref{eq:app:seq_kl_def}). For the right-hand side, group the summands by time step and apply the tower property: for each $t$, conditioning on the prefix $s_t$ generated by $\pi$, the realized token obeys $y_t\sim\pi(\cdot\mid s_t)$, so
\begin{equation*}
\mathbb{E}_{y\sim\pi}\!\left[\log\frac{\pi(y_t\mid s_t)}{\pi_T(y_t\mid s_t)}\right]
=\mathbb{E}_{y\sim\pi}\!\left[\mathbb{E}_{a\sim\pi(\cdot\mid s_t)}\!\log\frac{\pi(a\mid s_t)}{\pi_T(a\mid s_t)}\right]
=\mathbb{E}_{y\sim\pi}\!\left[\KL\!\big(\pi(\cdot\mid s_t)\,\|\,\pi_T(\cdot\mid s_t)\big)\right].
\end{equation*}
Summing over $t=1,\dots,L(y)$ yields Eq.~\eqref{eq:app:kl_chain_rule}. Treating the end-of-sequence marker as a terminal token in the shared vocabulary makes the factorization (and hence this identity) exact even when the length $L(y)$ varies across trajectories.
\end{proof}

\begin{proposition}[Closed-form optimum for teacher-regularized bandit objective (sum-KL)]
\label{thm:bandit_opt}
Fix an input $x$. Consider optimizing over all distributions $\pi(\cdot\mid x)$:
\begin{equation}
\max_{\pi(\cdot\mid x)}
\;\;
\mathbb{E}_{y\sim\pi(\cdot\mid x)}[r(x,y)]
-\beta\,\KL\!\big(\pi(\cdot\mid x)\ \big\|\ \pi_T(\cdot\mid x)\big).
\label{eq:app:bandit_obj}
\end{equation}
Assuming $\pi_T(y\mid x)>0$ for all $y$ in the support, an optimal solution is
\begin{equation}
\pi^\star(y\mid x)
=
\frac{\pi_T(y\mid x)\exp\!\big(r(x,y)/\beta\big)}
{\sum_{y'}\pi_T(y'\mid x)\exp\!\big(r(x,y')/\beta\big)}.
\label{eq:app:boltzmann_opt}
\end{equation}
\end{proposition}

\begin{proof}
Introduce Lagrangian with constraint $\sum_y \pi(y\mid x)=1$:
\[
\mathcal{L}(\pi,\lambda)=
\sum_y \pi(y\mid x)r(x,y)
-\beta\sum_y \pi(y\mid x)\log\frac{\pi(y\mid x)}{\pi_T(y\mid x)}
+\lambda\Big(\sum_y \pi(y\mid x)-1\Big).
\]
Take derivative w.r.t.\ $\pi(y\mid x)$ and set to zero:
\[
r(x,y)-\beta\Big(\log\pi(y\mid x)-\log\pi_T(y\mid x)+1\Big)+\lambda=0.
\]
Rearrange to obtain $\log\pi(y\mid x)=\log\pi_T(y\mid x)+r(x,y)/\beta + c$,
where $c$ absorbs constants.
Exponentiate and normalize over $y$ to obtain Eq.~\eqref{eq:app:boltzmann_opt}.
\end{proof}

\paragraph{Remark (standard result).}
Eq.~\eqref{eq:app:boltzmann_opt} is the standard closed form of a KL-regularized (maximum-entropy) objective and recurs throughout the RLHF / control-as-inference literature (e.g., the DPO reward--policy correspondence \citep{rafailov2023direct}); we state it as a proposition, included for self-containedness, rather than as a novel result.

\paragraph{Scope clarification (sum-KL vs.\ mean-KL).}
Proposition~\ref{thm:bandit_opt} provides intuition for a \emph{sum-KL} regularized objective.
Our method uses the \emph{token-mean} regularizer in Eq.~\eqref{eq:app:token_mean_revkl}.
When $L$ is fixed, sum- and mean-normalized variants differ only by a constant scaling of $\beta$.
When $L(y)$ is trajectory-dependent, they are not identical; the bandit analysis should therefore
be interpreted as a simplified view rather than an exact characterization of Eq.~\eqref{eq:app:main_obj_beta}.

\subsection{Gradients: Reward Term vs.\ Teacher Anchor}
\label{app:gradients}

We derive the gradient estimators for Eq.~\eqref{eq:app:main_obj_beta} and then describe the practical
approximation used in our implementation.
For brevity, omit expectation over $(x,a^\star)$.

\paragraph{Reward term (REINFORCE).}
Using the score-function identity,
\begin{equation}
\nabla_\theta\;
\mathbb{E}_{y\sim\pi_\theta(\cdot\mid x)}[r(x,y)]
=
\mathbb{E}_{y\sim\pi_\theta(\cdot\mid x)}
\Big[r(x,y)\sum_{t=1}^{L(y)}\nabla_\theta\log\pi_\theta(y_t\mid s_t)\Big].
\label{eq:app:reinforce}
\end{equation}
A baseline $b(x)$ can be subtracted from $r(x,y)$ to reduce variance without bias.

\paragraph{Teacher anchor gradient (analytic full-vocabulary form).}
For a fixed state $s$, define
$g(\theta;s)=\KL(\pi_\theta(\cdot\mid s)\|\pi_T(\cdot\mid s))$.
Since $\pi_T$ is fixed, $g(\theta;s)$ is a differentiable function of the student logits.
Expanding the KL over the vocabulary $\mathcal{V}$,
\begin{equation}
g(\theta;s)
=
\sum_{a\in\mathcal{V}}
\pi_\theta(a\mid s)\Big(\log\pi_\theta(a\mid s)-\log\pi_T(a\mid s)\Big).
\label{eq:app:kl_vocab}
\end{equation}
Therefore $\nabla_\theta g(\theta;s)$ can be computed by standard automatic differentiation
through the student softmax probabilities (and does \emph{not} require backpropagating through sampling).
The full token-mean regularizer gradient is
\begin{equation}
\nabla_\theta\mathcal{L}_{\mathrm{revKL}}(x,y)
=
\frac{1}{L(y)}\sum_{t=1}^{L(y)}\nabla_\theta g(\theta;s_t).
\label{eq:app:grad_token_mean}
\end{equation}

\paragraph{Note on the omitted score-function term.}
Because the KL penalty $\mathcal{L}_{\mathrm{revKL}}(x,y)$ appears \emph{inside} the expectation over $y\sim\pi_\theta$ in Eq.~\eqref{eq:app:main_obj_beta}, the exact policy gradient of that term includes an additional score-function (REINFORCE) contribution:
$-\beta\,\mathbb{E}_{y\sim\pi_\theta}\!\big[\mathcal{L}_{\mathrm{revKL}}(x,y)\sum_{t}\nabla_\theta\log\pi_\theta(y_t\mid s_t)\big]$.
Following standard practice in RL fine-tuning of LLMs~\citep{DBLP:conf/iclr/AgarwalVZSGGB24}, we intentionally omit this high-variance term and apply the KL penalty exclusively through its direct gradient as an auxiliary loss (cf.\ Eq.~\eqref{eq:app:grad_token_mean}).
This yields a \emph{surrogate} gradient estimator rather than the exact gradient of Eq.~\eqref{eq:app:main_obj_beta}; the resulting update is equivalent to optimizing a modified objective in which the KL penalty is treated as an auxiliary loss rather than a term inside the policy expectation.
This approximation is widely adopted (e.g., in PPO-style KL penalties and GKD) and empirically effective, as verified by our training dynamics (\Cref{fig:Task reward,fig:Reverse KL loss}).

\paragraph{Implementation note (sampled-token approximation).}
Our implementation does not use the analytic full-vocabulary form in Eq.~\eqref{eq:app:kl_vocab} but a
sampled-token approximation, in the consideration of efficiency.
Concretely, following the low-variance $k_3$ KL estimator of \citet{schulman2020kl} (as implemented in \texttt{veRL}~\citep{DBLP:journals/corr/abs-2409-19256}),
the per-token KL at state $s_t$ is estimated only at the sampled token $y_t$:
\begin{equation}
\hat g(\theta;s_t)=\exp(\Delta_t)-\Delta_t-1,
\qquad
\Delta_t=\log\pi_T(y_t\mid s_t)-\log\pi_\theta(y_t\mid s_t),
\label{eq:app:kl_lowvar}
\end{equation}
which is non-negative and satisfies
$\mathbb{E}_{y_t\sim\pi_\theta(\cdot\mid s_t)}[\hat g(\theta;s_t)]=g(\theta;s_t)$,
i.e., it is an unbiased estimate of the per-token KL value.
Gradients are taken directly through $\log\pi_\theta(y_t\mid s_t)$, consistent with the direct-gradient
treatment above (the score-function correction is again omitted, as discussed in the preceding paragraph).

\paragraph{Optional alternative (score-function estimator; not required here).}
If one instead chooses to \emph{estimate} KL terms by sampling fresh actions $a\sim\pi_\theta(\cdot\mid s)$,
a score-function form exists (cf.\ standard KL-gradient identities); we mention it only for completeness.

\subsection{Stability: Bounding Deviation from the Teacher via KL}
\label{app:stability}

We formalize a simple implication of the KL anchor: small KL limits distributional drift.

\begin{lemma}[Pinsker-style bound on bounded payoffs]
\label{lem:pinsker}
Let $f(y)\in[0,1]$ be any bounded function. Then for any two distributions $P,Q$ over $y$,
\begin{equation}
\big|\mathbb{E}_{y\sim P}[f(y)]-\mathbb{E}_{y\sim Q}[f(y)]\big|
\le
\mathrm{TV}(P,Q)
\le
\sqrt{\tfrac{1}{2}\KL(P\|Q)}.
\label{eq:app:pinsker}
\end{equation}
\end{lemma}

\begin{proof}
The first inequality is standard:
$\sup_{0\le f\le 1}|\mathbb{E}_P f-\mathbb{E}_Q f|=\mathrm{TV}(P,Q)$.
The second follows from Pinsker's inequality.
\end{proof}

\begin{corollary}[Teacher-anchor limits reward deviation (sequence-level)]
Fix $x$ and assume $r(x,y)\in[0,1]$. Then
\begin{equation}
\big|\mathbb{E}_{y\sim\pi_\theta(\cdot\mid x)}[r(x,y)]
-\mathbb{E}_{y\sim\pi_T(\cdot\mid x)}[r(x,y)]\big|
\le
\sqrt{\tfrac{1}{2}\KL(\pi_\theta(\cdot\mid x)\ \|\ \pi_T(\cdot\mid x))}.
\label{eq:app:reward_gap_bound}
\end{equation}
\end{corollary}

\paragraph{Remark (general bounded reward).}
The unit-interval assumption is only a normalization. For any reward bounded in $[r_{\min},r_{\max}]$, applying Lemma~\ref{lem:pinsker} to $f=(r-r_{\min})/(r_{\max}-r_{\min})$ gives the same bound scaled by the range:
$\big|\mathbb{E}_{\pi_\theta}[r]-\mathbb{E}_{\pi_T}[r]\big|\le (r_{\max}-r_{\min})\sqrt{\tfrac12\KL(\pi_\theta\|\pi_T)}$.
This is the relevant form for our deployed reward, whose weighting ($\lambda_{\text{lang}}{=}0.5,\lambda_{\text{3g}}{=}\lambda_{\text{judge}}{=}1$) and length term (Appendix~\ref{app:reward_weights}) put $r$ in a bounded range with $r_{\max}\approx 3.5$ rather than $1$; the qualitative conclusion (small teacher KL bounds reward drift) is unchanged.

\paragraph{Relating token-mean KL to sequence-level KL (when length varies).}
By Lemma~\ref{lem:kl_chain_rule}, the sequence-level KL corresponds to an expected \emph{sum} of per-token KL terms:
\[
\KL(\pi_\theta(\cdot\mid x)\|\pi_T(\cdot\mid x))
=
\mathbb{E}_{y\sim\pi_\theta(\cdot\mid x)}\Big[\sum_{t=1}^{L(y)} \KL_t\Big]
=
\mathbb{E}_{y\sim\pi_\theta(\cdot\mid x)}\Big[L(y)\cdot \mathcal{L}_{\mathrm{revKL}}(x,y)\Big],
\]
where $\KL_t=\KL(\pi_\theta(\cdot\mid s_t)\|\pi_T(\cdot\mid s_t))$.
If generation uses a maximum length $L_{\max}$ (as in standard decoding),
then
\begin{equation}
\KL(\pi_\theta(\cdot\mid x)\|\pi_T(\cdot\mid x))
\le
L_{\max}\cdot
\mathbb{E}_{y\sim\pi_\theta(\cdot\mid x)}\big[\mathcal{L}_{\mathrm{revKL}}(x,y)\big].
\label{eq:app:seqkl_bound_by_mean}
\end{equation}
Thus controlling the expected token-mean KL still provides a meaningful stability control under bounded-length decoding.

\subsection{Why Token-Mean Normalization Helps (Length Bias)}
\label{app:length_norm}

A practical issue in sequence optimization is \emph{length bias}:
a token-sum penalty scales with $L$ and can undesirably favor shorter outputs.
We show a simple relationship between sum- and mean-normalized forms at the trajectory level.

\begin{lemma}[Sum vs.\ mean reverse KL (trajectory level)]
Let
$\mathcal{L}_{\mathrm{sum}}(x,y)=\sum_{t=1}^{L(y)}\KL(\pi_\theta(\cdot\mid s_t)\|\pi_T(\cdot\mid s_t))$
and $\mathcal{L}_{\mathrm{mean}}(x,y)$ be Eq.~\eqref{eq:app:token_mean_revkl}.
Then
\begin{equation}
\mathcal{L}_{\mathrm{sum}}(x,y)=L(y)\cdot \mathcal{L}_{\mathrm{mean}}(x,y).
\label{eq:app:sum_mean_relation}
\end{equation}
\end{lemma}

\begin{proof}
Immediate from definitions.
\end{proof}

\paragraph{Implication.}
With a fixed $\beta$, a sum penalty yields an effective regularization strength $\beta L(y)$,
which couples optimization to output length. The mean normalization decouples the anchor from length
and yields more stable tuning of $\beta$ across datasets/backbones and decoding settings.

\subsection{Optional: Entropy-Regularized RL View (Identity + Estimator)}
\label{app:lower_bound}

Using Lemma~\ref{lem:revkl_identity} inside Eq.~\eqref{eq:app:token_mean_revkl},
the teacher anchor can be expressed as a sum of two \emph{per-state} terms:
\begin{equation}
-\beta\,\mathcal{L}_{\mathrm{revKL}}(x,y)
=
\frac{\beta}{L(y)}\sum_{t=1}^{L(y)}
\Big(
\mathbb{E}_{a\sim\pi_\theta(\cdot\mid s_t)}[\log \pi_T(a\mid s_t)]
+\mathcal{H}(\pi_\theta(\cdot\mid s_t))
\Big).
\label{eq:app:anchor_entropy_form}
\end{equation}
This highlights an entropy-regularized RL interpretation plus an additional teacher log-likelihood shaping term.

\paragraph{Implementation note (realized-token estimator).}
In practice, one may estimate the expectation
$\mathbb{E}_{a\sim\pi_\theta(\cdot\mid s_t)}[\log \pi_T(a\mid s_t)]$
with the realized token $y_t\sim\pi_\theta(\cdot\mid s_t)$, i.e.,
$\log\pi_T(y_t\mid s_t)$.
This is a Monte Carlo estimator of the \emph{per-state} expectation; however, when $L(y)$ is trajectory-dependent,
we avoid claiming that substituting realized tokens yields an \emph{equivalent} rewritten objective.
We use Eq.~\eqref{eq:app:anchor_entropy_form} as the correct identity and treat realized-token quantities
only as estimators for optimization/intuition.


\section{Pseudocode}
\label{sec:app:pseudocode}
We outline the training procedure of \model\ in Algorithm~\ref{alg:trrag}.

\begin{algorithm}[h]
\algrenewcommand\algorithmicrequire{\textbf{Input:}}
\algrenewcommand\algorithmicensure{\textbf{Output:}}
\small
\caption{\textsc{Training procedure of \model}}
\label{alg:trrag}
\begin{algorithmic}[1]
    \Require Training distribution $\mathcal{B}$ over $(x,a^\star)$ with $x=(q,D)$; student policy $\pi_\theta$; \emph{frozen} teacher $\pi_T$;
    weights $\lambda_{\text{lang}},\lambda_{\text{3g}},\lambda_{\text{judge}}$; teacher weight $\beta$; rollouts per input $K$; minibatch size $b$; advantage-normalization constant $\epsilon{=}10^{-6}$.
    \Ensure Trained student policy parameters $\theta$.
    \State Initialize $\theta$;
    \While{not converged}
        \State sample a minibatch $\{(x_i,a_i^\star)\}_{i=1}^{b} \sim \mathcal{B}$;
        \For{each instance $(x_i,a_i^\star)$ in the minibatch}
            \For{$k=1$ \textbf{to} $K$}
                \State sample on-policy rollout $y_i^{(k)} \sim \pi_\theta(\cdot \mid x_i)$; \RightComment{Eq.~\eqref{eq:rollout_group}}
                \State $r_{\text{lang}} \gets \mathbb{I}[\mathrm{LID}(y_i^{(k)})=\mathrm{LID}(a_i^\star)]$; \RightComment{Eq.~\eqref{eq:reward_lang}}
                \State $r_{\text{3g}} \gets \mathrm{Recall}@3\mathrm{gram}(y_i^{(k)},a_i^\star)$; \RightComment{Eq.~\eqref{eq:reward_3gram}}
                \State $r_{\text{judge}} \gets \mathrm{Judge}(q_i,D_i,a_i^\star,y_i^{(k)})$; \RightComment{Eq.~\eqref{eq:reward_judge}}
                \State $r_i^{(k)} \gets \lambda_{\text{lang}}r_{\text{lang}}+\lambda_{\text{3g}}r_{\text{3g}}+\lambda_{\text{judge}}r_{\text{judge}}$; \RightComment{Eq.~\eqref{eq:reward_total}}
                \State query teacher on student prefixes $\pi_T(\cdot\mid s_{i,t}^{(k)})$ for $t=1,\ldots,L_i^{(k)}$; \RightComment{no teacher rollouts}
                \State $\mathcal{L}_{\text{revKL},i}^{(k)} \gets \frac{1}{L_i^{(k)}}\sum_{t=1}^{L_i^{(k)}}\mathrm{KL}\!\big(\pi_\theta(\cdot\mid s_{i,t}^{(k)})\ \|\ \pi_T(\cdot\mid s_{i,t}^{(k)})\big)$; \RightComment{Eq.~\eqref{eq:token_revkl}}
            \EndFor
            \State $\bar r_i \gets \frac{1}{K}\sum_{k=1}^{K} r_i^{(k)}$, $\sigma_i \gets \sqrt{\frac{1}{K}\sum_{k=1}^{K}(r_i^{(k)}-\bar r_i)^2}$;
            \State $A_i^{(k)} \gets \frac{r_i^{(k)}-\bar r_i}{\sigma_i+\epsilon}$ for $k=1,\ldots,K$;
        \EndFor
        \State update $\theta$ by optimizing the GRPO objective with advantages $\{A_i^{(k)}\}$ and adding the teacher-anchor loss
        $\beta \cdot \frac{1}{bK}\sum_{i,k}\mathcal{L}_{\text{revKL},i}^{(k)}$ (direct gradients);
    \EndWhile
    \State \Return $\theta$.
\end{algorithmic}
\end{algorithm}

\section{Experiments}
\label{sec:app:exp}

\subsection{\mbox{ENKB-RAG-5}: Multilingual Variants of rag-mini-bioasq and HotpotQA}
\label{sec:app:exp:datasets}

\paragraph{Overview and motivation.}
We release \textbf{BioASQ-ENKB5} and \textbf{Hotpot-ENKB5}, multilingual variants of rag-mini-bioasq and HotpotQA
constructed for the \emph{English-evidence} cross-lingual RAG regime studied in this work: the user query and required answer are in a target language, while retrieved evidence passages remain English.
For each example, we keep evidence passages in English and translate the question $q$ and reference answer $a^\star$
into five target languages: Indonesian (\textit{id}), Korean (\textit{ko}), Thai (\textit{th}), Vietnamese (\textit{vi}), and Japanese (\textit{ja}).
We treat \textit{id/ko/th/vi} as in-domain (ID) training languages and evaluate generalization on held-out Japanese (\textit{ja}) as well as the original English (\textit{en}).%
\footnote{We will release the data, preprocessing scripts, and documentation upon publication.}

We adopt translation-based construction as a pragmatic compromise between \emph{task difficulty} and \emph{setting fidelity}.
While several existing cross-lingual/multilingual QA benchmarks are valuable, many are primarily designed for extractive or short factual QA and can be less diagnostic for the \emph{generation-side} failures we target (language drift and brittle evidence use under injected English passages) with modern instruction-tuned LLMs.
In contrast, rag-mini-bioasq and HotpotQA provide more challenging open-domain questions (domain-specific terminology and multi-hop reasoning), where we observe substantial headroom and clear failure modes in English-evidence deployments.

\paragraph{Source data and splits.}
\textbf{BioASQ-ENKB5} is derived from rag-mini-bioasq \citep{tsatsaronis2015overview}, containing 4.7k single-hop questions
and 40k biomedical passages. We use the same 80/20 random split over questions (train/test) as in the main paper.
\textbf{Hotpot-ENKB5} is derived from HotpotQA \citep{yang2018hotpotqa}; we sample 4.4k bridge-type questions for training
and 1.1k for testing.
We preserve original example identifiers to enable traceability to the source benchmarks.

\paragraph{Translation pipeline.}
We translate $(q,a^\star)$ using \texttt{Qwen3-235B-A22B-Instruct-2507} \citep{yang2025qwen3}.
To reduce translation artifacts that may affect evidence grounding (e.g., entity/term drift), the prompt enforces:
(i) faithful meaning preservation,
(ii) no added or dropped facts,
(iii) \emph{verbatim preservation} of numbers, dates, abbreviations, and salient entities/terms whenever possible,
and (iv) output in the target language only (no explanations).
We decode deterministically with \textbf{temperature=0.0}, \textbf{top\_p=1.0}, and \textbf{max\_tokens=4096}.
This translation is a \emph{one-time, offline} preprocessing step, run once per dataset; it is not part of the \model\ training loop or of inference, and therefore adds no recurring cost to the method itself.
Post-processing normalizes whitespace/punctuation, strips surrounding quotes when present, and removes obvious boilerplate
(non-target-language meta text, if any).
The translation prompts are shown in \Cref{fig:prompt-id,fig:prompt-ja,fig:prompt-ko,fig:prompt-th}; the Vietnamese prompt follows the same template with the target language set to Vietnamese.

\paragraph{Potential artifacts and why they matter.}
Constructing multilingual instances via translation can introduce two classes of distribution shift relative to ``native'' multilingual queries:
\textbf{(a) translationese/style shift}, where phrasing becomes more literal or unnatural; and
\textbf{(b) entity/terminology perturbations}, where rare entities, abbreviations, or biomedical terms may be inconsistently transliterated or normalized.
Both can, in principle, interact with our targeted phenomena:
for example, unnatural phrasing may change the model's tendency to copy English spans, and term drift can weaken lexical anchoring between the (translated) query/answer and the English evidence, affecting both grounding and overlap-based signals.
We therefore treat translation artifacts as a \emph{known limitation} of this construction and explicitly include sanity checks and reporting to mitigate and quantify their impact.

\paragraph{Quality mitigation and known risks.}
To mitigate translation artifacts, we rely primarily on \emph{prompt-level} quality control during translation: the translation prompt explicitly enforces faithful meaning preservation, verbatim retention of numbers/dates/abbreviations/entities, and target-language-only output (see \Cref{fig:prompt-id,fig:prompt-ja,fig:prompt-ko,fig:prompt-th}), combined with deterministic decoding (\textbf{temperature=0.0}).
Post-processing normalizes whitespace/punctuation, strips surrounding quotes, and removes obvious boilerplate.
We additionally performed a manual spot-check on a randomly sampled subset of 10 examples per target language to assess
(i) fluency/naturalness,
(ii) whether entities/terms remain semantically consistent,
and (iii) whether the qualitative failure modes studied in the paper (language drift and evidence misuse under English passages) persist beyond a small number of translation outliers.
However, we did \emph{not} apply systematic automatic post-translation filtering (e.g., LID-based filtering, number-preservation checks, or entity-consistency heuristics) on the final training data beyond the prompt-level and post-processing controls described above.
This means that residual translation errors (such as occasional language mixing, entity transliteration inconsistencies, or dropped numeric details) may remain in the dataset.
We acknowledge this as a limitation and note that applying such automatic filters in future work could further improve data quality.
For reproducibility, we release (i) the translation prompts and decoding parameters and (ii) the post-processing scripts.

\paragraph{Discussion: relation to native multilingual-query benchmarks.}
An important complementary evaluation is to test on benchmarks where the \emph{original} questions are authored in non-English languages while evidence remains English.
Existing datasets in this direction (e.g., XOR-TyDi QA \citep{asai2021xor}  and related cross-lingual QA resources\citep{DBLP:journals/tacl/LongpreLD21} ) are valuable but often differ in task structure and difficulty (e.g., emphasis on retrieval or short-answer extraction), which may make them less diagnostic for the open-domain \emph{generation-side} failures we target with strong instruction-tuned models.
We therefore focus our main experiments on challenging open-domain QA where we observe substantial headroom, while treating native multilingual-query benchmarks as a complementary axis for future and expanded evaluation.
We also make our construction pipeline modular so that the same training recipe can be applied to other English-evidence cross-lingual datasets as they become available or as difficulty-matched suites emerge.

\paragraph{Limitations.}
Because no systematic automatic post-translation filtering was applied, \mbox{ENKB-RAG-5} may exhibit translationese, residual terminology inconsistencies, and occasional translation errors.
Accordingly, we interpret results as evidence about \emph{the English-evidence cross-lingual generation problem under a controlled construction}, and we encourage follow-up evaluations on native multilingual-query settings when suitable difficulty-matched benchmarks are available.

\begin{figure}[t]
   \centering
   \fbox{\includegraphics[width=0.8\linewidth]{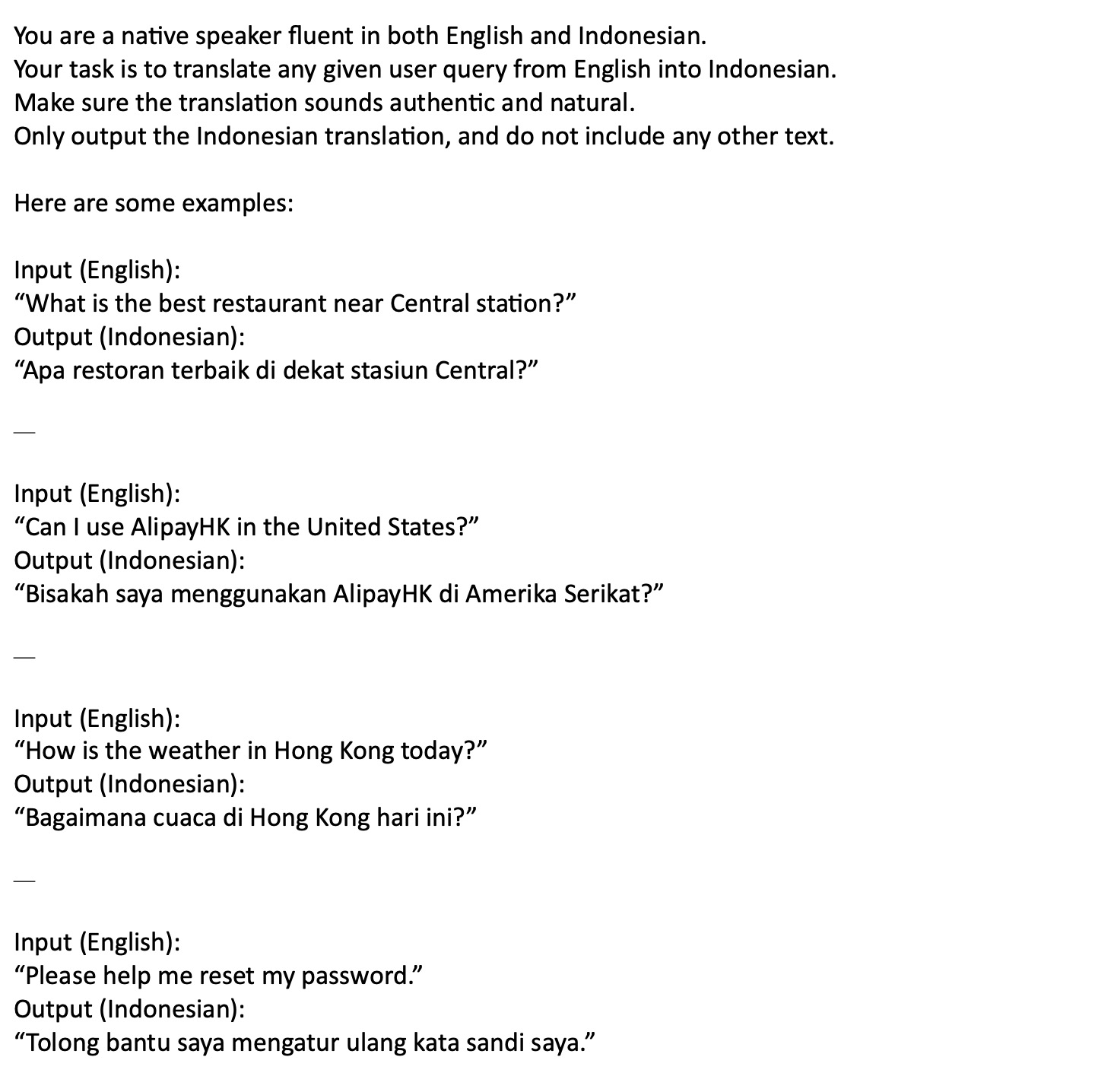}}
   \caption{\textbf{Indonesian translation prompt.} Prompt template used to translate English queries into Indonesian (\textit{id}) when constructing \textbf{BioASQ-ENKB5} and \textbf{Hotpot-ENKB5}.}
   \vspace{-2mm}
   \label{fig:prompt-id}
\end{figure}

\begin{figure}[t]
   \centering
   \fbox{\includegraphics[width=0.8\linewidth]{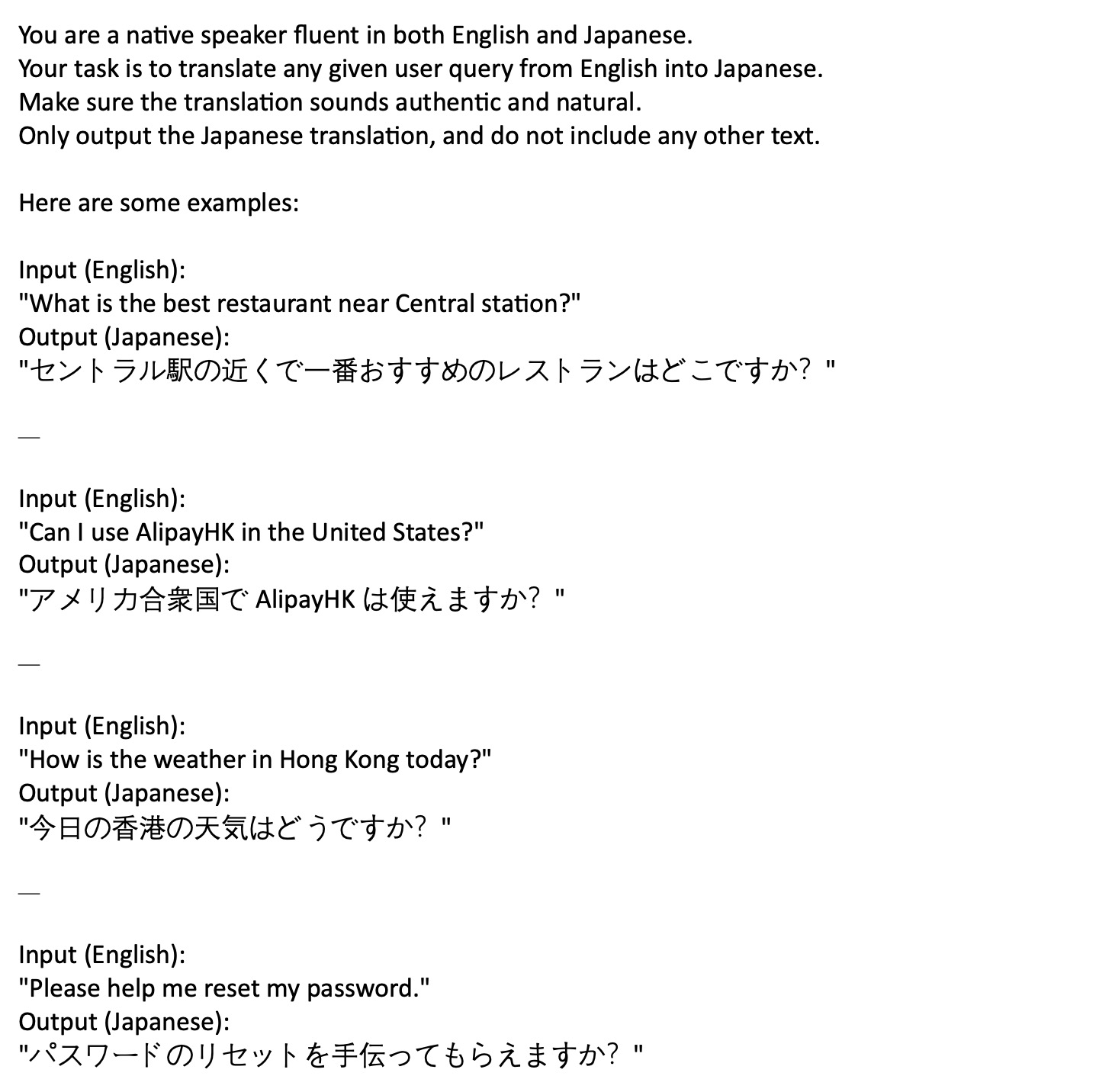}}
   \caption{\textbf{Japanese translation prompt.} Prompt template used to translate English queries into Japanese (\textit{ja}) when constructing \textbf{BioASQ-ENKB5} and \textbf{Hotpot-ENKB5}.}
   \vspace{-2mm}
   \label{fig:prompt-ja}
\end{figure}

\begin{figure}[t]
   \centering
   \fbox{\includegraphics[width=0.8\linewidth]{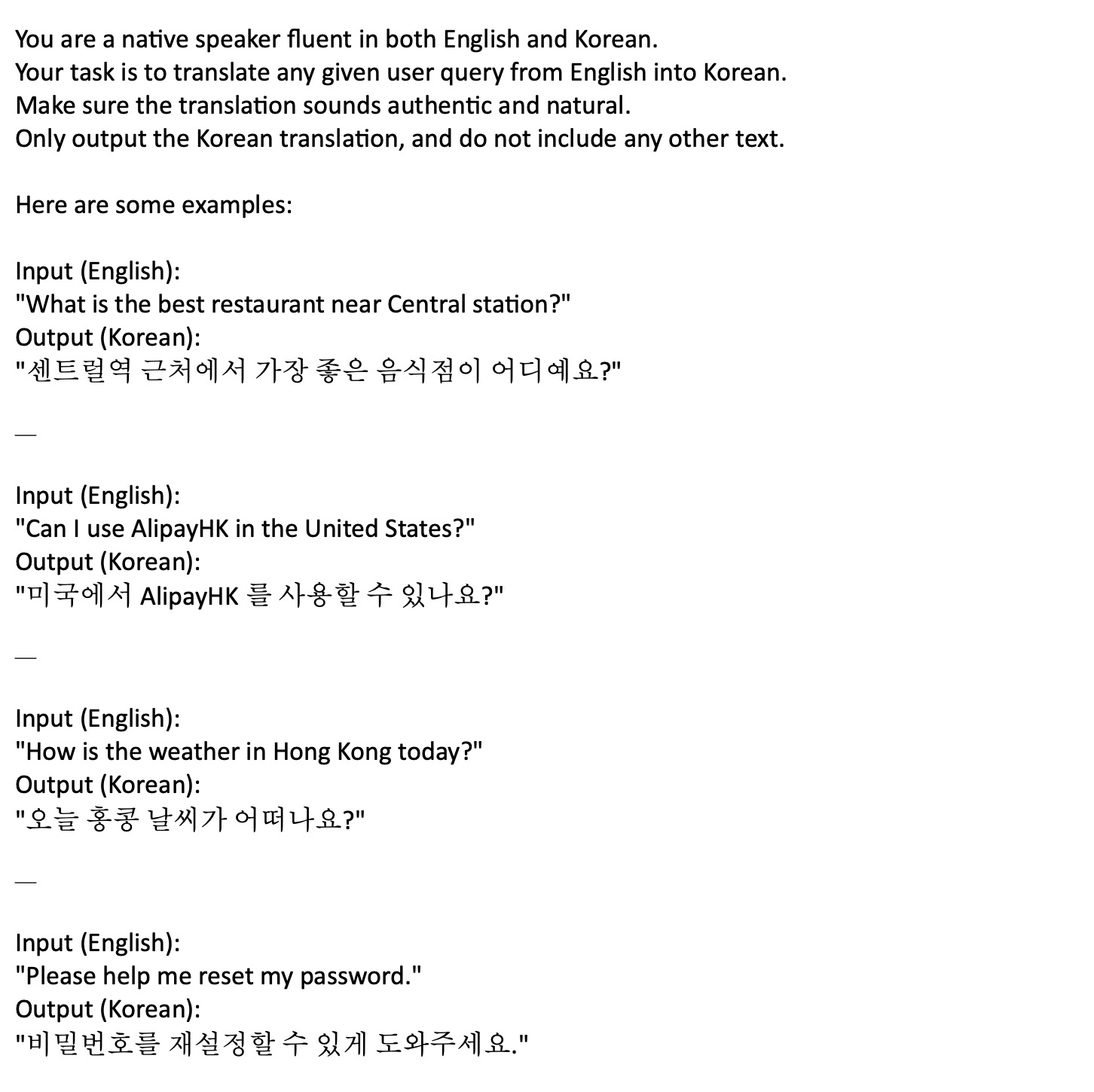}}
   \caption{\textbf{Korean translation prompt.} Prompt template used to translate English queries into Korean (\textit{ko}) when constructing \textbf{BioASQ-ENKB5} and \textbf{Hotpot-ENKB5}.}
   \vspace{-2mm}
   \label{fig:prompt-ko}
\end{figure}


\begin{figure}[t]
   \centering
   \fbox{\includegraphics[width=0.8\linewidth]{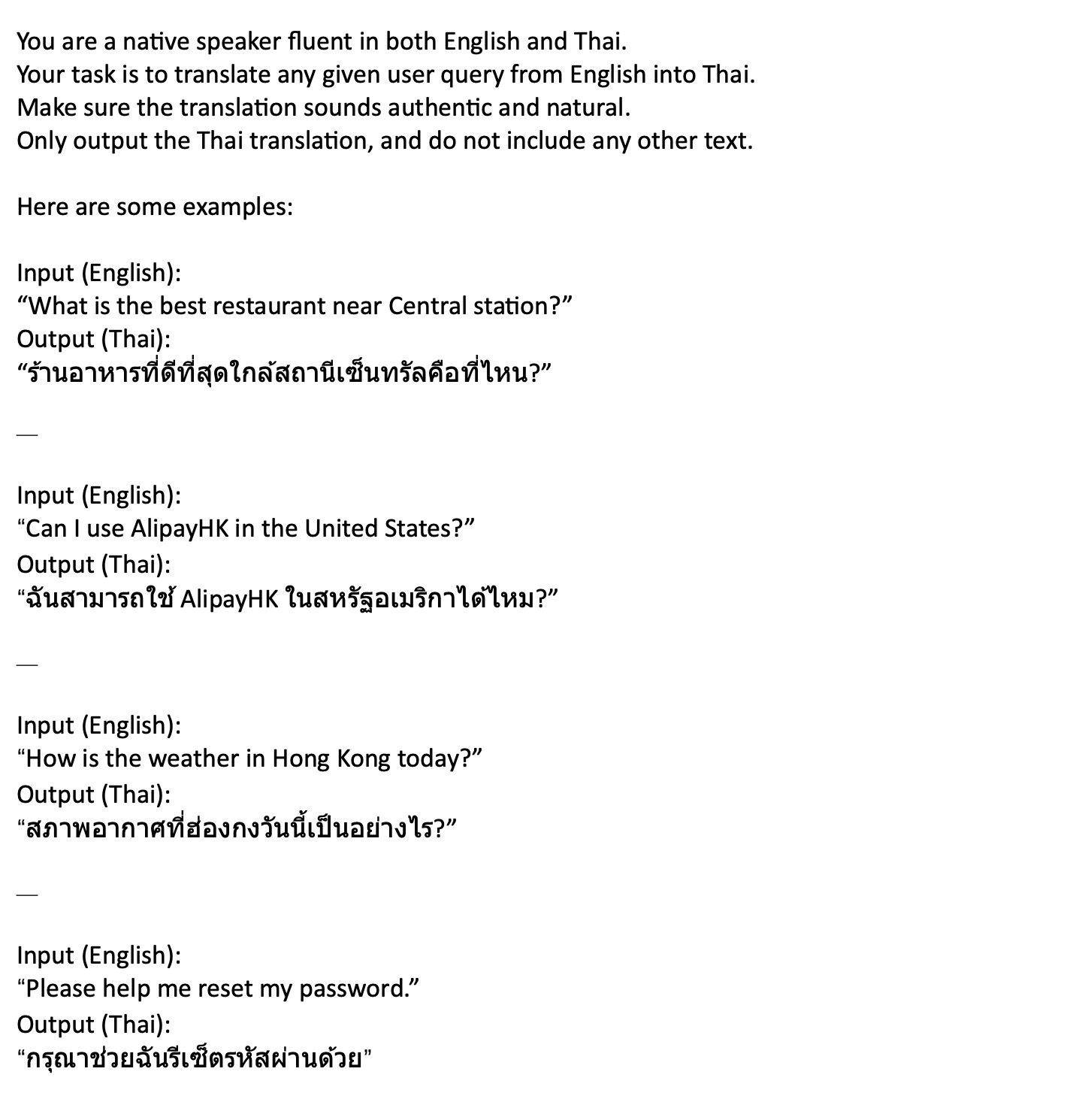}}
   \caption{\textbf{Thai translation prompt.} Prompt template used to translate English queries into Thai (\textit{th}) when constructing \textbf{BioASQ-ENKB5} and \textbf{Hotpot-ENKB5}.}
   \vspace{-2mm}
   \label{fig:prompt-th}
\end{figure}

\begin{figure}[t]
   \centering
   \fbox{\includegraphics[width=0.8\linewidth]{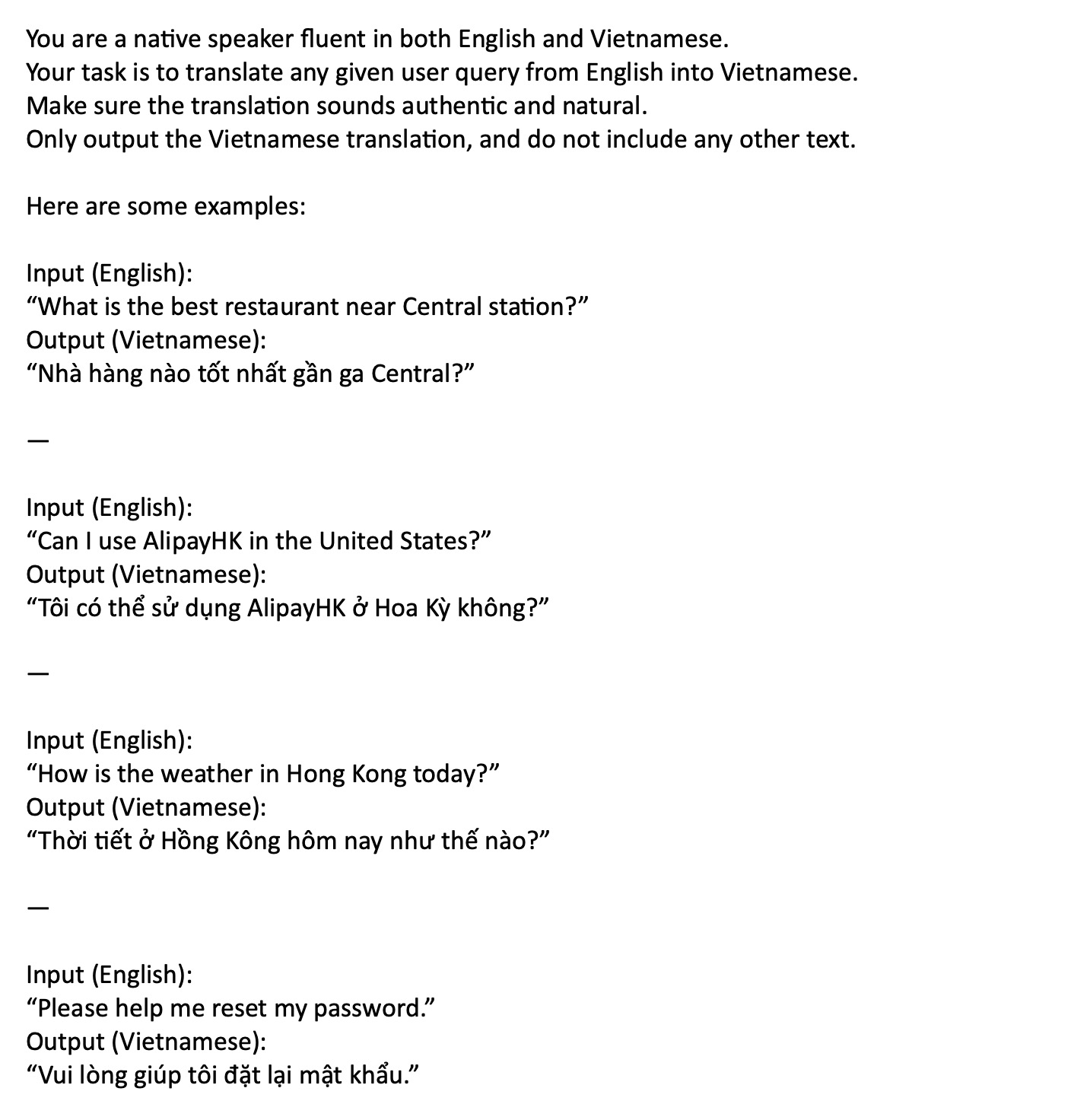}}
   \caption{\textbf{Vietnamese translation prompt.} Prompt template used to translate English queries into Vietnamese (\textit{vi}) when constructing \textbf{BioASQ-ENKB5} and \textbf{Hotpot-ENKB5}.}
   \vspace{-2mm}
   \label{fig:prompt-vi}
\end{figure}


\subsection{Metrics}
\label{sec:app:exp:metrics}

To align with our reward decomposition (Eq.~\eqref{eq:reward_total}), we evaluate three complementary criteria for each generated answer $y$ given $(q,D)$, all mapped to $[0,1]$.
\textbf{Language consistency} measures target-language adherence under English evidence: we apply the fastText \texttt{lid.176} language identification model \citep{joulin2017bag, grave2018learning} and use a binary indicator $r_{\text{lang}}(y)=\mathbb{I}[\mathrm{LID}(y)=\mathrm{LID}(a^\star)]\in\{0,1\}$, matching the training reward (Eq.~\eqref{eq:reward_lang}).
\textbf{Character 3-gram recall} follows Eq.~\eqref{eq:reward_3gram} and quantifies lightweight lexical faithfulness to the reference answer $a^\star$ via character n-gram overlap; this choice is motivated by character n-gram MT metrics such as chrF \citep{popovic-2015-chrf}, which are less sensitive to tokenization and morphology.
\textbf{LLM-judge} evaluates answer quality with access to the question $q$, retrieved context $D$, reference answer $a^\star$, and student answer $y$; it outputs a graded score $r_{\text{judge}}(q,D,a^\star,y)\in[0,1]$ according to a rubric emphasizing (i) faithfulness to $D$ (hallucination checking), (ii) correctness/alignment with $a^\star$ (including entity-level accuracy for BioASQ and reasoning completeness for HotpotQA), and (iii) conciseness.
Our primary judge is \texttt{Qwen3-Next-80B-A3B-Instruct} \citep{yang2025qwen3, qiu2025gated} with deterministic decoding (temperature $=0$); we parse the scalar score from the XML \texttt{<score>} field, following prior work on using strong LLMs as evaluators \citep{DBLP:conf/nips/ZhengC00WZL0LXZ23}. To rule out within-family reward hacking (the training reward also uses a Qwen-family model), we additionally report a \emph{cross-family} judge, \texttt{Llama-3.3-70B-Instruct} \citep{dubey2024llama}, evaluated under the same rubric and prompt (see \Cref{sec:app:exp:metrics:cross_family_judge} for detailed comparison). The LLM-judge prompt is shown in Figure~\ref{fig:judge-prompt}.
Finally, since all metrics lie in $[0,1]$, we define a single summary metric.
For the \mbox{ENKB-RAG-5} benchmarks (which use a single Qwen judge), $\text{Composite}=\frac{1}{3}(r_{\text{lang}}+r_{\text{3g}}+r_{\text{judge}})$.
For the MKQA benchmark, where we additionally report a cross-family Llama judge to rule out within-family reward hacking (\Cref{sec:app:exp:metrics:cross_family_judge}), we include both judges in the summary: $\text{Composite}_{\text{MKQA}}=\frac{1}{4}(r_{\text{lang}}+r_{\text{3g}}+r_{\text{judge}}^{\text{Qwen}}+r_{\text{judge}}^{\text{Llama}})$.
We report the composite alongside the individual metrics to avoid masking trade-offs.

\begin{figure}[!t]
   \centering
   \fbox{\includegraphics[width=0.95\linewidth]{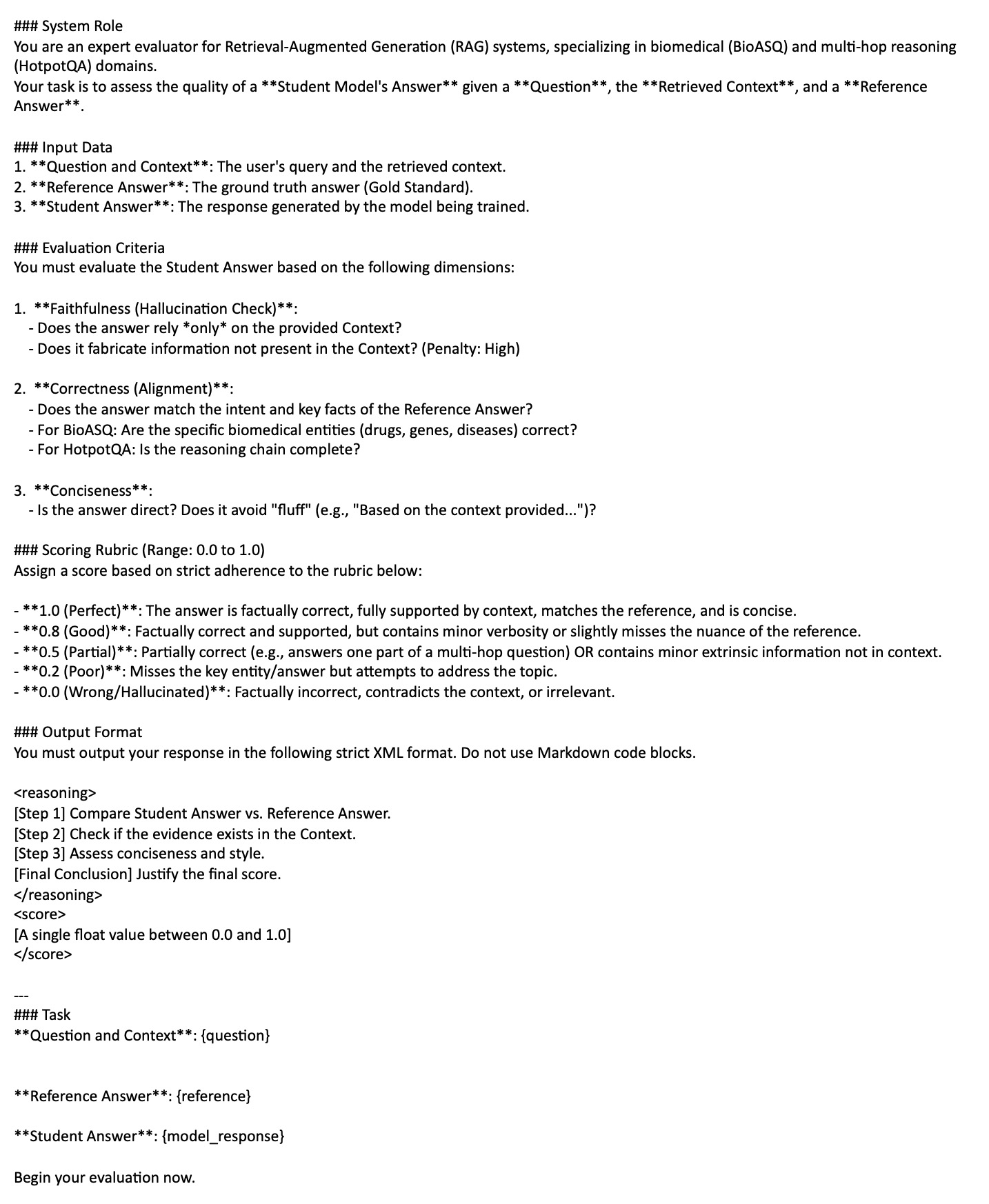}}
   \caption{\textbf{LLM-judge prompt.} We use a rubric-based evaluator prompt to score student answers for RAG, conditioned on the question $q$, retrieved English context $D$, reference answer $a^\star$ (training only), and student output $y$. The prompt instructs the judge to assess (i) \emph{faithfulness} (hallucination checking against $D$), (ii) \emph{correctness} (alignment with $a^\star$, with entity focus for BioASQ and multi-hop completeness for HotpotQA), and (iii) \emph{conciseness}. The judge returns a scalar score in $[0,1]$ following a strict rubric (e.g., 1.0/0.8/0.5/0.2/0.0) and outputs it in a constrained XML format for reliable parsing.}
   \vspace{-2mm}
   \label{fig:judge-prompt}
\end{figure}

\subsubsection{On the inadequacy of exact match for free-form cross-lingual RAG.}
\label{sec:app:exp:metrics:em}

A natural question is why we do not report Exact Match (EM) against the gold answer. The reason is that our setup is open-domain, free-form generation: instruction-tuned student models almost always wrap the answer entity in a full sentence, so a single string match against a short gold answer is uninformative. \Cref{tab:em-zeroshot} reports zero-shot EM (\%) across all three datasets (\texttt{Llama-3.2-3B-Instruct} backbone): every method (Base, Naive-RL, \model) collapses to (near-)zero EM regardless of generation quality.

\begin{table}[h]
\centering
\setlength{\tabcolsep}{4pt}
\caption{Zero-shot Exact Match (\%) across MKQA, HotpotQA, and BioASQ (\texttt{Llama-3.2-3B-Instruct} backbone). EM is uninformative for free-form RAG because all methods generate full sentences rather than extractive spans.}
\label{tab:em-zeroshot}
\begin{tabular}{l|ccc|ccc|ccc}
\toprule
\multirow{2}{*}{Method} & \multicolumn{3}{c|}{MKQA} & \multicolumn{3}{c|}{HotpotQA} & \multicolumn{3}{c}{BioASQ} \\
 & ID & OOD & ALL & ID & OOD & ALL & ID & OOD & ALL \\
\midrule
Base     & 0.33 & 0 & 0.17 & 0 & 0 & 0 & 0 & 0 & 0 \\
Naive-RL & 0    & 0 & 0    & 0 & 0 & 0 & 0 & 0 & 0 \\
\model   & 0    & 0 & 0    & 0 & 0 & 0 & 0 & 0 & 0 \\
\bottomrule
\end{tabular}
\end{table}

This pattern matches prior observations that EM under-estimates open-ended generative QA \citep{Kamalloo2023EvaluatingOQ} and motivates recent multilingual RAG evaluations to adopt character-level overlap and language-consistency metrics instead \citep{DBLP:journals/corr/abs-2407-01463}. We therefore report char 3-gram recall (a chrF-style character n-gram signal), the soft language-ID probability, and an LLM-judge, the same trio used as our reward decomposition.

\subsubsection{LLM-judge reliability and potential overfitting.}
\label{sec:app:exp:metrics:judge_reliability}

We use an LLM-judge (\texttt{Qwen3-Next-80B-A3B-Instruct}) in two roles: (i) as a training reward component $r_{\text{judge}}$ to encourage evidence-grounded correctness, and (ii) as one of the evaluation metrics reported in the tables.
This dual use raises a natural concern that models could overfit to the particular judge and obtain inflated judge scores without corresponding improvements in actual factual correctness.

\paragraph{Mitigations in our training setup.}
First, $r_{\text{judge}}$ is only one component in our reward decomposition; optimization is jointly constrained by the language-adherence and char 3-gram rewards, and, for \model, by the reverse-KL anchor to the \emph{frozen teacher} on student-visited prefixes (Section~\ref{sec:method:opd}). To avoid confusion: this teacher anchor is the \emph{only} KL term we use (we do not add any separate KL-to-initial-policy penalty), and Naive-RL (\model\ with $\beta{=}0$) therefore has no KL term at all, so the ``Naive-RL collapse'' we report is not an artifact of a hidden reference-policy regularizer.
Second, the judge is queried on student-generated outputs but does not provide gradients; the reward is combined with other signals and does not directly train the model to mimic the judge's textual style.
Third, we report multiple complementary metrics beyond the judge score (e.g., correctness/overlap and language adherence), and our main conclusions do not rely on the judge metric alone.

\paragraph{Cross-judge consistency.}
We additionally evaluate the same outputs with an alternative judge model, \texttt{Qwen3-235B-A22B-Instruct-2507} \citep{yang2025qwen3}. \Cref{tab:judge_cross} reports the ALL-AVG judge score on Hotpot-ENKB5 (100 samples/language; \texttt{Llama-3.2-3B-Instruct} backbone).

\begin{table}[h]
\centering
\caption{Two Qwen judges on Hotpot-ENKB5 (\texttt{Llama-3.2-3B-Instruct} backbone, ALL-AVG \%). Identical method ranking; the larger judge assigns a wider gap to \model.}
\label{tab:judge_cross}
\begin{tabular}{l|cc}
\toprule
Method & Qwen3-Next-80B & Qwen-235B \\
\midrule
Base    & 56.51 & 56.41 \\
Naive-RL & \underline{75.88} & \underline{76.50} \\
\rowcolor{gray!15}
\model  & \textbf{76.58} & \textbf{77.77} \\
\bottomrule
\end{tabular}
\end{table}

The two Qwen judges produce identical rankings, and the larger 235B judge assigns a \emph{wider} margin to \model\ (+1.27 vs.\ +0.70), suggesting that our reported improvements are not artifacts of a particular judge.

\subsubsection{Conciseness as a confounder for the 3-gram--judge tension.}
\label{sec:app:exp:metrics:conciseness}

A subtle interaction arises between our character-3-gram reward and the LLM-judge: the 3-gram reward encourages more complete (and therefore longer) responses, while the judge's rubric explicitly rewards \emph{conciseness}. To quantify this confound, we re-evaluate \model\ outputs on MKQA (\texttt{Llama-3.2-3B-Instruct} backbone) with two judge variants: (i) the \emph{original} prompt that includes correctness, faithfulness, and conciseness; and (ii) an \emph{ablated} prompt that drops the conciseness criterion while keeping the rest unchanged.

\begin{table}[h]
\centering
\caption{Conciseness ablation on the LLM judge (MKQA, \model\ outputs, \texttt{Llama-3.2-3B-Instruct} backbone). Removing the conciseness criterion increases the judge score, indicating that the original judge actively penalizes the more complete responses promoted by the 3-gram reward rather than reflecting lower correctness/faithfulness.}
\label{tab:judge_conciseness}
\begin{tabular}{l|cc}
\toprule
Judge variant & Qwen judge & Llama judge \\
\midrule
Original (correctness + faithfulness + conciseness) & 71.50 & 70.45 \\
\rowcolor{gray!15}
Ablated (conciseness removed) & \textbf{73.35} (\textbf{+1.85}) & \textbf{71.80} (\textbf{+1.35}) \\
\bottomrule
\end{tabular}
\end{table}

Removing the conciseness criterion lifts the judge score by $+1.85$ (Qwen) and $+1.35$ (Llama), confirming that the small 3-gram--judge gap is largely a length-conciseness artifact rather than a correctness/faithfulness regression.

\subsubsection{Cross-family judge: full comparison.}
\label{sec:app:exp:metrics:cross_family_judge}

To rule out within-family reward hacking against our Qwen training judge, \Cref{tab:judge_cross_family} pairs every condition with a cross-family \texttt{Llama-3.3-70B-Instruct} judge under the same rubric, broken down by ID/OOD/ALL on MKQA.

\begin{table}[h]
\centering
\setlength{\tabcolsep}{6pt}
\caption{Qwen3-Next-80B (within-family) vs.\ Llama-3.3-70B-Instruct (cross-family) judges on the MKQA English-evidence setup, by ID/OOD/ALL (\%). Method rankings are identical across both judges; the cross-family judge assigns a \emph{larger} ALL-AVG gap to \model.}
\label{tab:judge_cross_family}
\begin{tabular}{l|cc|cc|cc}
\toprule
\multirow{2}{*}{Method} & \multicolumn{2}{c|}{ID-AVG} & \multicolumn{2}{c|}{OOD-AVG} & \multicolumn{2}{c}{ALL-AVG} \\
& Qwen & Llama & Qwen & Llama & Qwen & Llama \\
\midrule
Base     & 65.02 & 62.37 & 66.35 & 65.12 & 65.69 & 63.75 \\
Naive-RL & \textbf{70.77} & \underline{69.03} & \underline{71.77} & \underline{72.13} & \underline{71.27} & \underline{70.58} \\
\rowcolor{gray!15}
\model   & \underline{70.50} & \textbf{70.62} & \textbf{73.45} & \textbf{72.90} & \textbf{71.98} & \textbf{71.76} \\
\bottomrule
\end{tabular}
\end{table}

\Cref{tab:judge_cross_family} shows the within- and cross-family judges produce identical rankings (\model\ $>$ Naive-RL $>$ Base) on every aggregated split. On ALL-AVG, \model\ leads Naive-RL by $+0.71$ under Qwen and $+1.18$ under Llama; the cross-family judge actually \emph{widens} the gap, ruling out judge-family-specific reward hacking. The only minor flip is on ID, where Naive-RL edges \model\ on Qwen (70.77 vs.\ 70.50) but \model\ leads under Llama (70.62 vs.\ 69.03); both gaps are small ($\le 1.6$pp) and do not affect the overall conclusion.

\subsection{Baselines}
\label{sec:app:exp:baselines}
The description for baselines is shown below:
\begin{itemize}
    \item \textbf{Translate}. This baseline reduces the cross-lingual generation problem to a \emph{pivot-English} pipeline. Specifically, given a user query $q$ in language $\ell$, we first translate it into English ($\hat{q}_{en}$), concatenate $\hat{q}_{en}$ with the retrieved English context $D$, and prompt the LLM to produce an English answer $\hat{y}_{en}$. We then translate $\hat{y}_{en}$ back into the target language to obtain the final response $\hat{y}_{\ell}$. Both the query$\rightarrow$en and answer$\rightarrow\ell$ steps use the same \texttt{Qwen3-235B-A22B-Instruct-2507} translator as in dataset construction (Appendix~\ref{sec:app:exp:datasets}) with deterministic decoding (temperature $0$, top-$p$ $1.0$), so that any degradation reflects the pivot-translation \emph{design} rather than a weak translator. Translate is a commonly used heuristic \citep{lewis2020mlqa,hu2020xtreme, asai2021xor} in cross-lingual QA/RAG systems as it avoids direct multilingual decoding under English evidence, but it can introduce compounding translation errors (query intent shifts and answer mistranslations) and may weaken evidence-groundedness due to information loss across two translation steps.
    \item \textbf{Prompt-Control}. Following \citet{DBLP:journals/corr/abs-2504-18428}, we append language-specific control directives to the input prompts, ensuring that the model produces outputs in the same language as the original query. For specific language control directives used for each language, it is shown in Figure~\ref{fig:prompt control}.
    \item \textbf{DIT}. DIT \citep{luo2025mmath} proposes a lightweight prompt-based control for multilingual reasoning. The original method places discourse markers right after the \texttt{<think>} token in reasoning models that natively support extended thinking (e.g., QwQ, DeepSeek-R1). Since \texttt{Qwen3-4B-Instruct-2507} natively supports a thinking mode with \texttt{<think>} tokens, we apply DIT as designed for this backbone. For \texttt{Llama-3.2-3B-Instruct}, which does not natively emit \texttt{<think>} tokens, we adapt DIT by prepending the same discourse markers at the beginning of the model's response, thereby providing an equivalent language-steering cue without introducing out-of-distribution formatting. The used multilingual discourse marks are shown in Figure~\ref{fig:DIT}.
    \item \textbf{Knowledge Distillation}. Following \citet{DBLP:journals/corr/HintonVD15}, we train a compact student model to imitate a stronger teacher by learning from
\emph{soft targets} rather than only one-hot labels.
Given an input $x$, the teacher produces logits $z_T(x)$ and the student produces logits $z_S(x)$; we compute
softened distributions $p_T^{(T)}=\mathrm{softmax}(z_T/T)$ and $p_S^{(T)}=\mathrm{softmax}(z_S/T)$ with temperature $T>1$.
The KD objective minimizes the divergence between the two softened distributions, and when gold labels are available
we linearly combine this distillation loss with the standard cross-entropy.
At inference time, the student uses the normal temperature $T=1$.
Concretely, this is an \emph{offline}, token-level scheme: teacher and student logits are compared on the \emph{gold} (teacher-forced) answer trajectory: soft-target matching on the reference answer, \emph{not} sequence-level distillation (SeqKD) on teacher-\emph{generated} text. This is exactly the offline counterpart that OPD and \model\ improve on by instead querying the teacher on \emph{student-generated} prefixes.
    \item \textbf{SDFT}.
SDFT (Self-Distillation Fine-Tuning \citep{yang-etal-2024-self}) attributes the instability of LLM fine-tuning (especially the trade-off between downstream gains and retaining general instruction-following abilities) to a distribution gap between task-specific datasets and the pre-trained/instruction-tuned model’s original data distribution. To bridge this gap, SDFT constructs a distilled dataset by having the model (as teacher) generate supervision that is closer to its own original distribution, and then uses this distilled data to guide fine-tuning of the target model. Empirically, the paper reports that SDFT reduces catastrophic forgetting and can maintain helpfulness/safety alignment, while achieving comparable or better downstream results than vanilla fine-tuning.
    \item \textbf{Naive-RL}.
This baseline applies standard reward optimization without any teacher-based anchoring.
Given an input $x=(q,D)$, the student policy $\pi_\theta$ (initialized from \texttt{Qwen3-4B-Instruct-2507} or \texttt{Llama-3.2-3B-Instruct} \citep{yang2025qwen3,dubey2024llama})
generates on-policy rollouts $y\sim\pi_\theta(\cdot\mid x)$.
We then update $\theta$ with policy-gradient using only our task reward $r(x,y)$ (language consistency + char 3-gram recall + LLM-judge),
while disabling the teacher regularizer (equivalently setting $\beta=0$ in Eq.~\eqref{eq:final_objective_beta}).
RL-only serves as a strong reference for evaluating whether reward shaping alone can address English-evidence multilingual generation issues such as language drift and brittle evidence use.
    \item \textbf{On-Policy Distillation (OPD)} \citep{DBLP:conf/iclr/AgarwalVZSGGB24}. This baseline uses teacher guidance as the sole training signal, without task-level rewards. The student (initialized from \texttt{Qwen3-4B-Instruct-2507} or \texttt{Llama-3.2-3B-Instruct} \citep{yang2025qwen3,dubey2024llama}) generates on-policy rollouts $y\sim\pi_\theta(\cdot\mid x)$. For each prefix state $s_t=(x,y_{<t})$, we query a higher-capacity teacher, \texttt{Qwen3-30B-A3B-Instruct-2507} (30.5B total / 3.3B activated MoE) or \texttt{Llama-3.3-70B-Instruct} \citep{yang2025qwen3,dubey2024llama}, to obtain $\pi_T(\cdot\mid s_t)$, and update the student by minimizing the same length-normalized reverse-KL anchor $\mathcal{L}_{\text{revKL}}(x,y)$ used by \model\ (Eq.~\eqref{eq:token_revkl}), the token-mean reverse KL over the student's own prefixes.
Importantly, the teacher is never used to generate rollouts; it only provides prefix-wise distributions on the trajectories produced by the student.
OPD-only provides a clean reference for assessing how far prefix-conditioned teacher supervision alone can improve multilingual generation under English evidence, independent of explicit reward optimization.
Note that OPD does not use a task reward (hence \textbf{no $\beta$} teacher-anchor hyperparameter); it trains purely by minimizing $\mathcal{L}_{\text{revKL}}$.

\end{itemize}

\begin{figure}[t]
   \centering
   \includegraphics[width=0.8\linewidth]{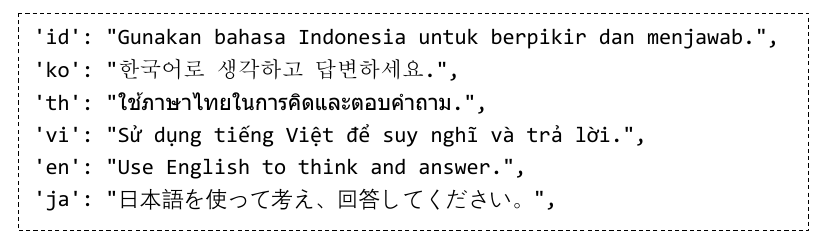}
   \caption{The language control instructions of the Prompt-Control baseline \citep{DBLP:journals/corr/abs-2504-18428, zhang2025think}.}
   \vspace{-2mm}
   \label{fig:prompt control}
\end{figure}

\begin{figure}[t]
   \centering
   \includegraphics[width=0.35\linewidth]{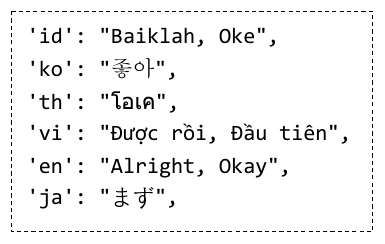}
   \caption{The multilingual discourse marks for each language of the DIT baseline. \citep{luo2025mmath, zhang2025think}.}
   \vspace{-2mm}
   \label{fig:DIT}
\end{figure}

\subsection{Implementation details.}
\label{app:impl_details}

\subsubsection{Training recipe (no task-specific SFT by default).}
Unless otherwise specified, we \emph{do not} apply an additional supervised fine-tuning (SFT) stage on our task data before reinforcement learning (RL).
Instead, we start RL directly from an off-the-shelf instruction-tuned checkpoint (e.g., \texttt{Llama-3.2-3B-Instruct}), i.e., an instruction-following model without any \emph{task-specific} warm-start.
This design follows recent evidence that large-scale RL can be effective without an explicit SFT stage \citep[e.g.,][]{DBLP:journals/corr/abs-2501-12948}, and we find it particularly suitable for \emph{English-evidence cross-lingual RAG}.

\subsubsection{Why task-SFT may underperform in our setting.}
Empirically, adding a task-specific SFT stage before RL can underperform the original instruction-tuned checkpoint in our setting.
A key reason is \emph{exposure bias} (prefix mismatch): SFT optimizes token-level likelihood under \emph{teacher forcing} and thus primarily trains on \emph{ground-truth} prefixes.
At inference time, however, the model conditions on its \emph{own} generated prefixes; early deviations (e.g., language drift or suboptimal evidence usage) shift the state distribution, causing subsequent decoding to operate out-of-distribution and compounding errors.
Consequently, improving likelihood on gold prefixes does not necessarily yield better behavior on the \emph{student-visited} prefixes that dominate actual generation.

This motivates our default recipe: \emph{direct on-policy RL with teacher regularization on student-visited prefixes}.
\model\ queries a frozen teacher only on student-generated prefixes and applies a prefix-wise reverse-KL anchor, providing dense stabilization exactly where the student policy operates during rollout and mitigating exposure-bias-induced compounding errors.
We still include task-SFT as a baseline, but do not rely on it as a warm-start in our main recipe.

\subsubsection{Systems and optimization.}
\label{app:systems}
We implement our method using the open-source \texttt{veRL} \citep{DBLP:journals/corr/abs-2409-19256} framework on a single node with $8\times$ NVIDIA H100 (80GB) GPUs.
Unless otherwise specified, we set the maximum prompt/response lengths to 8{,}192 / 2{,}048 tokens, filter overlong prompts, and enforce strict truncation (raise an error on overflow).
We use GRPO with $K=8$ rollouts per prompt, sampling temperature $0.9$, and a PPO-style clipping ratio of $0.2$ (the \texttt{veRL} default); the advantage-normalization constant is $\epsilon=10^{-6}$.
The student (actor) is \emph{full fine-tuned} with learning rate $1\times 10^{-6}$ and a 10\% warmup ratio, using PPO-style updates with micro-batch size 2 per GPU (mini-batch size 32).
Each run is trained for approximately 250 GRPO update steps (the mean task reward plateaus by then; see \Cref{fig:Task reward}), and we evaluate the \emph{final} checkpoint rather than selecting on a development metric. At \emph{evaluation}, student models decode greedily (temperature $0$) with up to $2{,}048$ new tokens.
We enable gradient checkpointing and train the actor with FSDP (no parameter/optimizer offloading), while the teacher model uses FSDP with parameter offloading enabled.
Rollouts are generated with \texttt{vLLM} \citep{DBLP:conf/sosp/KwonLZ0ZY0ZS23} (tensor model parallel size 1), using a maximum of 81{,}920 batched tokens and GPU memory utilization 0.3.
For log-prob computation, we use micro-batch size 8 per GPU for the actor and 16 per GPU for the teacher model.
We disable padding removal in the model implementation and use sequence parallel size 1.
The teacher-anchor coefficient $\beta$ (Eq.~\eqref{eq:final_objective_beta}) controls the weight on the $\mathcal{L}_{\text{revKL}}$ loss term.

\subsubsection{Reward weights.}
\label{app:reward_weights}

Our total reward combines three components (Eq.~\eqref{eq:reward_total}): $r_{\text{lang}}$, $r_{\text{3g}}$, and $r_{\text{judge}}$.
We choose weights to (i) keep the three terms on comparable scales and (ii) avoid over-optimizing any single proxy.

\paragraph{Practical heuristic.}
All three reward components are bounded in $[0,1]$ by construction.
We start from equal weights $(1,1,1)$ and adjust only $\lambda_{\text{lang}}$ downward to prevent the language-adherence term from dominating early training (which can reduce evidence usage), while keeping the grounding-related terms balanced.
We validate this choice with a small sanity check on a held-out development split by comparing a few candidate settings,
$\lambda_{\text{lang}}\in\{0.25,0.5,1.0\}$ with $\lambda_{\text{3g}}{=}\lambda_{\text{judge}}{=}1$,
under the same training budget.

\paragraph{Final choice.}
We use $\lambda_{\text{lang}}{=}0.5$, $\lambda_{\text{3g}}{=}1$, $\lambda_{\text{judge}}{=}1$ for all experiments, which provides a stable trade-off between language adherence and evidence-grounded correctness.

\paragraph{Length penalty.}
A length penalty is incorporated into the language consistency reward aiming to control the length of LLM's answer. Concretely, the implemented language-consistency component is $\lambda_{\text{lang}}\,r_{\text{lang}}(y,a^\star)+r_{\text{len}}(y)$, where the bounded length score
$r_{\text{len}}(y)=\max\{0,\,1-10^{-4}\max(0,\,|y|-3500)\}$
equals $1$ for answers within 3{,}500 characters and decays linearly beyond this threshold, discouraging excessively long outputs while leaving shorter responses unaffected.

\paragraph{Language-consistency check on short texts.}
When computing the training reward $r_{\text{lang}}$ (Eq.~\eqref{eq:reward_lang}), the fastText LID comparison between the generated answer and the reference answer is skipped for very short strings, on which LID is unreliable: if either text is shorter than 10 characters, the response is treated as language-consistent.

\subsubsection{Compute and randomness reporting.}
All large-scale RL runs are computationally expensive; following common practice in recent LLM RL work \citep{DBLP:conf/iclr/AgarwalVZSGGB24, DBLP:journals/corr/abs-2501-12948, zhang2025think}, we report point estimates from a single (default) random seed.

\subsubsection{Baseline hyperparameters and tuning}
\label{sec:app:expimplementation_details:baseline_hparams}

\paragraph{Shared setup.}
Unless otherwise specified, we keep the data preprocessing, decoding/truncation rules, context lengths, and the RL rollout configuration consistent across RL-based methods, and we report any method-specific coefficients explicitly.

\paragraph{Teacher-anchor coefficient $\beta$.}
The coefficient $\beta$ (Eq.~\eqref{eq:final_objective_beta}) controls the strength of the reverse-KL teacher anchor $\mathcal{L}_{\text{revKL}}$.
We select $\beta$ via the grid described in Appendix~\ref{sec:app:exp:exp:hparams}; this yields $\beta=0.02$, which we use for all backbone--teacher pairs and in every reported \model\ run.

\paragraph{OPD baseline (no $\beta$).}
Our OPD baseline uses only the teacher-distillation loss $\mathcal{L}_{\text{revKL}}$ and does \emph{not} use a task reward (hence \textbf{no $\beta$} teacher-anchor hyperparameter).
OPD follows the same shared RL optimization/rollout configuration as other RL baselines, with its method-specific objective defined in the main text.

\paragraph{SFT baselines.}
When we include a task-SFT baseline, we fine-tune the same instruction-tuned checkpoint using LoRA with
learning rate $1\times 10^{-4}$,
LoRA rank $r=16$ and $\alpha=32$,
for 3 epochs,
with per-device batch size 2 on 8 GPUs (single node), using the same tokenizer and data preprocessing pipeline as other methods.
We use this configuration as it follows standard hyperparameter choices widely adopted in instruction-tuning practice for LoRA-based SFT \citep{DBLP:conf/emnlp/HuWLXLB0PL23}.
The SFT results reported in our tables are obtained by fine-tuning a \emph{separate} model on the training split of each in-domain language (one specialized model per language), rather than training a single model jointly on the mixture of all in-domain languages; we leave a systematic comparison with joint multilingual SFT to future work.
Where feasible, we additionally run a full-parameter SFT baseline in a representative setting (BioASQ-ENKB5 with \texttt{Llama-3.2-3B-Instruct} and \texttt{Qwen3-4B-Instruct-2507}), using the same set of hyperparameters as that of \model. As shown in \Cref{tab:sft}, its performance is \textbf{even worse} and does not alter the trends reported in the main tables.

\begin{table}[tbp]
    \centering
    \caption{SFT with different tuning approaches, in BioASQ-ENKB5, ALL-AVG as metric.}
    \label{tab:sft}
    \begin{tabular}{c|c|c}
    \hline
    Tuning approach/Backbone & \texttt{Llama-3.2-3B-Instruct} & \texttt{Qwen3-4B-Instruct-2507} \\ \hline
    SFT (LoRA) & 63.14 & 68.46 \\ 
    SFT (\textbf{Full fine-tuning}) & 60.84 & 66.55 \\ \hline
    \model & 74.84 & 82.85 \\ \hline 
    \end{tabular}
\end{table}

\paragraph{KD baselines.}
For the knowledge distillation baseline, we adopt standard soft-target KD hyperparameters, distillation temperature $T{=}2$ and an equal ($\alpha{=}0.5$) mixing of the distillation and cross-entropy losses, commonly used in the distillation literature \citep{DBLP:journals/corr/HintonVD15, sun2019patient}, adapted to our decoder-only generative setting, and keep them fixed across datasets/backbones.



\subsection{Main experiments}
\label{sec:app:exp:experiment:main}
For space reasons, only \Cref{tab:llama-bioasq} is kept in the main text. We provide here the remaining three of the four BioASQ/Hotpot main-experiment tables: BioASQ-ENKB5 with the Qwen-4B student (\Cref{tab:qwen-bioasq}), and Hotpot-ENKB5 with both the Llama-3B and Qwen-4B students (\Cref{tab:llama-hotpotqa,tab:qwen-hotpotqa}). All three reproduce the qualitative conclusions reported in the main text.

\paragraph{On the language-consistency plateau in Hotpot-ENKB5 (a LID artifact).}
A striking pattern in \Cref{tab:llama-hotpotqa,tab:qwen-hotpotqa} is that \emph{every} method (including the frozen 70B/30B teachers) clusters at almost identical LC values on several languages (e.g., \textit{id}~$\approx$~55\%, \textit{vi}~$\approx$~71\%, and \textit{en} pinned at 92.90/92.91). Ten systems agreeing to two decimal places cannot reflect genuine model differences; it is a systematic artifact of the fastText LID classifier on this data. Two well-documented fastText failure modes explain it: Indonesian is frequently confused with Malay (\textit{id}/\textit{ms}), and short answers that contain Latin-script English entities (common for HotpotQA bridge answers) are often mislabeled as English. Because we construct the data so that the reference-answer language equals the query language ($\mathrm{LID}(a^\star)\equiv\ell(q)$; Section~\ref{sec:method:reward}), a cleaner alternative that sidesteps LID noise on short gold strings is to score $\mathbb{I}[\mathrm{LID}(y)=\ell(q)]$ against the \emph{known} query language directly (optionally with light manual verification); we report the $\mathrm{LID}(a^\star)$ form for consistency with the training reward and treat this Hotpot LC column as saturated/uninformative rather than as evidence of method parity. It does not affect the char 3-gram, judge, or Composite comparisons, on which the methods separate clearly.

\begin{table*}[!t]
\centering
\scriptsize
\renewcommand{\arraystretch}{0.95}  
\setlength{\tabcolsep}{3pt} 
\caption{Performance comparison in BioASQ-ENKB5, with \texttt{Qwen3-4B-Instruct-2507} as the backbone, and \texttt{Qwen3-30B-A3B-Instruct-2507} as the teacher. \textbf{Bold} is the best, and \underline{underline} is the runner-up. The teacher (\texttt{Qwen3-30B-A3B-Instruct-2507}) and base (\texttt{Qwen3-4B-Instruct-2507}) rows are reference-only and are \emph{not} ranked.
\vspace{-2mm}
}
\resizebox{\linewidth}{!}{%
\begin{tabular}{l|rrrrr|rrr|r}
\hline
\multirow{2}{*}{Methods} &
\multicolumn{5}{c|}{In-Domain Languages} &
\multicolumn{3}{c|}{Out-of-Domain Languages} &
\multirow{2}{*}{ALL-AVG} \\
\cline{2-9}
& id & ko & th & vi & ID-AVG
& en & ja & OOD-AVG & \\
\hline
\multicolumn{10}{c}{\textit{Metric: language consistency (LC, \%)}} \\
\hline
\texttt{Qwen3-30B-A3B-Instruct-2507}   &88.45&97.88&99.15&97.24&95.68 &99.78&99.25&99.52 &96.96 \\
\hline
\rowcolor{gray!8}
\texttt{Qwen3-4B-Instruct-2507}   &89.30&95.76&\underline{99.25}&98.09&95.60 &\underline{99.78}&99.25&\underline{99.52} &96.91 \\
Translate  &90.04&\textbf{98.52}&\textbf{99.26}&\textbf{98.41}&\textbf{96.56} &\underline{99.78}&\underline{99.26}&\underline{99.52} &\textbf{97.55} \\
Prompt-Control    &\textbf{90.25}&96.93&\textbf{99.26}&\textbf{98.41}&96.21 &\textbf{99.79}&\textbf{99.36}&\textbf{99.58} &97.33 \\
DIT          &88.67&97.25&\textbf{99.26}&97.99&95.79 &\textbf{99.79}&98.62&99.21 &96.93 \\
SFT                  &84.75&96.93&97.14&95.97&93.70 &99.05&98.62&98.84 &95.41 \\
Knowledge distillation                        &88.24&97.35&98.94&97.46&95.50 &\textbf{99.79}&97.88&98.84 &96.61 \\
SDFT         &87.81&97.35&98.94&97.45&95.39 &\underline{99.78}&98.30&99.04 &96.61 \\
Naive-RL      &\underline{90.15}&\underline{98.20}&\textbf{99.26}&\underline{98.31}&96.48 &99.58&\textbf{99.36}&99.47 &97.48 \\
On-Policy Distillation   &89.61&97.24&\underline{99.25}&96.92&95.76 &99.68&98.62&99.15 &96.89 \\
\rowcolor{gray!15}
\model &\underline{90.15}&\textbf{98.52}&\textbf{99.26}&98.20&\underline{96.53} &99.58&\textbf{99.36}&99.47 &\underline{97.51} \\

\hline
\multicolumn{10}{c}{\textit{Metric: char 3-gram recall (\%)}} \\
\hline
\texttt{Qwen3-30B-A3B-Instruct-2507}   &81.10&48.74&59.19&77.17&66.55 &82.43&45.66&64.05 &65.72 \\
\hline
\rowcolor{gray!8}
\texttt{Qwen3-4B-Instruct-2507}   &75.06&42.66&52.82&69.65&60.05 &77.56&38.50&58.03 &59.38 \\
Translate &72.76&37.37&47.84&69.52&56.87 &77.56&35.75&56.66 &56.80 \\
Prompt-Control    &76.02&43.36&53.01&72.55&61.24 &77.91&39.42&58.67 &60.38 \\
DIT          &77.98&44.35&52.67&73.00&62.00 &84.05&40.29&62.17 &62.06 \\
SFT                 &41.57&26.17&32.83&41.55&35.53 &50.43&25.89&38.16 &36.41 \\
Knowledge distillation                        &80.06&46.54&58.01&76.97&65.40 &82.30&\underline{42.31}&62.31 &64.37 \\
SDFT         &79.60&45.59&56.88&75.49&64.39 &81.74&41.36&61.55 &63.44 \\
Naive-RL       &\underline{81.14}&46.02&58.70&76.94&65.70 &82.83&42.49&\underline{62.66} &64.69 \\
On-Policy Distillation   &80.80&\underline{47.62}&\underline{58.83}&\underline{77.28}&\underline{66.13} &\underline{84.50}&40.45&62.48 &\underline{64.91} \\
\rowcolor{gray!15}
\model &\textbf{85.15}&\textbf{49.73}&\textbf{62.72}&\textbf{81.51}&\textbf{69.78} &\textbf{87.43}&\textbf{45.93}&\textbf{66.68} &\textbf{68.75} \\
\hline

\multicolumn{10}{c}{\textit{Metric: LLM-judge (\%)}} \\
\hline
\texttt{Qwen3-30B-A3B-Instruct-2507}   &79.65&76.63&79.17&79.01&78.62 &81.21&76.62&78.92 &78.72 \\
\hline
\rowcolor{gray!8}
\texttt{Qwen3-4B-Instruct-2507}   &79.50&73.82&77.34&80.24&77.73 &83.61&76.41&80.01 &78.49 \\
Translate &77.74&64.49&71.33&75.72&72.32 &83.61&70.06&76.84 &73.83 \\
Prompt-Control    &78.81&73.74&76.96&78.62&77.03 &83.96&75.65&79.81 &77.96 \\
DIT          &75.49&73.31&75.16&74.52&74.62 &76.27&73.32&74.80 &74.68 \\
SFT                  &75.74&70.02&70.23&72.67&72.17 &79.45&73.31&76.38 &73.57 \\
Knowledge distillation                        &76.78&72.06&75.66&76.89&75.35 &81.04&74.23&77.64 &76.11 \\
SDFT         &77.55&72.49&75.78&77.16&75.75 &80.41&74.32&77.37 &76.29 \\
Naive-RL       &\textbf{86.26}&\textbf{80.30}&\textbf{82.86}&\textbf{84.82}&\textbf{83.56} &\textbf{86.72}&\textbf{81.72}&\textbf{84.22} &\textbf{83.78} \\
On-Policy Distillation   &77.28&73.60&76.12&77.54&76.14 &80.12&75.87&78.00 &76.76 \\

\rowcolor{gray!15}
\model &\underline{83.97}&\underline{79.83}&\underline{80.73}&\underline{83.21}&\underline{81.94}&\underline{85.91}&\underline{80.19}&\underline{83.05} &\underline{82.31} \\
\hline

\multicolumn{10}{c}{\textit{Metric: Composite (\%)}} \\
\hline

\texttt{Qwen3-30B-A3B-Instruct-2507}   &83.07&74.42&79.17&84.47&80.28 &87.81&73.84&80.83 &80.46 \\
\hline
\rowcolor{gray!8}
\texttt{Qwen3-4B-Instruct-2507}   &81.29&70.75&76.47&82.66&77.79 &86.98&71.39&79.19 &78.26 \\
Translate &80.18&66.79&72.81&81.22&75.25 &86.98&68.36&77.67 &76.06 \\
Prompt-Control    &81.69&71.34&76.41&83.19&78.16 &87.22&71.48&79.35 &78.56 \\
DIT          &80.71&71.64&75.70&81.84&77.47 &86.70&70.74&78.72 &77.89 \\
SFT                  &67.35&64.37&66.73&70.06&67.13 &76.31&65.94&71.13 &68.46 \\
Knowledge distillation                        &81.69&71.98&77.54&83.77&78.75 &87.71&71.47&79.59 &79.03 \\
SDFT         &81.65&71.81&77.20&83.37&78.51 &87.31&71.33&79.32 &78.78 \\
Naive-RL       &\underline{85.85}&\underline{74.84}&\underline{80.27}&\underline{86.69}&\underline{81.91} &\underline{89.71}&\underline{74.52}&\underline{82.12} &\underline{81.98} \\
On-Policy Distillation   &82.56&72.82&78.07&83.91&79.34 &88.10&71.65&79.87 &79.52\\
\rowcolor{gray!15}
\model &\textbf{86.42}&\textbf{76.03}&\textbf{80.90}&\textbf{87.64}&\textbf{82.75} &\textbf{90.97}&\textbf{75.16}&\textbf{83.07} &\textbf{82.85} \\
\hline
\end{tabular}}
\vspace{-2mm}
\label{tab:qwen-bioasq}
\end{table*}

\begin{table*}[!t]
\centering
\scriptsize
\renewcommand{\arraystretch}{0.95}  
\setlength{\tabcolsep}{3pt} 
\caption{Performance comparison in Hotpot-ENKB5, with \texttt{Llama-3.2-3B-Instruct} as the backbone, and \texttt{Llama-3.3-70B-Instruct} as the teacher. \textbf{Bold} is the best, and \underline{underline} is the runner-up. The teacher (\texttt{Llama-3.3-70B-Instruct}) and base (\texttt{Llama-3.2-3B-Instruct}) rows are reference-only and are \emph{not} ranked. See the note in \Cref{sec:app:exp:experiment:main} on the language-consistency plateau in this table.
}
\resizebox{\linewidth}{!}{%
\begin{tabular}{l|rrrrr|rrr|r}
\hline
\multirow{2}{*}{Methods} &
\multicolumn{5}{c|}{In-Domain Languages} &
\multicolumn{3}{c|}{Out-of-Domain Languages} &
\multirow{2}{*}{ALL-AVG} \\
\cline{2-9}
& id & ko & th & vi & ID-AVG
& en & ja & OOD-AVG & \\
\hline
\multicolumn{10}{c}{\textit{Metric: language consistency (\%)}} \\
\hline

\texttt{Llama-3.3-70B-Instruct}   &55.00&99.22&99.09&70.40&80.93 &92.63&99.59&96.11 &85.99 \\
\hline
\rowcolor{gray!8}
\texttt{Llama-3.2-3B-Instruct}   &\underline{56.72}&99.63&99.27&\textbf{72.09}&\textbf{81.93} &92.90&99.81&96.36 &\textbf{86.54} \\
Translate &55.72&97.27&98.91&70.91&80.70 &\underline{92.90}&99.55&96.23 &85.88 \\
Prompt-Control    &55.09&99.63&\underline{99.81}&\underline{71.18}&81.43 &\underline{92.90}&\underline{99.90}&\underline{96.40} &86.42 \\
DIT          &55.81&99.36&\underline{99.81}&70.72&81.43 &\underline{92.90}&99.72&96.32 &86.39 \\
SFT                  &\textbf{63.73}&90.00&95.18&70.55&79.87 &89.64&92.00&90.82 &83.52 \\
Knowledge distillation                        &55.55&\underline{99.72}&\underline{99.81}&\underline{71.18}&\underline{81.57} &\underline{92.90}&\underline{99.90}&\underline{96.40} &\underline{86.51} \\
SDFT         &55.00&99.54&99.72&71.09&81.34 &\underline{92.90}&\underline{99.90}&\underline{96.40} &86.36 \\
Naive-RL      &54.64&\textbf{99.73}&\textbf{99.82}&71.00&81.30 &\textbf{92.91}&\textbf{99.91}&\textbf{96.41} &86.34 \\
On-Policy Distillation   &54.73&99.64&\textbf{99.82}&71.09&81.32 &\textbf{92.91}&\textbf{99.91}&\textbf{96.41} &86.35 \\
\rowcolor{gray!15}
\model &54.45&\underline{99.72}&\underline{99.81}&\underline{71.18}&81.29 &\underline{92.90}&\underline{99.90}&\underline{96.40} &86.37 \\

\hline
\multicolumn{10}{c}{\textit{Metric: char 3-gram recall (\%)}} \\
\hline

\texttt{Llama-3.3-70B-Instruct}   &78.89&44.26&44.57&71.21&59.73 &90.66&49.73&70.20 &63.22 \\
\hline
\rowcolor{gray!8}
\texttt{Llama-3.2-3B-Instruct}   &44.99&18.81&24.18&44.38&33.09 &56.22&25.16&40.69 &35.62 \\
Translate &40.06&11.37&18.28&35.13&26.21 &56.22&18.07&37.15 &29.86 \\
Prompt-Control    &48.67&20.73&26.20&49.60&36.30 &62.57&27.99&45.28 &39.29 \\
DIT          &50.37&17.83&19.85&51.89&34.99 &58.25&24.19&41.22 &37.06 \\
SFT                 &57.21&\textbf{34.02}&39.31&56.14&46.67 &75.12&\underline{42.16}&58.64 &50.66 \\
Knowledge distillation                        &63.04&27.72&34.63&59.10&46.12 &76.16&35.37&55.77 &49.34 \\
SDFT         &67.25&29.87&35.37&64.87&49.34 &81.00&38.09&59.55 &52.74 \\
Naive-RL       &69.87&29.58&\underline{43.15}&\underline{66.87}&\underline{52.37} &\underline{83.15}&37.56&\underline{60.36} &\underline{55.03} \\
On-Policy Distillation   &\underline{70.01}&31.98&35.61&65.42&50.76 &79.72&39.39&59.56 &53.69 \\
\rowcolor{gray!15}
\model &\textbf{72.20}&\underline{33.77}&\textbf{45.37}&\textbf{68.50}&\textbf{54.96} &\textbf{83.47}&\textbf{43.47}&\textbf{63.47} &\textbf{57.80} \\
\hline
\multicolumn{10}{c}{\textit{Metric: LLM-judge (\%)}} \\
\hline

\texttt{Llama-3.3-70B-Instruct}   &83.74&81.25&81.31&85.13&82.86 &87.06&84.30&85.68 &83.80 \\
\hline
\rowcolor{gray!8}
\texttt{Llama-3.2-3B-Instruct}   &54.87&49.50&50.60&57.65&53.16 &59.53&50.33&54.93 &53.75 \\
Translate &44.83&19.93&25.47&39.62&32.46 &59.53&27.35&43.44 &36.12 \\
Prompt-Control    &56.93&46.84&49.11&58.96&52.96 &62.09&49.44&55.77 &53.90 \\
DIT          &52.53&43.15&35.40&55.62&46.68 &60.15&45.97&53.06 &48.80 \\
SFT                  &69.03&52.35&57.19&69.57&62.04 &81.77&59.13&70.45 &64.84 \\
Knowledge distillation                        &72.40&60.79&64.29&73.22&67.68 &79.11&63.69&71.40 &68.92 \\
SDFT         &75.02&65.99&70.53&77.74&72.32 &82.61&68.90&75.76 &73.47 \\
Naive-RL       &\textbf{79.72}&\textbf{69.80}&\underline{72.38}&\underline{80.44}&\underline{75.59} &\underline{84.91}&\underline{72.43}&\underline{78.67} &\underline{76.61} \\
On-Policy Distillation   &76.31&66.10&69.41&78.15&72.49 &82.29&68.24&75.27 &73.42 \\
\rowcolor{gray!15}
\model &\underline{79.66}&\underline{69.74}&\textbf{73.65}&\textbf{81.52}&\textbf{76.14} &\textbf{85.68}&\textbf{74.76}&\textbf{80.22} &\textbf{77.50} \\
\hline

\multicolumn{10}{c}{\textit{Metric: Composite (\%)}} \\
\hline

\texttt{Llama-3.3-70B-Instruct}   &72.54&74.91&74.99&75.58&74.51 &90.12&77.87&84.00 &77.67 \\
\hline
\rowcolor{gray!8}
\texttt{Llama-3.2-3B-Instruct}   &52.19&55.98&58.02&58.04&56.06 &69.55&58.43&64.00 &58.70 \\
Translate &46.87&42.86&47.55&48.55&46.46 &69.55&48.32&58.94 &50.62 \\
Prompt-Control    &53.56&55.73&58.37&59.91&56.90 &72.52&59.11&65.82 &59.87 \\
DIT          &52.90&53.45&51.69&59.41&54.36 &70.43&56.63&63.53 &57.42 \\
SFT                  &63.32&58.79&63.89&65.42&62.86 &82.18&64.43&73.30 &66.34 \\
Knowledge distillation                        &63.66&62.74&66.24&67.83&65.12 &82.72&66.32&74.52 &68.25 \\
SDFT         &65.76&65.13&68.54&71.23&67.67 &85.50&68.96&77.23 &70.86 \\
Naive-RL       &\underline{68.08}&\underline{66.37}&\underline{71.78}&\underline{72.77}&\underline{69.75} &\underline{86.99}&\underline{69.97}&\underline{78.48} &\underline{72.66} \\
On-Policy Distillation   &67.02&65.91&68.28&71.55&68.19 &84.97&69.18&77.08 &71.15 \\
\rowcolor{gray!15}
\model &\textbf{68.77}&\textbf{67.74}&\textbf{72.94}&\textbf{73.73}&\textbf{70.80} &\textbf{87.35}&\textbf{72.71}&\textbf{80.03} &\textbf{73.88} \\
\hline

\end{tabular}}

\label{tab:llama-hotpotqa}
\end{table*}

\begin{table*}[t]
\centering
\scriptsize
\renewcommand{\arraystretch}{0.95}  
\setlength{\tabcolsep}{3pt} 
\caption{Performance comparison in Hotpot-ENKB5, with \texttt{Qwen3-4B-Instruct-2507} as the backbone, and \texttt{Qwen3-30B-A3B-Instruct-2507} as the teacher. \textbf{Bold} is the best, and \underline{underline} is the runner-up. The teacher (\texttt{Qwen3-30B-A3B-Instruct-2507}) and base (\texttt{Qwen3-4B-Instruct-2507}) rows are reference-only and are \emph{not} ranked. See the note in \Cref{sec:app:exp:experiment:main} on the language-consistency plateau in this table.
}
\resizebox{\linewidth}{!}{%
\begin{tabular}{l|rrrrr|rrr|r}
\hline
\multirow{2}{*}{Methods} &
\multicolumn{5}{c|}{In-Domain Languages} &
\multicolumn{3}{c|}{Out-of-Domain Languages} &
\multirow{2}{*}{ALL-AVG} \\
\cline{2-9}
& id & ko & th & vi & ID-AVG
& en & ja & OOD-AVG & \\
\hline
\multicolumn{10}{c}{\textit{Metric: language consistency (\%)}} \\
\hline

\texttt{Qwen3-30B-A3B-Instruct-2507}   &55.54&99.63&99.81&71.09&81.52 &92.90&99.99&96.45 &86.49 \\
\hline
\rowcolor{gray!8}
\texttt{Qwen3-4B-Instruct-2507}   &55.45&\underline{99.72}&\underline{99.81}&70.90&81.47 &\underline{92.90}&99.72&96.31 &86.42 \\
Translate &55.09&\textbf{99.73}&\textbf{99.82}&71.18&81.46 &\underline{92.90}&\textbf{99.91}&\textbf{96.41} &86.44 \\
Prompt-Control    &\underline{56.54}&\underline{99.72}&\underline{99.81}&70.81&\underline{81.72} &\underline{92.90}&99.81&96.36 &\textbf{86.60} \\
DIT          &56.00&99.18&\underline{99.81}&\underline{71.27}&81.57 &\underline{92.90}&99.54&96.22 &86.45 \\
SFT                  &\textbf{69.36}&93.09&95.18&\textbf{77.45}&\textbf{83.77} &91.09&90.91&91.00 &86.18 \\
Knowledge distillation                        &54.81&\underline{99.72}&\underline{99.81}&71.00&81.34 &\underline{92.90}&\underline{99.90}&\underline{96.40} &86.36 \\
SDFT         &54.90&\underline{99.72}&\underline{99.81}&71.09&81.38 &\underline{92.90}&99.81&96.36 &86.37 \\
Naive-RL      &28.27&92.55&92.00&50.55&65.84 &82.18&93.00&87.59 &73.09 \\
On-Policy Distillation   &54.90&99.45&\underline{99.81}&71.00&81.29 &\underline{92.90}&\underline{99.90}&\underline{96.40} &86.33 \\
\rowcolor{gray!15}
\model &55.60&\textbf{99.73}&\textbf{99.82}&71.09&81.41 &\textbf{92.91}&\textbf{99.91}&\textbf{96.41} &\underline{86.51} \\

\hline
\multicolumn{10}{c}{\textit{Metric: char 3-gram recall (\%)}} \\
\hline

\texttt{Qwen3-30B-A3B-Instruct-2507}   &83.02&54.37&53.03&78.16&67.15 &91.46&57.45&74.46 &69.58 \\
\hline
\rowcolor{gray!8}
\texttt{Qwen3-4B-Instruct-2507}   &79.75&44.46&43.27&74.42&60.48 &89.36&49.88&69.62 &63.52 \\
Translate &78.41&40.43&43.81&72.59&58.81 &89.36&48.34&68.85 &62.16 \\
Prompt-Control    &\underline{82.93}&45.23&44.10&\underline{77.86}&62.53 &\underline{90.88}&50.10&70.49 &65.18 \\
DIT          &\textbf{86.57}&44.78&43.01&\textbf{79.10}&63.37 &96.76&47.69&\underline{72.23} &66.32 \\
SFT                 &65.71&43.24&45.11&64.56&54.66 &76.02&46.41&61.22 &56.84 \\
Knowledge distillation                        &80.20&45.97&47.33&74.75&62.06 &89.43&51.20&70.32 &64.81 \\
SDFT         &80.60&46.15&47.88&75.72&62.59 &90.32&52.13&71.23 &65.47 \\
Naive-RL       &79.70&\underline{50.12}&\textbf{55.40}&73.53&\underline{64.69} &89.67&\underline{53.34}&71.51 &\underline{66.96} \\
On-Policy Distillation   &80.25&48.09&48.78&76.39&63.38 &\textbf{91.30}&51.37&71.34 &66.03 \\
\rowcolor{gray!15}
\model &79.96&\textbf{51.41}&\underline{55.22}&74.80&\textbf{65.35} &90.83&\textbf{55.54}&\textbf{73.19} &\textbf{67.96} \\
\hline
\multicolumn{10}{c}{\textit{Metric: LLM-judge (\%)}} \\
\hline

\texttt{Qwen3-30B-A3B-Instruct-2507}   &79.98&74.78&76.64&79.52&77.73 &84.23&74.21&79.22 &78.23 \\
\hline
\rowcolor{gray!8}
\texttt{Qwen3-4B-Instruct-2507}   &76.32&69.39&73.14&77.09&73.98 &81.59&72.11&76.85 &74.94 \\
Translate &75.53&57.49&65.03&73.68&67.93 &81.59&64.85&73.22 &69.70 \\
Prompt-Control    &74.86&69.21&71.24&75.09&72.60 &80.13&71.15&75.64 &73.61 \\
DIT          &74.67&70.72&72.14&74.58&73.03 &77.50&71.33&74.42 &73.49 \\
SFT                  &77.79&63.61&65.49&78.40&71.32 &86.23&66.84&76.54 &73.06 \\
Knowledge distillation                        &77.85&69.24&72.52&79.08&74.67 &83.88&71.36&77.62 &75.66 \\
SDFT         &77.96&69.84&73.18&79.06&75.01 &83.90&71.80&77.85 &75.96 \\
Naive-RL       &\textbf{86.41}&\underline{79.68}&\underline{80.90}&\textbf{87.15}&\underline{83.54} &\underline{90.07}&\underline{81.61}&\underline{85.84} &\underline{84.30} \\
On-Policy Distillation   &78.53&73.22&74.25&79.48&76.37 &84.31&72.90&78.61 &77.12 \\
\rowcolor{gray!15}
\model &\underline{85.74}&\textbf{80.27}&\textbf{81.88}&\underline{86.79}&\textbf{83.67} &\textbf{91.42}&\textbf{81.94}&\textbf{86.68} &\textbf{84.67} \\
\hline

\multicolumn{10}{c}{\textit{Metric: Composite (\%)}} \\
\hline

\texttt{Qwen3-30B-A3B-Instruct-2507}   &72.85&76.26&76.49&76.26&75.46 &89.53&77.22&83.37 &78.10 \\
\hline
\rowcolor{gray!8}
\texttt{Qwen3-4B-Instruct-2507}   &70.51&71.18&72.07&74.14&71.97 &87.95&73.90&80.93 &74.96 \\
Translate &69.68&65.88&69.55&72.48&69.40 &87.95&71.03&79.49 &72.76 \\
Prompt-Control    &71.44&71.39&71.72&74.59&72.28 &87.97&73.69&80.83 &75.13 \\
DIT          &\underline{72.41}&71.56&71.65&74.98&72.65 &89.05&72.85&80.95 &75.42 \\
SFT                  &70.95&66.65&68.59&73.47&69.92 &84.45&68.05&76.25 &72.03 \\
Knowledge distillation                        &70.95&71.64&73.22&74.94&72.69 &88.74&74.15&81.45 &75.61 \\
SDFT         &71.15&71.90&73.62&75.29&72.99 &89.04&74.58&81.81 &75.93 \\
Naive-RL       &64.79&\underline{74.12}&\underline{76.10}&70.41&71.36 &87.31&\underline{75.98}&81.65 &74.79 \\
On-Policy Distillation   &71.23&73.59&74.28&\underline{75.62}&\underline{73.68} &\underline{89.50}&74.72&\underline{82.11} &\underline{76.49} \\
\rowcolor{gray!15}
\model &\textbf{73.57}&\textbf{77.14}&\textbf{78.97}&\textbf{77.56}&\textbf{76.81} &\textbf{91.72}&\textbf{79.13}&\textbf{85.43} &\textbf{79.68} \\
\hline

\end{tabular}}

\label{tab:qwen-hotpotqa}
\end{table*}

\newpage
\clearpage

\subsection{Ablation studies}
\label{sec:app:exp:ablation studies}

More ablation studies of other three metrics: \emph{language consistency}, \emph{char 3-gram recall}, and \emph{LLM-judge} score are shown in \Cref{table:ablation-Language Consistency,table:ablation-char 3-gram recall,table:ablation-LLM-judge}.

\begin{table*}[tbp]
    \centering
    \caption{Ablation studies (\textit{\textbf{Language Consistency}, ALL-AVG} metric, \%).\\
    }
    \vspace{-2mm} 
    \label{table:ablation-Language Consistency}%
    \renewcommand{\arraystretch}{1} 
    \setlength{\tabcolsep}{2pt} 
    \resizebox{\linewidth}{!}{%
    \begin{tabular}{@{}l|cccc|cc|cc|c@{}}
    \toprule
    \multirow{2}*{Variants}
    &Language&Char 3-gram& LLM-judge & On-Policy  &\multicolumn{2}{c|}{BioASQ-ENKB5} &\multicolumn{2}{c|}{Hotpot-ENKB5} & \multirow{2}*{Average}  \\
    &Reward&Reward& Reward & Distillation & Llama$_{3B}$ & Qwen$_{4B}$ & Llama$_{3B}$  & Qwen$_{4B}$ &    \\
    \midrule
    RL-only
    &$\checkmark$&$\checkmark$ & $\checkmark$ & $\times$  &97.02&97.48&86.34&73.09&88.48 \\ 
    OPD-only
    &$\times$&$\times$ & $\times$ & $\checkmark$ &97.16&96.89&86.35&86.37&91.69 \\
    No language reward
    &$\times$&$\checkmark$&$\checkmark$&$\checkmark$ &96.54&97.46&86.40&86.24&91.66  \\
    No 3-gram reward
    &$\checkmark$&$\times$&$\checkmark$&$\checkmark$ &96.44&96.05&\textbf{92.29}&\textbf{92.44}&\textbf{94.31}  \\
    No LLM-judge reward
    &$\checkmark$&$\checkmark$&$\times$&$\checkmark$ &\underline{97.37}&\textbf{97.55}&\underline{86.56}&\underline{86.62}&\underline{92.03} \\
    \midrule
    \model
    & $\checkmark$ & $\checkmark$ &$\checkmark$&$\checkmark$   &\textbf{97.42}&\underline{97.51}&86.33&86.41&91.92  \\  
    \bottomrule
    \end{tabular}}
\end{table*}

\begin{table*}[tbp]
    \centering
    \caption{Ablation studies (\textit{\textbf{char 3-gram recall}, ALL-AVG} metric, \%).\\
    }
    \vspace{-2mm} 
    \label{table:ablation-char 3-gram recall}%
    \renewcommand{\arraystretch}{1} 
    \setlength{\tabcolsep}{2pt} 
    \resizebox{\linewidth}{!}{%
    \begin{tabular}{@{}l|cccc|cc|cc|c@{}}
    \toprule
    \multirow{2}*{Variants}
    &Language&Char 3-gram& LLM-judge & On-Policy  &\multicolumn{2}{c|}{BioASQ-ENKB5} &\multicolumn{2}{c|}{Hotpot-ENKB5} & \multirow{2}*{Average}  \\
    &Reward&Reward& Reward & Distillation & Llama$_{3B}$ & Qwen$_{4B}$ & Llama$_{3B}$  & Qwen$_{4B}$ &    \\
    \midrule
    RL-only
    &$\checkmark$&$\checkmark$ & $\checkmark$ & $\times$  &51.77&64.69&55.03&66.96&59.61 \\ 
    OPD-only
    &$\times$&$\times$ & $\times$ & $\checkmark$ &48.01&64.91&53.69&66.03&58.16 \\
    No language reward
    &$\times$&$\checkmark$&$\checkmark$&$\checkmark$ &55.23&\underline{69.27}&55.74&67.66&61.98  \\
    No 3-gram reward
    &$\checkmark$&$\times$&$\checkmark$&$\checkmark$ &34.49&43.06&41.71&49.85&42.28  \\
    No LLM-judge reward
    &$\checkmark$&$\checkmark$&$\times$&$\checkmark$ &\textbf{62.99}&\textbf{71.76}&\textbf{65.98}&\textbf{75.30}&\textbf{69.01} \\
    \midrule
    \model
    & $\checkmark$ & $\checkmark$ &$\checkmark$&$\checkmark$   &\underline{58.09}&68.75&\underline{57.80}&\underline{67.96}&\underline{63.15}  \\  
    \bottomrule
    \end{tabular}}
\end{table*}

\begin{table*}[tbp]
    \centering
    \caption{Ablation studies (\textit{\textbf{LLM-judge}, ALL-AVG} metric, \%).\\
    }
    \vspace{-2mm} 
    \label{table:ablation-LLM-judge}%
    \renewcommand{\arraystretch}{1} 
    \setlength{\tabcolsep}{2pt} 
    \resizebox{\linewidth}{!}{%
    \begin{tabular}{@{}l|cccc|cc|cc|c@{}}
    \toprule
    \multirow{2}*{Variants}
    &Language&Char 3-gram& LLM-judge & On-Policy  &\multicolumn{2}{c|}{BioASQ-ENKB5} &\multicolumn{2}{c|}{Hotpot-ENKB5} & \multirow{2}*{Average}  \\
    &Reward&Reward& Reward & Distillation & Llama$_{3B}$ & Qwen$_{4B}$ & Llama$_{3B}$  & Qwen$_{4B}$ &    \\
    \midrule
    RL-only
    &$\checkmark$&$\checkmark$ & $\checkmark$ & $\times$  &70.38&83.78&76.62&84.30&78.77 \\ 
    OPD-only
    &$\times$&$\times$ & $\times$ & $\checkmark$ &64.85&76.76&73.42&77.12&73.04 \\
    No language reward
    &$\times$&$\checkmark$&$\checkmark$&$\checkmark$ &\underline{70.63}&\underline{84.45}&\underline{77.08}&\underline{85.20}&\underline{79.34}  \\
    No 3-gram reward
    &$\checkmark$&$\times$&$\checkmark$&$\checkmark$ &\textbf{78.46}&\textbf{86.73}&76.32&\textbf{85.74}&\textbf{81.81}  \\
    No LLM-judge reward
    &$\checkmark$&$\checkmark$&$\times$&$\checkmark$ &57.28&74.95&61.03&73.47&66.68 \\
    \midrule
    \model
    & $\checkmark$ & $\checkmark$ &$\checkmark$&$\checkmark$   &69.02&82.31&\textbf{77.50}&84.67&78.38  \\  
    \bottomrule
    \end{tabular}}
\end{table*}


\subsection{Sensitivity to the teacher weight $\beta$}
\label{sec:app:exp:exp:hparams}
The main hyperparameter in TR-RAG is the teacher-anchor weight $\beta$ (Eq.~\eqref{eq:final_objective_beta}), which trades off task reward optimization against the reverse-KL teacher anchor computed on student-generated prefixes.
Intuitively, larger $\beta$ enforces closer alignment to the teacher and typically stabilizes training under early prefix deviations (e.g., language drift), while smaller $\beta$ allows more reward-driven exploration but may increase variance.
We select $\beta$ on a held-out development split from the in-domain training languages by grid search over $\beta\in\{0.2,\,0.02,\, 0.002,\,0.0002\}$, and use the same $\beta$ for all runs within a backbone--teacher pair.
Unless stated otherwise, we keep all remaining hyperparameters fixed across baselines and variants, including the rollout budget ($K$ samples per prompt for GRPO), maximum generation length, optimizer and learning-rate schedule, and the reward weights $(\lambda_{\text{lang}},\lambda_{\text{3g}},\lambda_{\text{judge}})$.
We conduct this experiment, using \texttt{Qwen3-4B-Instruct-2507} as backbone, in Hotpot-ENKB5. Results are shown in \Cref{fig:hyper sensitivity,fig:hyper sensitivity-LC,fig:hyper sensitivity-3-gram,fig:hyper sensitivity-LLM-judge}. 
As can be seen, language consistency (\Cref{fig:hyper sensitivity-LC}) and char 3-gram recall (\Cref{fig:hyper sensitivity-3-gram}) improve as we increase $\beta$ from very small values and remain stable in a moderate range (e.g., $0.002$--$0.02$). The LLM-judge score (\Cref{fig:hyper sensitivity-LLM-judge}), however, peaks at smaller $\beta$ values and strictly decreases as $\beta$ grows, reflecting a metric-specific trade-off: stronger teacher anchoring improves lexical faithfulness and language adherence at a small cost to the judge's preferred output style. A large $\beta$ (e.g., $0.2$) degrades all metrics, consistent with over-regularization that keeps the student too close to the frozen teacher and limits task-specific adaptation. Overall, TR-RAG is not overly sensitive within a broad mid-range of $\beta$, while avoiding excessively large values is important. Guided by these curves, we set $\beta=0.02$ for every backbone--teacher pair in all reported experiments.

\subsection{Impact of Teacher Quality.}
\label{sec:app:exp:teacher_quality}
To probe the boundary of \model’s reliance on teacher quality, we run a controlled ablation that replaces the frozen teacher with weaker base (non–instruction-tuned) checkpoints while keeping the student backbone and training recipe fixed. Concretely, for the Llama-3.2-3B-Instruct student, we compare a weak teacher (frozen Llama-3.2-3B) against a strong teacher (Llama-3.3-70B-Instruct); for the Qwen3-4B-Instruct-2507 student, we compare a weak teacher (frozen Qwen3-4B) against a strong teacher (Qwen3-30B-A3B-Instruct-2507). As shown in \Cref{fig:teacher_comparison_bioasq,fig:teacher_comparison_hotpotQA}, across both backbones, \model\ remains beneficial under the weak-teacher setting and improves over non-teacher baselines, while stronger teachers systematically yield larger gains, indicating that reverse-KL anchoring is not purely “teacher-capability transfer” but also a stabilizing regularizer on student-visited prefixes. At the same time, these results clarify the method’s limits: if the teacher is severely misaligned in the target language (e.g., hallucinations or poor cross-lingual competence), the anchor could propagate bias; in practice, we recommend selecting teachers with at least reasonable target-language reliability and, when feasible, validating with a small held-out subset or an alternative teacher/judge to guard against teacher-induced errors.

\begin{figure}[tbp]
  \centering
  \begin{minipage}[t]{0.39\linewidth}
    \centering
    \includegraphics[width=1\linewidth]{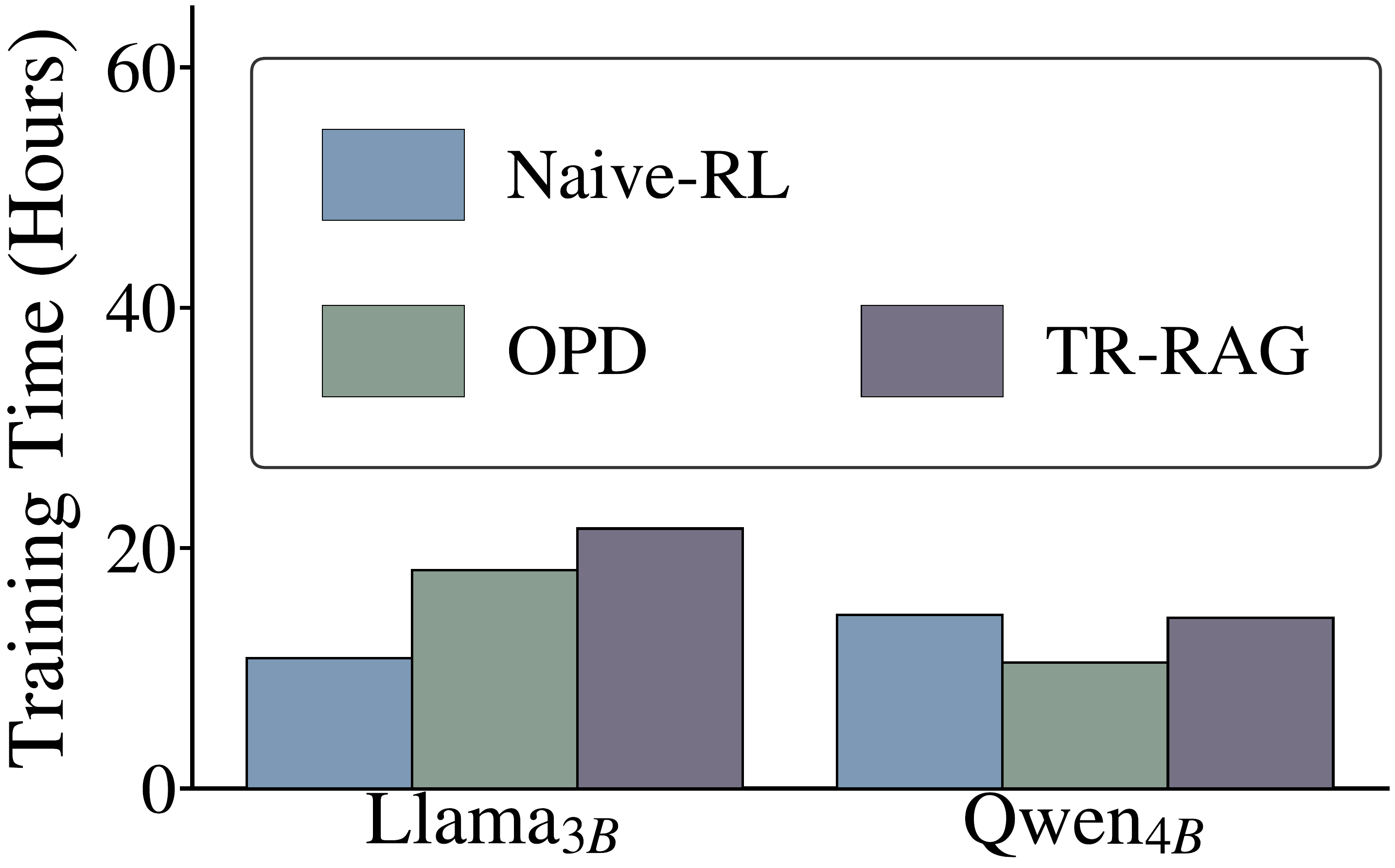}
    \caption{
    Training time (\texttt{Llama-3.2-3B} \& \texttt{Qwen3-4B} backbones), Hotpot-ENKB5.}
    \label{fig:qwen_llama_time}
  \end{minipage}
  \hfill
  \begin{minipage}[t]{0.39\linewidth}
    \centering
    \includegraphics[width=1\linewidth]{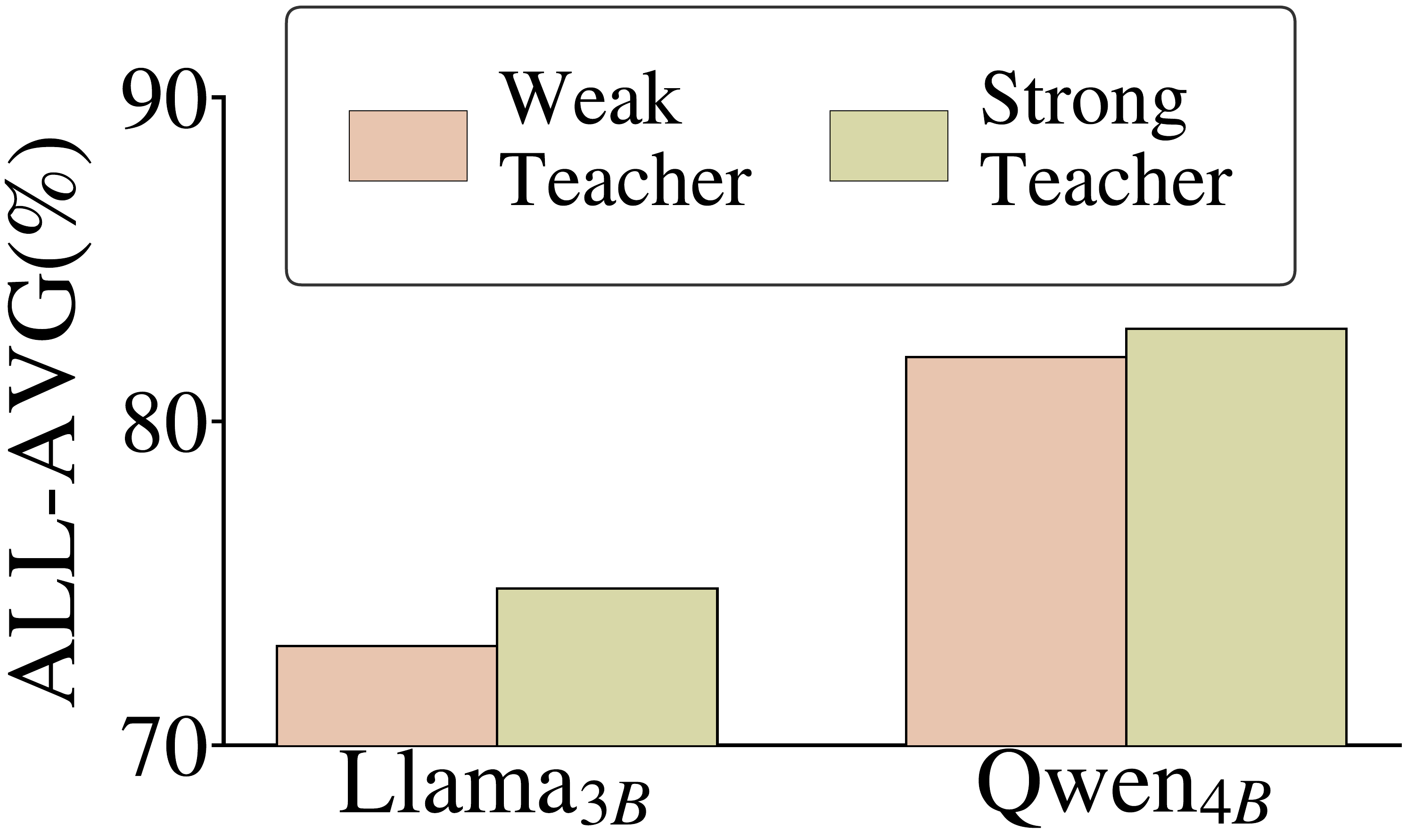}
    \caption{
    Teacher choice comparison, in BioASQ-ENKB5.}
    \label{fig:teacher_comparison_bioasq}
  \end{minipage}
\end{figure}

\begin{figure}[tbp]
  \centering
  \begin{minipage}[t]{0.31\linewidth}
    \centering
    \includegraphics[width=1\linewidth]{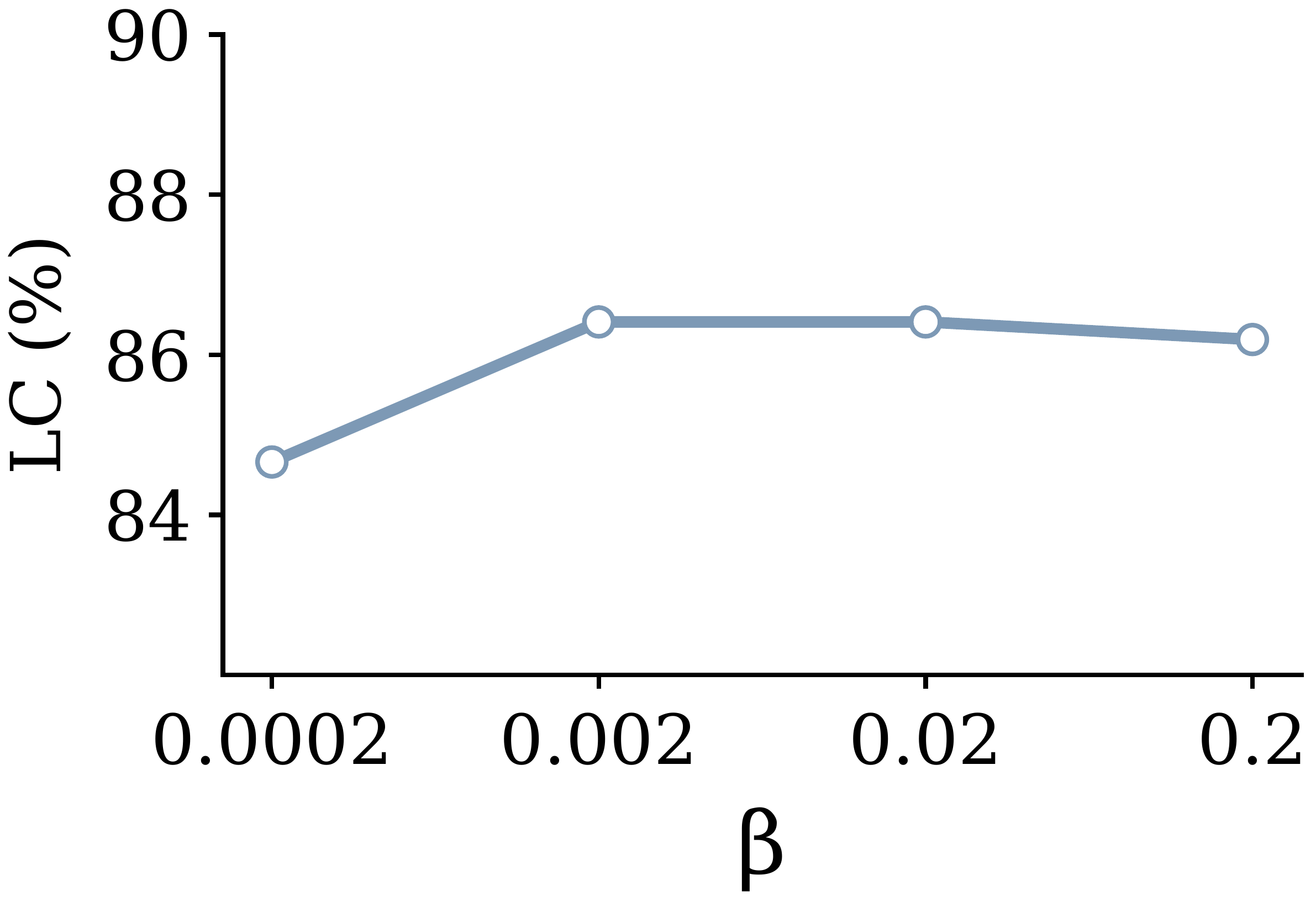}
    \caption{
    Teacher weight $\beta$ in \model, \emph{language consistency} as metric}
    \label{fig:hyper sensitivity-LC}
  \end{minipage}
  \hfill
  \begin{minipage}[t]{0.31\linewidth}
    \centering
    \includegraphics[width=1\linewidth]{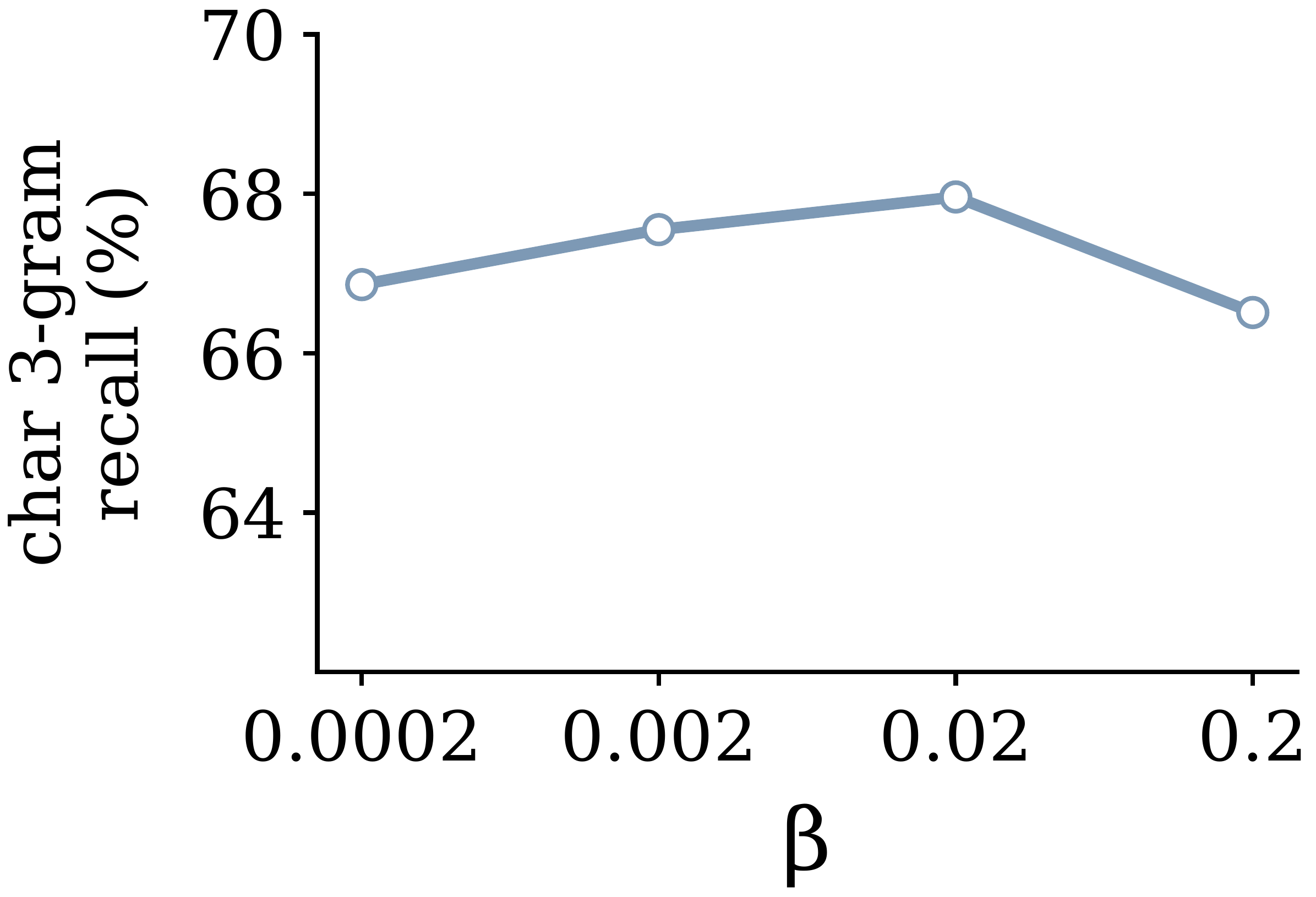}
    \caption{
    Teacher weight $\beta$ in \model, \emph{char 3-gram recall} as metric}
    \label{fig:hyper sensitivity-3-gram}
  \end{minipage}
  \hfill
  \begin{minipage}[t]{0.31\linewidth}
    \centering
    \includegraphics[width=1\linewidth]{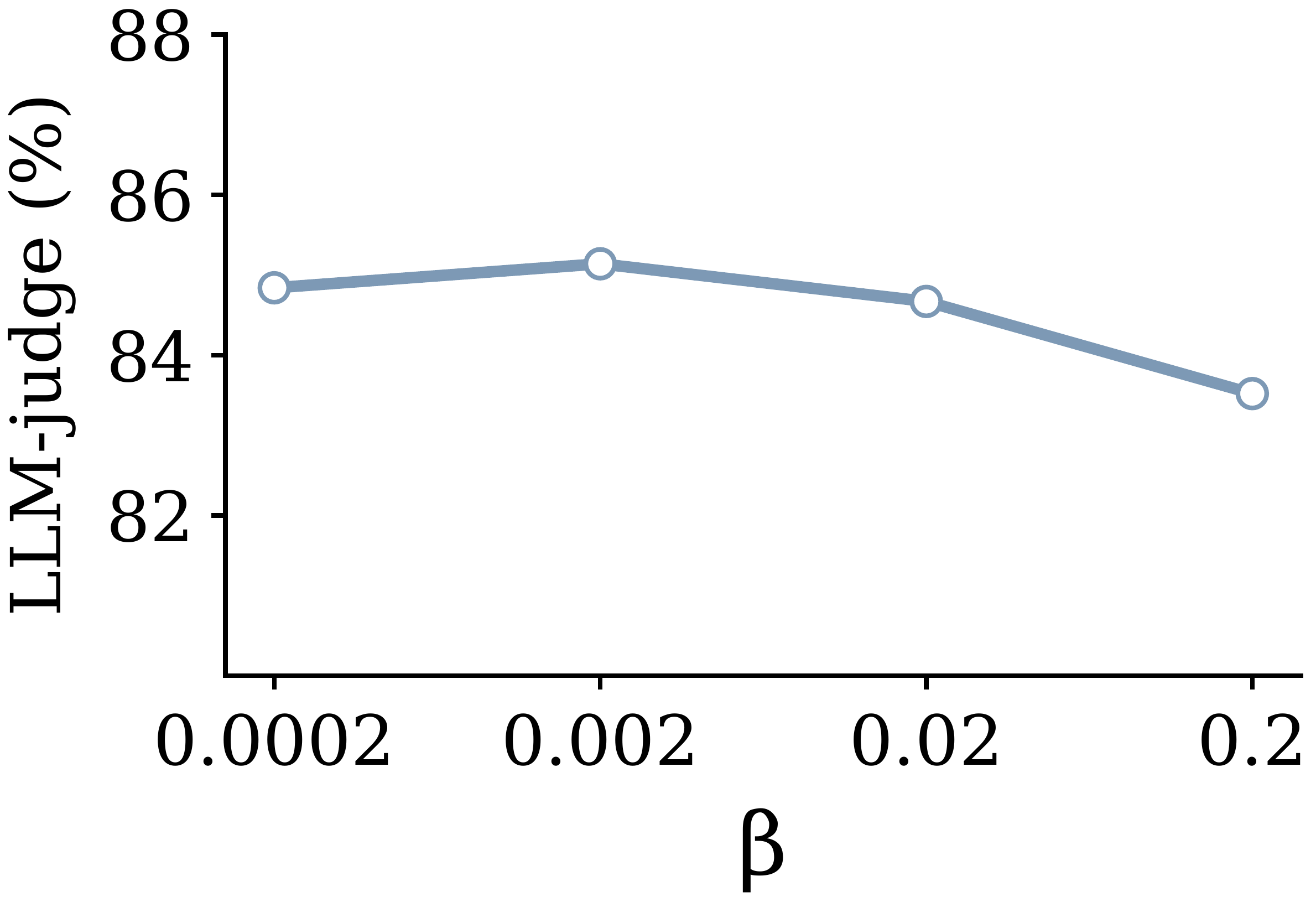}
    \caption{
    Teacher weight $\beta$ in \model, \emph{LLM-judge} score as metric}
    \label{fig:hyper sensitivity-LLM-judge}
  \end{minipage}
  \label{fig:hyper 3-metric}
\end{figure} 


\subsection{Naturally Multilingual MKQA: Detailed Results and Stress Tests}
\label{sec:app:exp:mkqa}

This appendix provides the supporting numbers for the MKQA generalization experiments referenced in \Cref{sec:exp:mkqa} and serves the low-resource OOD discussion in the main text. We use the original native-speaker MKQA questions \citep{DBLP:journals/tacl/LongpreLD21} paired with English Wikipedia as the sole evidence source, with the same \texttt{Llama-3.2-3B-Instruct} backbone and reward configuration as our \mbox{ENKB-RAG-5} runs.

\paragraph{Aggregated metrics (ALL-AVG).}
Beyond the per-language Composite, Qwen judge, and Llama judge already reported in \Cref{tab:mkqa}, \Cref{tab:mkqa-allavg} reports the two judge-independent signals (language consistency and char 3-gram recall) averaged over all 6 MKQA evaluation languages.

\begin{table}[h]
\centering
\setlength{\tabcolsep}{6pt}
\caption{Aggregated MKQA results (ALL-AVG over ko/th/vi/en/ja/no, \%; \texttt{Llama-3.2-3B-Instruct} backbone). \model\ matches or improves Naive-RL on every metric, including the two judge-independent signals (LC and char 3-gram).}
\label{tab:mkqa-allavg}
\begin{tabular}{l|ccccc}
\toprule
Method & Composite & LC & Char 3-gram & Qwen judge & Llama judge \\
\midrule
Base     & 53.70 & 44.00 & 41.39 & 65.69 & 63.75 \\
Naive-RL & \underline{57.18} & 44.00 & \underline{42.85} & \underline{71.27} & \underline{70.58} \\
\rowcolor{gray!15}
\model   & \textbf{57.82} & \textbf{44.50} & \textbf{43.05} & \textbf{71.98} & \textbf{71.76} \\
\bottomrule
\end{tabular}
\end{table}

\subsubsection{Stress tests on extremely low-resource OOD languages.}
\label{sec:app:exp:mkqa:lowres}

The most informative MKQA splits for assessing teacher anchoring are languages on which the base \texttt{Llama-3.2-3B-Instruct} model itself nearly fails. We highlight two such cases: \textbf{Norwegian (no)}, a Germanic language typologically distant from all training languages (Koreanic, Tai--Kadai, Austroasiatic) with base LC of only $7\%$; and \textbf{Khmer (km)}, an Austroasiatic language with a unique Khmer script that is severely underrepresented in standard LLM pre-training corpora, where the base model achieves only $18\%$ LC and $2.20\%$ char 3-gram recall.

\begin{table}[h]
\centering
\footnotesize
\setlength{\tabcolsep}{6pt}
\caption{Low-resource OOD stress tests on MKQA: Norwegian (no, base LC $7\%$) and Khmer (km, base LC $18\%$, unique script). \texttt{Llama-3.2-3B-Instruct} backbone. \model\ is the only method that improves grounding (char 3-gram, both judges) on both stress tests; on Norwegian it also lifts LC.}
\label{tab:mkqa-lowres}
\begin{tabular}{l|l|cccc}
\toprule
Language & Method & LC (\%) & Char 3-gram (\%) & Qwen judge (\%) & Llama judge (\%) \\
\midrule
\multirow{3}{*}{Norwegian (no)}
 & Base     & 7 & 51.46 & 67.16 & 63.80 \\
 & Naive-RL & 7 & \underline{56.96} & \underline{74.38} & \underline{72.15} \\
\rowcolor{gray!15}\cellcolor{white}
 & \model   & \textbf{9} & \textbf{57.63} & \textbf{77.93} & \textbf{75.45} \\
\midrule
\multirow{3}{*}{Khmer (km)}
 & Base     & 18 & 2.20 & 10.70 & 15.95 \\
 & Naive-RL & 18 & \underline{12.97} & \underline{30.25} & \underline{31.75} \\
\rowcolor{gray!15}\cellcolor{white}
 & \model   & \textbf{18} & \textbf{13.63} & \textbf{30.75} & \textbf{33.60} \\
\bottomrule
\end{tabular}
\end{table}

Two observations are consistent across both languages:
(i) \textbf{Naive-RL stalls at the base LC ceiling.} On Norwegian it remains at $7\%$ LC and on Khmer at $18\%$, indicating that unconstrained reward optimization cannot recover language adherence when the base model lacks the underlying linguistic competence.
(ii) \textbf{The teacher anchor stabilizes grounding even when LC is capacity-bottlenecked.} \model\ achieves the largest OOD composite gain among all evaluated languages on Norwegian (+2.38 over Naive-RL) and the largest cross-family Llama-judge gain on Khmer (+1.85), without dragging the model further toward English. The Khmer ceiling at $18\%$ LC reflects a hard capacity limit of the 3B backbone for that script; nonetheless, content quality (char 3-gram, judge scores) still improves under the prefix-wise anchor.


\subsection{Qualitative Case Study}
\label{sec:app:exp:case study}


The main text presents the Indonesian case (\Cref{fig:case}); \Cref{fig:case_ja,fig:case_ko,fig:case_thai,fig:case_viet} give parallel Japanese, Korean, Thai, and Vietnamese examples on Hotpot-ENKB5 with the \texttt{Llama-3.2-3B-Instruct} backbone. Highlighting follows the same convention as \Cref{fig:case}: red marks language drift or evidence-unsupported content, and green marks correct, evidence-grounded content in the target language.

Across languages, the same two patterns recur. \textbf{(i) Language drift.} Baseline systems start in the target language but slip into copied English spans from the retrieved passages (most conspicuously for Korean and Thai, whose non-Latin scripts make the switch obvious), whereas \model\ sustains the target language throughout the answer. \textbf{(ii) Evidence misuse.} Baselines tend to latch onto a salient-but-unsupported entity, while \model\ follows the full evidence chain to the entity the question actually asks about. These qualitative differences mirror the aggregate gains in language consistency and char 3-gram recall.

\paragraph{A case where the student surpasses its teacher.} The Japanese example (\Cref{fig:case_ja}) is especially instructive: the frozen 70B teacher \emph{abstains} (it declines to commit to an answer), yet the \model\ student, anchored to that same teacher only on its own visited prefixes, produces the correct, evidence-grounded Japanese answer. This concretely shows that the reverse-KL anchor acts as a \emph{stabilizer on student-visited prefixes} rather than as pure capability transfer: the student can exceed the teacher on a given instance, consistent with our teacher-quality ablation (\Cref{sec:app:exp:teacher_quality}), where even a same-size weak teacher still yields gains.

\begin{figure*}[tbp]
   \centering
   \includegraphics[width=1\linewidth]{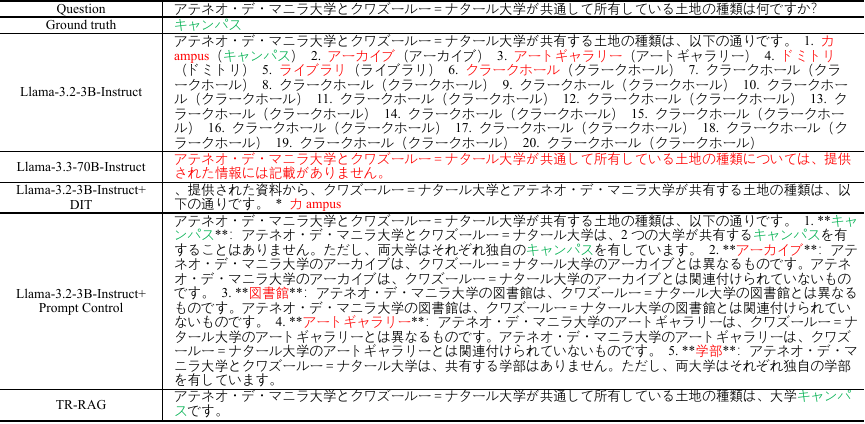}
   \caption{Case study. Llama-3.2-3B-Instruct as the backbone, Japanese as language, in Hotpot-ENKB5.}
   \vspace{-2mm}
   \label{fig:case_ja}
\end{figure*}

\begin{figure*}[tbp]
   \centering
   \includegraphics[width=1\linewidth]{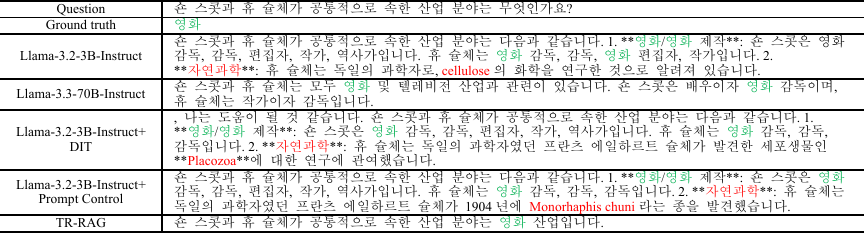}
   \caption{Case study. Llama-3.2-3B-Instruct as the backbone, Korean as language, in Hotpot-ENKB5.}
   \vspace{-2mm}
   \label{fig:case_ko}
\end{figure*}

\begin{figure*}[tbp]
   \centering
   \includegraphics[width=1\linewidth]{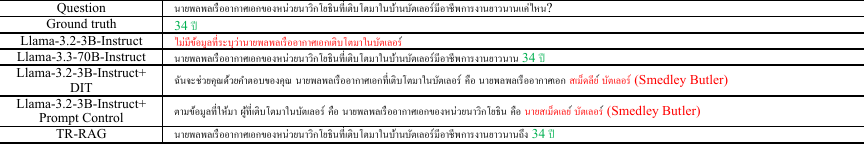}
   \caption{Case study. Llama-3.2-3B-Instruct as the backbone, Thai as language, in Hotpot-ENKB5.}
   \vspace{-2mm}
   \label{fig:case_thai}
\end{figure*}

\begin{figure*}[tbp]
   \centering
   \includegraphics[width=1\linewidth]{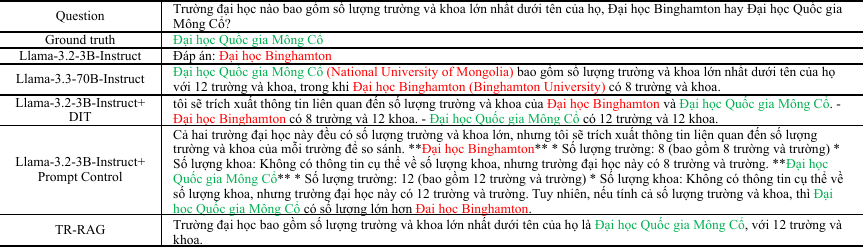}
   \caption{Case study. Llama-3.2-3B-Instruct as the backbone, Vietnamese as language, in Hotpot-ENKB5.}
   \vspace{-2mm}
   \label{fig:case_viet}
\end{figure*}

\subsection{Robustness as an Insurance Premium: Extended Discussion}
\label{sec:app:exp:robustness}

A consistent observation across \Cref{tab:llama-bioasq,tab:qwen-bioasq,tab:llama-hotpotqa,tab:qwen-hotpotqa,tab:mkqa} is that Naive-RL occasionally edges \model\ on the LLM-judge metric in ``safe'' settings (e.g., BioASQ Llama-3B by 1--1.5pp), while \model\ wins on every other axis. This is a \emph{designed} trade-off rather than a flaw: an unconstrained policy that is free to optimize the judge alone will, by construction, reach a marginally higher peak than the same policy under a teacher anchor, \emph{provided it does not collapse}.

But the same Naive-RL pipeline collapses outside this safe regime: on Hotpot-ENKB5 with Qwen-4B (\Cref{tab:qwen-hotpotqa}), Naive-RL's Indonesian LC drops to \textbf{28.27\%} (vs.\ \textbf{55.60\%} for \model), roughly $27$pp below both \model\ and the base model, and on the typologically distant Norwegian/Khmer stress tests (\Cref{tab:mkqa-lowres}; Appendix~\ref{sec:app:exp:mkqa:lowres}) the base LC is only 7\%/18\%, where Naive-RL stalls at that ceiling while \model\ still improves grounding. The teacher anchor therefore functions as a \emph{safety net} whose cost is a small \emph{insurance premium}: we give up 1--1.5pp on the judge in safe settings in exchange for protection against these large in-domain LC collapses and against base-ceiling stalls on distant OOD languages.

A second, complementary effect is precision: on lexically demanding domains (e.g., BioASQ entities such as \texttt{BRCA1} or \texttt{Acetaminophen}) the judge's smoother signal lets Naive-RL paraphrase, while the prefix-wise anchor keeps \model\ on the exact terminology, yielding consistently higher char 3-gram recall.

\newpage
\clearpage

\section{Related Work}
\label{sec:app:related}

\paragraph{Cross-lingual RAG and generation-side failures.}
Retrieval-augmented generation (RAG) grounds LLMs with external evidence and is a standard recipe for knowledge-intensive generation \citep{DBLP:conf/nips/LewisPPPKGKLYR020, DBLP:conf/nips/AsaiYKH21, Gao2023RetrievalAugmentedGF}.
Recent biomedical RAG evidence further cautions that retrieval alone does not guarantee better answers: \citet{nourbakhsh2026retrieval} report small and inconsistent gains from retrieval across open-weight models and biomedical QA settings, pointing to the generator's ability to use evidence as a central bottleneck.
Cross-lingual open-retrieval QA \citep{DBLP:journals/tacl/ClarkPNCGCK20, asai2021xor, DBLP:journals/tacl/LongpreLD21} and multilingual IR benchmarks such as MIRACL \citep{zhang-etal-2023-miracl} reveal the difficulty of retrieving and answering across languages, but most prior work focuses on \emph{retrieval quality}; we instead hold retrieval fixed and target \emph{generation-side} post-training under English-only evidence.
When the evidence language differs from the user language, multilingual RAG surfaces generation-side failures (response language drift, code-switching, and brittle evidence use \citep{DBLP:journals/corr/abs-2505-10089, DBLP:journals/corr/abs-2407-01463, li2025language}), even if evidence is homogenized into a pivot language \citep{DBLP:journals/corr/abs-2504-03616, Qi2025OnTC}.
At the mechanistic level, \citet{zhang2026safetyneurons} show that multilingual LLMs route cross-lingual behavior through a tiny set of shared English-centric neurons; non-high-resource languages lack autonomous internal pathways and instead piggyback on English representations, a structural bias that helps explain why injecting English evidence amplifies language drift.
Concurrent benchmarks XRAG \citep{DBLP:journals/corr/abs-2505-10089} and CrossRAG \citep{DBLP:journals/corr/abs-2504-03616} characterize these failures but do not propose \emph{training} recipes.
Recent concurrent work applies RL to multilingual or cross-lingual RAG: LcRL \citep{qi2026lcrl} introduces language-coupled group sampling and an anti-consistency penalty within GRPO for search-augmented multilingual generation; CroSearch-R1 \citep{qi2026crosearch} uses multi-turn cross-lingual retrieval with RL to dynamically integrate knowledge across languages; and RAG-RL \citep{huang2025ragrl} trains a reasoning reader with GRPO and curriculum learning for monolingual RAG.
These methods focus on retrieval-side optimization or monolingual settings.
In the broader multilingual RL space, BRIDGE \citep{hwang2025bridge} combines multilingual GRPO with a language-consistency reward for reasoning fidelity, and \citet{elhady2026xlsc} enforce cross-lingual self-consistency via unsupervised RL without parallel data; both target multilingual reasoning rather than RAG generation under language-mismatched evidence.
\textbf{Our work is, to our knowledge, the first to address cross-lingual generation-side failures via a dedicated post-training framework that couples reward optimization with on-policy distillation under English-only evidence}.

\paragraph{Reward-based RL, knowledge distillation, and on-policy distillation for LLMs.}
Preference optimization and RLHF fine-tune LLMs with reward signals \citep{schulman2017proximal, ouyang2022training}; recent on-policy variants stabilize training under noisy sequence-level rewards.
A recurring challenge is that early decoding deviations cascade (\emph{exposure bias} / prefix mismatch), making offline SFT/KD unable to correct errors on \emph{student-visited} prefixes, while sequence-level rewards (e.g., discrete language checks or judge scores) yield delayed, high-variance gradients.
Knowledge distillation (KD) transfers capability from a larger teacher to a smaller student \citep{DBLP:journals/corr/HintonVD15}, with a classic token-level line (DistilBERT \citep{sanh2019distilbert}, TinyBERT \citep{jiao2020tinybert}, MiniLM \citep{wang2020minilm}) and a sequence-level variant (SeqKD \citep{kim2016sequence}).
More recently, on-policy distillation addresses prefix mismatch by querying the teacher on student-generated prefixes: MiniLLM \citep{gu2023minillm} optimizes a sequence-level reverse-KL objective; GKD \citep{DBLP:conf/iclr/AgarwalVZSGGB24} applies token-level KL regularization on on-policy samples without teacher rollouts.
Subsequent work further explores the RL--distillation interface: RLAD \citep{zhang2026rlad} proposes Trust Region Ratio Distillation, replacing standalone KL regularization with a PPO/GRPO-style likelihood-ratio objective anchored to a teacher--student mixture, achieving advantage-aware distillation on student rollouts; OPSD \citep{zhao2026opsd} eliminates the need for an external teacher entirely by conditioning a single model on privileged information (e.g., verified reasoning traces) to supervise its own on-policy trajectories; and \citet{shen2026rlkd} leverage LLM-as-a-judge rewards for label-free RL-based knowledge distillation, enabling efficient distillation over unlabeled data. All three target monolingual reasoning benchmarks.
For multilingual small models, \citet{payoungkhamdee2026dudi} combine sequence-level and token-level distillation signals with a cross-lingual verbalizer; our setting instead embeds teacher anchoring inside reward-optimized RAG generation under language-mismatched evidence.
Concurrent RLVR work also revisits credit assignment under sequence-level verifiable rewards: \citet{pala2026grail} propose gradient-reweighted advantages so that policy-gradient updates focus more on influential reasoning tokens rather than uniformly broadcasting one scalar advantage across the full response.
\textbf{Our method builds on GKD's on-policy distillation principle but makes three key departures that, taken together, address a problem none of the above works target}: (i)~we \emph{combine} the teacher anchor with GRPO-style task-reward optimization in a single objective rather than using distillation alone: the ablations in \Cref{table:ablation} confirm that this combination is strictly better than either component in isolation on the composite metric (+2.19 over RL-only, +3.52 over OPD-only); (ii)~we design a \emph{cross-lingual reward decomposition} (language consistency + character 3-gram recall + LLM-judge) tailored to English-evidence multilingual generation, where the LC reward, script-agnostic char 3-gram metric, and cross-lingual teacher selection are all domain-specific design choices meaningless in the monolingual settings prior work considers; and (iii)~we demonstrate empirically that the teacher anchor serves as a ``safety net'' against \emph{catastrophic language drift}, i.e., large in-domain LC collapses (up to $\sim$27pp, where reward-only RL drifts below even the base model) and base-ceiling stalls on distant OOD languages that pure reward-only RL cannot prevent (\Cref{sec:app:exp:robustness}); a failure mode unique to cross-lingual generation under language-mismatched evidence.

\paragraph{Evaluation signals for cross-lingual generation.}
LLM-as-a-judge \citep{DBLP:conf/nips/ZhengC00WZL0LXZ23} provides graded, evidence-conditioned evaluation for open-ended generation; character n-gram metrics such as chrF/chrF++ \citep{popovic-2015-chrf, popovic-2017-chrf} are robust across diverse scripts; and fastText-based language identification \citep{joulin2017bag, grave2018learning} monitors response-language adherence.
Recent judge-calibration work shows that self-evaluation signals can be elicited from base LLMs with minimal data to predict external LLM-judge scores \citep{zhang2026selfevaluation}.
At the same time, rubric-based RL can expose policies to LLM-judge biases and reward hacking; CHERRL \citep{wang2026cherrl} provides a controlled environment for reproducing and detecting such failures.
\textbf{Unlike prior work that uses these signals only for offline evaluation}, we repurpose all three as \emph{online reward components} within a unified RL objective, providing complementary training pressure on language adherence, lexical faithfulness, and semantic correctness simultaneously.



\end{document}